\documentclass[english]{article}
\usepackage[a4paper]{geometry}
\usepackage{setspace}

\usepackage[authoryear]{natbib}
\usepackage[unicode=true,bookmarks=true,bookmarksnumbered=false,bookmarksopen=false,breaklinks=false,pdfborder={0 0 0},pdfborderstyle={},backref=page,colorlinks=true]{hyperref}
\hypersetup{pdftitle={SIFG},pdfauthor={Shiyue Zhang, Ziheng Cheng, Cheng Zhang},linkcolor=RoyalBlue,citecolor=RoyalBlue}
\usepackage{float}
\usepackage[dvipsnames,svgnames,x11names]{xcolor}
\usepackage{nicefrac}

\usepackage[textsize=tiny]{todonotes}

\makeatletter
\newcommand{\myfnsymbol}[1]{%
  \expandafter\@myfnsymbol\csname c@#1\endcsname
}

\newcommand{\@myfnsymbol}[1]{%
  \ifcase #1
  \or 1
  \or 2
  \or 3
  \or \TextOrMath{\textasteriskcentered}{*}
  \or \TextOrMath{\textdagger}{\dagger}
  \fi
}

\newcommand{\affiliationA}{\@myfnsymbol{1}}
\newcommand{\affiliationB}{\@myfnsymbol{2}}
\newcommand{\affiliationC}{\@myfnsymbol{3}}
\newcommand{\correspondingA}{\@myfnsymbol{4}}
\newcommand{\correspondingB}{\@myfnsymbol{5}}
\makeatother

\usepackage{amsmath}
\usepackage{graphicx}
\usepackage{enumerate}
\usepackage{natbib}
\usepackage{url} 
\usepackage{microtype}
\usepackage{multirow} 
\usepackage{diagbox}
\usepackage{wrapfig}
\usepackage{amsthm,amsmath,amssymb, amscd}
\usepackage{mathrsfs}
\usepackage{amsfonts}
\usepackage[utf8]{inputenc}
\usepackage{subfigure}
\usepackage{graphicx}
\usepackage{bbding}
\everymath{\displaystyle}
\usepackage{indentfirst} 
\usepackage{enumerate}
\usepackage{bm,bbm}
\usepackage{enumitem}
\usepackage{algorithm, algorithmic}
\usepackage{dsfont}
\usepackage{booktabs}
\usepackage{xspace}
\usepackage{pifont}
\usepackage{xcolor}
\PassOptionsToPackage{compress, square}{natbib}
\usepackage{natbib}
\numberwithin{equation}{section}
\usepackage{apptools}

\theoremstyle{plain}
\newtheorem{theorem}{Theorem}[section]
\newtheorem{proposition}[theorem]{Proposition}
\newtheorem{lemma}[theorem]{Lemma}
\newtheorem{corollary}[theorem]{Corollary}
\theoremstyle{definition}
\newtheorem{definition}[theorem]{Definition}
\newtheorem{assumption}[theorem]{Assumption}
\theoremstyle{remark}

\AtAppendix{\counterwithin{theorem}{section}}

\usepackage{amsmath,amsfonts,bm}

\def\eqref#1{equation~\ref{#1}}

\def\1{\bm{1}}

\DeclareMathAlphabet{\mathsfit}{\encodingdefault}{\sfdefault}{m}{sl}
\SetMathAlphabet{\mathsfit}{bold}{\encodingdefault}{\sfdefault}{bx}{n}

\newcommand{\E}{\mathbb{E}}

\newcommand{\R}{\mathbb{R}}

\newcommand{\KL}{D_{\mathrm{KL}}}
\newcommand{\Var}{\mathrm{Var}}

\DeclareMathOperator*{\argmin}{arg\,min}

\renewcommand{\eqref}[1]{(\ref{#1})}
\renewcommand{\P}{{\mathbb{P}}}

\begin{document}

\title{Semi-Implicit Functional Gradient Flow for Efficient Sampling}

\author{
Shiyue Zhang\textsuperscript{\affiliationA,\correspondingA},
Ziheng Cheng\textsuperscript{\affiliationB,\correspondingA},
Cheng Zhang\textsuperscript{\affiliationC,\correspondingB}
}

\date{
}

\renewcommand{\thefootnote}{\myfnsymbol{footnote}}
\maketitle

\footnotetext[1]{School of Mathematical Sciences, Peking University. Email: zhangshiyue@stu.pku.edu.cn}%
\footnotetext[2]{Department of Industrial Engineering and Operations Research, University of California, Berkeley. \\Email: ziheng\_cheng@berkeley.edu}%
\footnotetext[3]{School of Mathematical Sciences and Center for Statistical Science, Peking University. Email: chengzhang@math.pku.edu.cn}
\footnotetext[4]{Equal contribution.}%
\footnotetext[5]{Corresponding author.}%

\setcounter{footnote}{0}
\renewcommand{\thefootnote}{\fnsymbol{footnote}}


\begin{abstract}
Particle-based variational inference methods (ParVIs) use nonparametric variational families represented by particles to approximate the target distribution according to the kernelized Wasserstein gradient flow for the Kullback-Leibler (KL) divergence.
Although functional gradient flows have been introduced to expand the kernel space for better flexibility, the deterministic updating mechanism may limit exploration and require expensive repetitive runs for new samples.
In this paper, we propose Semi-Implicit Functional Gradient flow (SIFG), a functional gradient ParVI method that uses perturbed particles with Gaussian noise as the approximation family.
We show that the corresponding functional gradient flow, which can be estimated via denoising score matching with neural networks, exhibits strong theoretical convergence guarantees
due to a higher-order smoothness brought to the approximation family via Gaussian perturbation.
In addition, we present an adaptive version of our method that automatically selects the appropriate noise magnitude during sampling, striking a good balance between exploration efficiency and approximation accuracy. 
Extensive experiments on both simulated and real-world datasets demonstrate the effectiveness and efficiency of the proposed framework.
\end{abstract}

\noindent%

\section{Introduction}
Bayesian inference provides powerful tools for modeling complex data and reasoning under uncertainty, which is a fundamental and challenging task in modern statistics and machine learning, with a wide range of applications in molecular dynamics \citep{krauth2006mech}, inverse problems \citep{dashti2013inverse}, and diffusion models \citep{song2020score}.
One of the core problems of modern Bayesian inference is to estimate the posterior distribution, which is often intractable for complex models, necessitating the use of approximation methods. 
Two common approaches to address this intractability are Markov Chain Monte Carlo (MCMC) and Variational Inference (VI).
By reformulating the inference problem into an optimization problem, VI seeks to find an approximation within a specific distribution family that minimizes the Kullback-Leibler (KL) divergence to the posterior \citep{Jordan1999AnIT, Wainwright08, Blei2017VariationalIA, yao2022mean}.
While VI allows fast training and scales effectively to large datasets with efficient optimization algorithms, the choice of variational distribution family may limit its approximation power.
On the other hand, MCMC simulates a Markov chain that satisfies the detailed balance condition to directly draw samples from the posterior \citep{duane87, MALA02, Neal2011-yo, SGLD, SGHMC}.
Although MCMC is asymptotically unbiased, it often suffers from slow convergence, and assessing convergence can be challenging.

Recently, there has been a surge of interest in the gradient flow formulation of both MCMC and VI, leading to the development of particle-based variational inference methods (ParVI) that combine the strengths of both approaches \citep{Liu2016SVGD, Chen2018unified, Liu2019Understanding, di2021neural, fan2022variational, alvarez-melis2022optimizing}.
From a variational perspective, ParVIs adopt a nonparametric approach, where the approximating distribution is represented by a set of particles.
These particles are iteratively updated in the steepest descent direction to minimize the KL divergence to the posterior, following the gradient flow in a space of distributions equipped with specific geometries (e.g., the Wasserstein metric).
This nonparametric nature enhances the flexibility of ParVIs compared to classical parametric VI methods, while the interactions between particles improve particle efficiency relative to MCMC.
One of the most prominent particle-based VI methods is Stein Variational Gradient Descent (SVGD) \citep{Liu2016SVGD}, where the steepest descent direction has a close form when confined within the unit ball of a reproducing kernel Hilbert space (RKHS) \citep{Liu2017SVGF, Chewi2020chi-squared}.
However, the performance of SVGD is highly dependent on the choice of the kernel function, and the quadratic computational complexity of the kernel matrix also hinders practical usage of a large number of particles.

To address these challenges, functional gradient flow methods have been introduced that expand the kernel space to enable more flexible gradient flow approximations \citep{Hu18, grathwohl2020, di2021neural, dong2023particlebased, cheng2023gwg}.
By leveraging richer function classes, such as neural networks, these functional gradient approaches have shown improved performance over vanilla SVGD while avoiding the need for costly kernel computations. However, the deterministic updating mechanism in these approaches often leads to limited exploration and can result in mode collapse when dealing with non-convex, multi-modal distributions. Furthermore, multiple runs are typically required to generate diverse sets of samples, adding to the computational burden.

In this work, we introduce Semi-Implicit Functional Gradient flow (SIFG), a stochastic ParVI method that enables both efficient exploration and diverse sampling.
The core idea is to use perturbed particles with Gaussian noise, rather than the original noiseless particles, as the variational approximation.
The introduction of Gaussian noise provides two key benefits: (i) it turns the deterministic updating mechanism into a stochastic one, facilitating exploration and aiding the discovery of new modes, and (ii) it prevents convergence to a single optimal set of particles, allowing the method to generate different samples on the fly.
As the variational distribution represented by the perturbed particles has a natural hierarchical structure, akin to those used in semi-implicit variational inference, we name our method Semi-Implicit Functional Gradient flow.
Moreover, due to this hierarchical structure, the corresponding Wasserstein gradient flow can be efficiently estimated via denoising score matching \citep{vincent2011connection} with neural networks, which scales well in high-dimensional settings \citep{ho2020ddpm}.
Given accurate neural network approximations and some mild smoothness conditions,
we establish a convergence guarantee for SIFG to a stationary point of the KL objective function.
This extends the convergence theory of variational inference with fixed-variance mixture Gaussians \citep{tom2024theoretical} to semi-implicit variational distributions, with a rate independent of the number of mixtures.
We also provide statistical guarantees for the sample complexity required to satisfy the neural network approximation assumption.
In addition, we propose an adaptive procedure that can automatically adjust the noise magnitude, balancing between sample accuracy and diversity. Extensive numerical experiments on both simulated and real-world datasets demonstrate the advantages of our method over existing ParVI approaches.

The rest of the paper is organized as follows. Section \ref{sec:back} provides a brief overview
of particle-based variational inference, and the more advanced variants, functional gradient flows. 
In Section \ref{sec:method}, we introduce our method semi-implicit functional gradient flow.
Section \ref{sec:convergence_parvi} presents the convergence analysis, including both optimization and statistical guarantees.
Numerical experiments are presented in Section \ref{section:exp}.
Finally, Section \ref{sec:concl} concludes the paper.

\section{Background}\label{sec:back}

\paragraph{Notations}
Let $\mathcal{P}(\mathbb{R}^d)$ denote all probability distributions on $\mathbb{R}^d$ that are absolutely continuous with respect to the Lebesgue measure. We do not distinguish a probabilistic distribution from its density function.
The same notation $\|\cdot\|$ is used for the standard Euclidean norm of a vector and the operator norm of a matrix or high-dimensional tensor.
We denote the inner product in $\R^d$ (or $L^2(\R^d)$) and a certain Hilbert space $\mathcal{H}$ by $\left\langle \cdot,\cdot \right\rangle$ and $\left\langle \cdot,\cdot \right\rangle_{\mathcal{H}}$, respectively. 
For a distribution $\mu(t,x)$, we use $\dot{\mu}$ and $\nabla\mu$ to denote the time and space derivatives of $\mu(t,x)$. Let $\mathcal{F}(\mu): \mathcal{P}(\mathbb{R}^d)\to \mathbb{R}$ be the energy functional of $\mu$. 
We use $\frac{\delta \mathcal{F}}{\delta \mu}$ to denote the first variation of the functional $\mathcal{F}$ and $\text{D}\mathcal{F}$ to denote the Fr$\acute{\text{e}}$chet sub-differential of the functional $\mathcal{F}$. We use $\nabla_{W_2}\mathcal{F}:=\nabla \frac{\delta \mathcal{F}}{\delta \mu}$ to denote the Wasserstein-2 gradient of $\mathcal{F}$.
We use the $\mathcal{O}(\cdot)$ notation to denote an upper bound that omits constant factors and lower-order terms.

\subsection{Particle-based Variational Inference}\label{subsec:wassgf}

Let $\pi\in\mathcal{P}(\mathbb{R}^d)$ denote the target probability distribution that we aim to sample from.
Unlike MCMC methods, VI tackles this problem by identifying the closest member $\mu^*$ from a family of candidate distributions $\mathcal{Q}$ (the variational family) that minimizes some statistical distance to the target distribution $\pi$, typically the Kullback-Leibler (KL) divergence:
\begin{equation}\label{eq:vari}    \mu^*:=\argmin_{\mu\in\mathcal{Q}} \KL (\mu\|\pi).
\end{equation}

Instead of assuming a parametric form for $\mathcal{Q}$ as in standard VI \citep{Jordan1999AnIT, bishop2000variational, Blei2017VariationalIA}, particle-based variational inference takes a nonparametric approach where $\mathcal{Q}$ is represented as a set of particles.
Let $\mu_0\in\mathcal{P}(\mathbb{R}^d)$ be the initial probability distribution that we can easily sample from (e.g., a Gaussian distribution).
Starting from $\mu_0$ and the initial particles $z_0^1,\ldots,z_0^M\sim\mu_0$, for $t>0$, we update the particle $z_t^{i}, i=1,\ldots,M$ according to an ODE system $dz_t=v_t(z_t)dt$, where $v_t$ is the velocity field at time $t$.
The distribution $\mu_t$ then evolves according to the continuity equation $\dot{\mu}_t + \nabla \cdot (\mu_t v_t)=0$, and the KL divergence decreases at the following rate:
\begin{equation}\label{eq:dKL}
    \frac{d}{dt} \KL(\mu_t\|\pi)=-\langle \log\frac{\mu_t}{\pi}, \nabla\cdot(\mu_t v_t)\rangle= -\E_{\mu_t} \langle \nabla\log\frac{\pi}{\mu_t}, v_t\rangle.
\end{equation}
The key challenge in particle-based VI is to find an appropriate $v_t$ that effectively drives the particles toward the target distribution $\pi$ (e.g., by decreasing the KL).
Given a Hilbert space $\mathcal{H}$, one can find such a $v_t$ in $\mathcal{H}$ by minimizing the following objective:
\begin{equation}\label{eq:optim}
    \min_{v_t\in \mathcal{H}} -\E_{\mu_t} \langle \nabla\log\frac{\pi}{\mu_t}, v_t\rangle + \frac{1}{2} \|v_t\|^2_{\mathcal{H}}.
\end{equation}
Although \eqref{eq:optim} takes a quadratic form, $v_t$ remains difficult to estimate as the score function $\nabla\log\mu_t$ of the current particle distribution is intractable.
This issue has been resolved in the celebrated Stein Variational Gradient Descent (SVGD) \citep{Liu2016SVGD}, where $\mathcal{H}$ is chosen as a Reproducing Kernel Hilbert Space (RKHS) with a kernel function $k(\cdot, \cdot)$.
Leveraging Stein's identity, it can be shown that the optimal velocity field for \eqref{eq:optim} takes the following form
\begin{equation}\label{eq:svgd1}
v_t^*(\cdot) = \E_{\mu_t}[k(\cdot, x) \nabla\log \pi(x) + \nabla_xk(\cdot, x)],
\end{equation}
which can be estimated using the current particles $z_t^1,\ldots,z_t^M\sim\mu_t$.

\subsection{Functional Gradient Flow}\label{sec:FGF}

From a geometric perspective, particle-based variational inference approaches (ParVIs) can be interpreted as gradient flows of probability distributions
that minimize a specific energy functional, typically the KL divergence.
Intuitively, a gradient flow describes a dynamical system that evolves to dissipate the objective energy functional $\mathcal{F}=\mathcal{F}(\mu)$ as efficiently as possible, where the rate of dissipation is governed by the geometric structure of the probability space.
Different geometric structures lead to different types of gradient flows.
 
In Euclidean space, given $u_t\in \mathbb{R}^d$ and an energy function $F:\mathbb{R}^d\to\mathbb{R}$, the gradient flow equation takes the form: $\dot{u}_t=-\nabla F(u_t)$. From the perspective of the variational principle, this equation can be rewritten as: $\nabla_{\dot{u}}(\frac{1}{2}\|\dot{u}_t\|^2)+\nabla_u F(u_t)=0$. Here $\frac{1}{2}\|\cdot\|^2$ depicts the standard Euclidean geometry dissipation.
A similar approach can be applied in the space of probability distributions.
One of the most commonly used geometries in this context is the 2-Wasserstein geometry, which can be derived from the Benamou-Brenier dynamic formulation of the 2-Wasserstein distance \citep{benamou2000wass}:
\begin{equation}
    \frac{1}{2}W^2_2(\mu_0,\mu_h)=\inf \{\int_0^h \frac{1}{2}\|v_t\|^2_{L^2(\mu_t)}dt:\dot{\mu}_t=-\nabla\cdot(\mu_t v_t)\}.
\end{equation}
The 2-Wasserstein dissipation is defined as $\mathfrak{R}_{W_2}(\mu_t,\dot{\mu}_t)=\inf \{\frac{1}{2}\|v_t\|^2_{L^2(\mu_t)}:\dot{\mu}_t=-\nabla\cdot(\mu_t v_t)\}$.
Similar to the Euclidean case, we define the Wasserstein gradient flow equation as: $\text{D}_{\dot{\mu}}\mathfrak{R}_{W_2}(\mu_t,\dot{\mu}_t)+\text{D}_{\mu}\mathcal{F}(\mu_t)=0$.
The solution of this equation leads to the velocity field $v_t=-\nabla\frac{\delta\mathcal{F}}{\delta\mu}(\mu_t):=-\nabla_{W_2}\mathcal{F}(\mu_t)$,
and thus the gradient flow equation becomes: $\dot{\mu}_t=\nabla\cdot(\mu_t\nabla\frac{\delta\mathcal{F}}{\delta\mu}(\mu_t))$.
Specifically, when $\mathcal{F}(\mu)=\KL(\mu\|\pi)$, the Wasserstein gradient flow equation simplifies to $\dot{\mu}_t=\nabla\cdot(\mu_t\nabla\log\frac{\mu_t}{\pi})$.
Please refer to Appendix \ref{app:wgf_basics} for further derivations on different gradient flows for various dissipation geometries.

Now suppose the kernel is square-integrable
$\|k\|_{L^2(\mu)}^2:=\int\int |k(x, x')|^2 d \mu(x)d \mu(x')<\infty$
with respect to a probability distribution $\mu$.
Under this assumption, 
the inclusion map from the associated RKHS $\mathcal{H}$ to $L^2(\mu)$, \(\text{Id} : \mathcal{H} \rightarrow L^2(\mu)\), is continuous.
Its adjoint operator, \(\mathcal{T}_{k,\mu}: L^2(\mu) \rightarrow \mathcal{H}\) is defined as
\begin{equation}\label{kernelop}
   \mathcal{T}_{k,\mu} g(x):=\int k\left(x, x^{\prime}\right) g\left(x^{\prime}\right) d \mu\left(x^{\prime}\right), \quad g \in L^2(\mu),
\end{equation}
which induces the integral operator $\mathcal{K}_{\mu}:= \text{Id} \circ \mathcal{T}_{k,\mu} :
L^2(\mu) \to L^2(\mu)$.
When $\mu$ is the standard Lebesgue measure, we denote $\mathcal{K}_{\mu}$ as $\mathcal{K}$.
Let the objective functional $\mathcal{F}$ be the KL divergence.
The SVGD dynamics can be expressed as a gradient flow $\dot{\mu}_t=\nabla\cdot(\mu_t\mathcal{K}_{\mu_t}\nabla\frac{\delta \mathcal{F}}{\delta \mu}(\mu_t))$ under the kernelized Wasserstein geometry, also named as the Stein geometry \citep{zhu2024rao}. 
Other gradient flows have also been explored using different geometric structures, such as the Fisher-Rao gradient flow, which is based on the Fisher-Rao metric \citep{carles2023fisher}.

While SVGD provides an effective approximation of gradient flows, its performance is sensitive to the choice of kernel function, and its computational cost scales quadratically with the number of particles due to the kernel matrix computation.
This has motivated recent research into functional gradient flows, which replace kernel-based representations with neural network parameterizations, allowing for more scalable and adaptive approximations of the gradient flow.
More specifically, functional gradient flow methods approximate the gradient flow by solving
\begin{equation}\label{eq:sm_full}
    v^*_t = \argmin_{ v\in\mathcal{{S}}}\E_{\mu_t} \left[-\langle \nabla\log \frac{\pi}{\mu_t}, v\rangle + g(v)\right] =\argmin_{ v\in\mathcal{{S}}}\E_{\mu_t} \left[-\langle \nabla\log \pi, v\rangle - \nabla\cdot v+ g(v)\right],
\end{equation}
where $\mathcal{S}$ is the neural network family.
The regularization term $g$ can be any Young function \citep{cheng2023gwg} that is  strictly convex and non-negative (e.g., $g(v)=\frac12\|v\|^2$), 
and the second equation follows from Stein's identity.

Existing functional gradient flow methods mainly differ in their choice of regularization, which corresponds to different dissipation geometries.
For example, \citet{di2021neural} proposed using neural networks instead of kernels to approximate the Wasserstein gradient with an $L_2$ regularization.
More recent works, such as \citet{dong2023particlebased} and \citet{cheng2023gwg}, have generalized this approach by considering quadratic and $L_p$ regularization terms, respectively.
Please refer to Appendix \ref{app:wgf_basics} for a detailed explanation.

\begin{figure*}[t]
   \centering
   \subfigure{
   \begin{minipage}[t]{0.24\linewidth}
   \centering
   \includegraphics[width=1\textwidth]{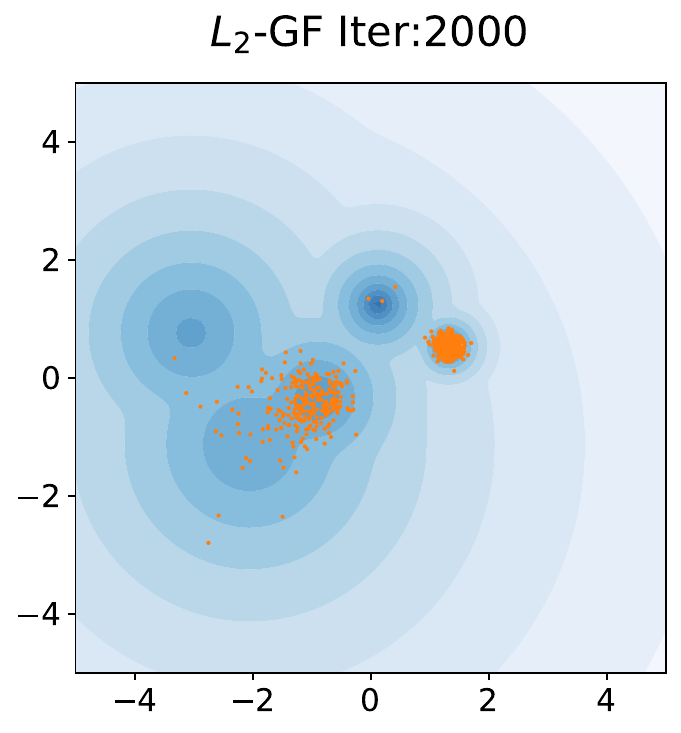}
   \end{minipage}%
   }%
   \hfill
   \subfigure{
   \begin{minipage}[t]{0.24\linewidth}
   \centering
   \includegraphics[width=1\textwidth]{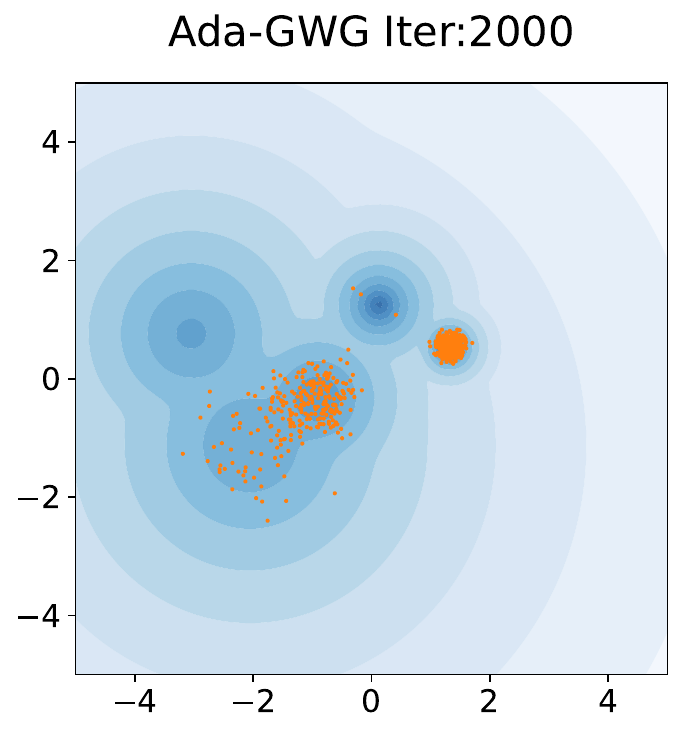}
   \end{minipage}%
   }%
   \hfill
   \subfigure{
   \begin{minipage}[t]{0.24\linewidth}
   \centering
   \includegraphics[width=1\textwidth]{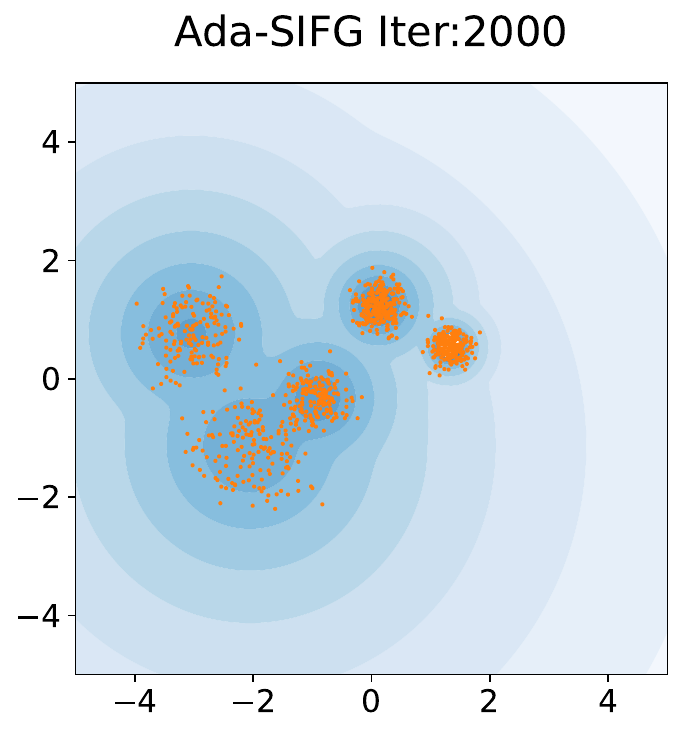}
   \end{minipage}%
   }%
   \hfill
   \subfigure{
   \begin{minipage}[t]{0.24\linewidth}
   \centering
   \includegraphics[width=1\textwidth]{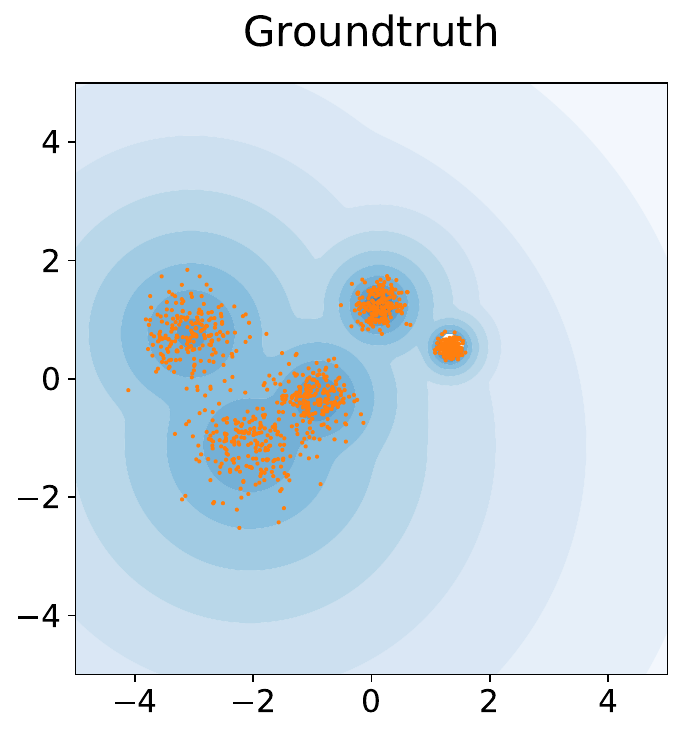}
   \end{minipage}%
   }%
   \centering
   \caption{Comparison of the sampled particles from different methods at the 2000-th iteration (sufficient for convergence of all methods) against the ground truth samples on a 2D Gaussian mixture model. 
   }
   \vspace{2em}
   \label{figure:toy2dexpl}
\end{figure*}

\section{Semi-Implicit Functional Gradient Flow}\label{sec:method}

While functional gradient flow methods offer greater flexibility and scalability due to the expressiveness of neural networks, their deterministic updates can hinder exploration, causing particles to become trapped in local modes and requiring multiple runs for diverse samples. To address these issues, we propose semi-implicit functional gradient flow (SIFG), a framework that integrates Gaussian perturbations into functional gradient flows. These perturbations facilitate escape from local minima, enhance exploration, and improve sample diversity, allowing a single run to generate multiple distinct sets of samples from the same approximate distribution on the fly.
Figure \ref{figure:toy2dexpl} illustrates the improved exploration of our method compared to existing approaches.

We begin in Section \ref{subsec:method1} by formulating the semi-implicit particle-based variational distribution, which establishes the connection between the perturbed particle distribution and the broader semi-implicit variational framework.
In Section \ref{subsec:sifggf}, we construct the Wasserstein gradient flow to minimize the KL divergence between the semi-implicit particle distribution and the target distribution,
deriving a practical algorithm that leverages denoising score matching \citep{vincent2011connection} for score function estimation.
Finally, Section \ref{subsec:method3} introduces a noise magnitude tuning scheme to balance approximation accuracy and sample diversity, ensuring effective exploration while maintaining fidelity to the target distribution.

\subsection{Semi-implicit Particle Distribution}\label{subsec:method1}

Suppose $\mu_t$ is the current particle distribution at time $t$.
To enhance exploration, we introduce a perturbed distribution by injecting noise into the particles.
This perturbation follows a hierarchical structure, similar to semi-implicit distributions used in semi-implicit variational inference \citep{yin2018semi, titsias2019unbiased, yu2023semi, cheng2024kernel}.
Accordingly, we refer to the resulting distribution as the semi-implicit particle distribution.
In this paper, we focus on isotropic Gaussian noise with variance $\sigma^2$, leading to the following semi-implicit particle distribution:
\begin{equation}\label{eq:sifgfamily}
\hat{\mu}_t(x)=\int q_{\sigma}(x|z)d\mu_t(z),
\end{equation}
where $q_{\sigma}(x|z)\propto e^{-\frac{\|x-z\|^2}{2\sigma^2}}$ is a Gaussian transition kernel with variance $\sigma^2$.
To establish a connection with the general gradient flow system introduced in Section \ref{subsec:wassgf}, we consider $k(x,z)=q_{\sigma}(x|z)$ to be the Gaussian kernel. Then $\hat{\mu}_t=\mathcal{K}\mu_t$, where $\mathcal{K}$ is the operator defined in (\ref{kernelop}) with the base measure $\mu$ being the Lebesgue measure.

A similar variational family has also been proposed in \citet{jen2024particle} for a new type of semi-implicit variational inference method.
However, our approach differs in several key aspects: (1) In their approach, the kernel is parameterized by neural networks, leading to a non-coercive optimization objective that requires additional regularization for the existence of minimizers, hence inducing bias. Our method avoids this by using Gaussian kernels, eliminating the need for explicit regularization. (2) As a functional gradient flow method, we estimate the Wasserstein gradient using neural networks, directly leveraging the smoothness of the particle distribution.
\citet{jen2024particle}, on the other hand, uses a Gaussian mixture approximation that misses the smoothness of the particle distribution 
and essentially resembles standard variational inference methods \citep{tom2024theoretical}. (3) Our methods provide end-to-end theoretical guarantees on the convergence rate and sample complexity. In contrast, \citet{jen2024particle} only offer guarantees regarding the existence and uniqueness of solutions, along with an asymptotic justification for using particles.

\subsection{Functional Gradient Flow Estimation}\label{subsec:sifggf}

Now we consider minimizing the KL divergence between the perturbed density $\hat{\mu}_t$ and the target distribution $\pi$:
\begin{align}
    \mathcal{F}(\hat{\mu}_t):=\KL(\hat{\mu}_t\|\pi)=\int \log \frac{\hat{\mu}_t(x)}{\pi(x)} d\hat{\mu}_t(x),
\end{align}
which can also be regarded as a functional of $\mu_t$ defined as $\hat{\mathcal{F}}(\mu_t):=\mathcal{F}(\hat{\mu}_t)=\mathcal{F}(\mathcal{K}\mu_t)$.
To update the particles, it suffices to compute the Wasserstein gradient flow of the energy function $\hat{\mathcal{F}}$.
The following theorem shows that this gradient flow takes a simple form when the transition kernel satisfies a certain skew-symmetric condition.

\begin{theorem}\label{mol_wgf}
If the gradient of the transition kernel is skew-symmetric, i.e., $\nabla_x k(x,z)=-\nabla_z k(x,z)$, then the Wasserstein gradient flow of the energy functional $\mathcal{\hat{F}}$ is 
\begin{equation}
     \dot{\mu}_t(z)=-\nabla\cdot\Big(\mu_t(z) \E_{k(x, z)} \nabla\log \frac{\pi(x)}{\hat{\mu}_t(x)}\Big)=-\nabla\cdot(\mu_t \mathcal{K}\nabla\log \frac{\pi}{\mathcal{K}\mu_t}),
\end{equation}
and the Wasserstein gradient is $\nabla_{W_2}\mathcal{\hat{F}}(\mu_t)(z)=-\E_{k(x,z)} \nabla\log \frac{\pi(x)}{\hat{\mu}_t(x)}=-\mathcal{K}\nabla\log \frac{\pi}{\mathcal{K}\mu_t}$.
\end{theorem}

Please refer to Appendix \ref{app:wgf_proof} for a detailed derivation.
Now, considering the semi-implicit variational distribution $\hat{\mu}_t=\mathcal{K}\mu_t$, the above gradient flow for $\mu_t$ can be transformed into a new type of gradient flow for $\hat{\mu}_t$. 
Since $\mu_t=\mathcal{K}^{-1}\hat{\mu}_t$, $\dot{\mu}_t=\mathcal{K}^{-1}\dot{\hat{\mu}}_t$ and $\nabla\frac{\delta \mathcal{F}}{\delta \hat{\mu}_t}=\nabla\log\frac{\hat{\mu}_t}{\pi}$, we have $\dot{\hat{\mu}}_t=\mathcal{K}\nabla\cdot(\mathcal{K}^{-1}\hat{\mu}_t\cdot\mathcal{K}\nabla\frac{\delta \mathcal{F}}{\delta \hat{\mu}_t})$.
Note that this equation is not the standard Wasserstein gradient flow equation of $\mathcal{F}$.
Instead, it can be interpreted as the semi-implicit functional gradient flow under the semi-implicit dissipation geometry. Please refer to Appendix \ref{app:wgf_basics} for detailed information. 
As long as the gradient of the transition kernel is skew-symmetric,  
the Wasserstein gradient flow of $\hat{\mathcal{F}}$ can be equivalently transformed into the semi-implicit functional gradient flow of $\mathcal{F}$.
In this paper, we only consider Gaussian transition kernel and leave the exploration of other kernel choices for future work.

We use a neural network $f_\gamma$ to approximate $ \nabla\log \frac{\pi}{\hat{\mu}_t}$.
As introduced in Section \ref{sec:FGF}, the typical method is to minimize the following objective function
\begin{align}
\mathcal{L}(\gamma)&:=\E_{\hat{\mu}_t}[\frac{1}{2}\|f_\gamma\|_2^2-\langle \nabla\log \frac{\pi}{\hat{\mu}_t}, f_\gamma\rangle ]\\
&=\E_{\hat{\mu}_t}[\frac{1}{2}\|f_\gamma\|_2^2-(\nabla \log \pi)^{T}f_\gamma-\nabla\cdot f_\gamma].
\end{align}
The exact computation of the divergence term $\nabla\cdot f_\gamma$ requires $\mathcal{O}(d)$ back-propagations, where $d$ is the dimension of $x$.
While previous works use Hutchinson's estimator \citep{hutchinson1989stochastic} to reduce computation in high dimensions, this approach often suffers from large variance.
Taking advantage of the semi-implicit structure of $\hat{\mu}_t$, it is feasible to directly estimate $ \nabla\log \hat{\mu}_t$ using denoising score matching (DSM) \citep{vincent2011connection} by minimizing
\begin{align}
\mathcal{L}_{\mathrm{DSM}}(\gamma):&=\E_{\mu_t(z)}\E_{q_{\sigma}(x|z)}[\|f_{\gamma}(x)-\nabla_x\log q_{\sigma}(x|z)\|^2]\\
&=\E_{\mu_t(z)}\E_{q_{\sigma}(x|z)}[\|f_{\gamma}(x)+\frac{x-z}{\sigma^2}\|^2].
\end{align}
Compared to Hutchinson's estimator, DSM provides more stable and accurate gradient estimates while scaling efficiently in high dimensions.
This highlights the advantages of semi-implicit particle distributions for functional gradient computation.
We refer to our method as Semi-Implicit Functional Gradient Flow (SIFG) and outline the full sampling procedure in Algorithm \ref{alg:sifg}.

\begin{algorithm}[t]
    \caption{SIFG: semi-implicit functional gradient flow}
    \label{alg:sifg}
    \begin{algorithmic}
        \REQUIRE{Unnormalized target distribution $\pi$, initial particles $\{z_0^i\}_{i=1}^{n}$, initial parameter $\gamma_0$, iteration number $N, N'$, particle step size $h$, parameter step size $\eta$}, noise magnitude $\sigma$.
        \FOR{$k=0, \cdots, N-1$}
            \STATE{Assign $\gamma_k^0 = \gamma_k$}
            \STATE{Obtain perturbed samples $x_k^i = z_k^i+\epsilon_k^i$}, where $\epsilon_k^i\sim \mathcal{N}(0,\sigma^2)$
            \FOR{$t=0, \cdots, N'-1$}
                \STATE{Compute 
                    \begin{equation}\label{eq:semidsm}
                        \widehat{\mathcal{L}}(\gamma) = \frac{1}{n}\sum_{i=1}^{n} \|f_\gamma(x_k^i)-\frac{x_k^i-z_k^i}{\sigma^2}\|^2 
                    \end{equation}
                }
                \STATE{Update $\gamma_k^{t+1} = \gamma_k^{t} - \eta \nabla_\gamma \widehat{\mathcal{L}}(\gamma_k^t)$}
            \ENDFOR 
            \STATE{Update $\gamma_{k+1} = \gamma_k^{N'}$}
            \STATE{Update particles $z_{k+1}^{i} = z_{k}^{i} + h(\nabla \log{\pi(x_k^i)}-f_{\gamma_{k+1}}(x_{k}^{i}))$ for $i=1, \cdots, n$}
        \ENDFOR
        \STATE{Obtain perturbed samples $x_N^i = z_N^i+\epsilon_N^i$}, where $\epsilon_N^i\sim \mathcal{N}(0,\sigma^2)$
        \RETURN{Particles $\{x_N^i\}_{i=1}^{n}$}
    \end{algorithmic}
\end{algorithm}

\subsection{An Adaptive Procedure for the Noise Magnitude}\label{subsec:method3}

The SIFG framework also allows adaptive tuning of the variance parameter $\sigma$ in the transition kernel, which helps to balance between sample accuracy and diversity.
Concretely, increasing $\sigma$ can enhance sample diversity but may also reduce 
the approximation ability of the variational family.
To further accelerate convergence to the target distribution, we update $\sigma$ to minimize the KL divergence.
This leads to a simple gradient descent procedure on $\sigma$ using the following gradient estimate:
\begin{align}
    \frac{d}{d \sigma} \mathcal{F}_{\sigma}(\mu_t)    
    &=\E _{z\sim \mu_t(z), w\sim \mathcal{N}(0,I)} \nabla \log \frac{\hat{\mu}_{\sigma,t}(z+\sigma w)}{\pi(z+\sigma w)}\cdot w \\ \label{adasifg:estimate}
    &\approx \E _{z\sim \mu_t(z), w\sim \mathcal{N}(0,I)} [f_{\gamma}(z+\sigma w)- \nabla \log \pi(z+\sigma w)]\cdot w,
\end{align}
where $f_\gamma$ is the estimated score function via DSM at time $t$.
Please refer to the Appendix \ref{appd:adasifg} for a detailed derivation.
In this way, we can adaptively adjust the noise level $\sigma$ on the fly.
To ensure the numerical stability of DSM, we constrain $\sigma$ within a reasonable range, $\sigma\in[\sigma_{\min},\sigma_{\max}]$ with $\sigma_{\min}>0, \sigma_{\max}<1$.
We call this adaptive version of SIFG, Ada-SIFG, and the full procedure can be found in Algorithm \ref{alg:adasifg} (Appendix \ref{appd:adasifg}).

\section{Theoretical Analysis}\label{sec:convergence_parvi}

In this section, we present the main theoretical results of SIFG, including convergence guarantee and sample complexity analysis.

\subsection{Optimization Guarantees}

We first analyze the optimization error in SIFG. The key step in our analysis is to establish a descent lemma for the functional gradient flow iterates.
Consider the discrete dynamics of the particles:
\begin{equation}
    z_{(k+1)h}=z_{ k h }+ h v_{k}(z_{k h}),
\end{equation}
along with its continuous interpolation:
\begin{equation}
    z_{(k+t)h}=z_{ k h }+ th v_{k}(z_{k h}), \text{ for } t\in[0,1],
\end{equation}
where $v_k$ is an approximation of the Wasserstein gradient $v_k^*(z)=\E_{q_{\sigma}(x|z)} [\nabla\log \pi(x)-\nabla\log \hat{\mu}_{kh}(x)]$ evaluated at time $kh$. 
To establish the descent lemma, it is essential to bound the derivatives of $\hat{\mu}_{kh}(x)$ and its score function $\nabla\log\hat{\mu}_{kh}(x)$.
For this purpose, we first present the following lemma.

\begin{lemma}\label{sec:lem_convolution}
    Let $\rho\in\mathcal{P}(\mathbb{R}^d)$ be a probability distribution with a finite second moment. Define the Gaussian transition kernel $k_\sigma(x) \propto e^{-\frac{\|x\|^2}{2\sigma^2}}$. Then for all $\theta\in \mathbb{R}^d$, $k\in \mathbb{N}_+$, there exist constants $\tilde{C}, \tilde{M}>0$ only depending on $k$, $\sigma$ and the finite second moment such that
    \begin{equation}
        \frac{\int \|\theta-y\|^k k_\sigma (\theta-y) d \rho(y)}{\int k_{\sigma}(\theta-y)d\rho(y)}\le \tilde{C}\|\theta\|^k+\tilde{M}
    \end{equation}
\end{lemma}

Please refer to Appendix 
\ref{appd:d_1_1} 
for a detailed proof. 
Lemma \ref{sec:lem_convolution} is useful for providing bounds on higher-order derivatives of the perturbed distribution and its score function, 
which serve as the foundation for our subsequent analysis.
For example, according to Lemma \ref{sec:lem_convolution}, when the second moment of $\mu_{kh}$ is finite, the growth rate of the norm of the score function $\nabla \log\hat{\mu}_{kh}(x)=\frac{1}{\sigma^2}\frac{\int (x-y) k_\sigma (x-y) d \mu_{kh}(y)}{\int k_{\sigma}(x-y)d\mu_{kh}(y)}$ is at most linear when taking $k=1$.
Notably, \cite{tom2024theoretical} also analyzes perturbed distributions under a transition kernel.
However, their analysis is limited to the cases where $\rho$ is a mixture of $n$ Diracs, rather than a general distribution with a regular density function.
Moreover, their approach does not fully utilize the tail information of the particle distribution since it does not incorporate higher-order moments (beyond the second moment). 
Consequently, the convergence rates of \cite{tom2024theoretical} depend on $n$ and may diverge to infinity as $n\rightarrow\infty$, restricting their analysis to finite mixtures of Gaussians.

Now assume $f_\gamma$ is the neural-net-estimated score function of the particle distribution via DSM, the velocity field then takes the form $v_k(z)=\E_{q_{\sigma}(x|z)} [\nabla\log \pi(x)-f_\gamma(x)]$.
To proceed with our analysis, we impose standard assumptions on the approximation accuracy of neural networks \citep{cheng2023gwg}, the regularity of the target distribution \citep{tom2024theoretical}, and some moment conditions on the sampling trajectory.

\begin{assumption} \label{assp:approx}
    For any $k$, $\E_{\hat{\mu}_{kh}} \| f_{\gamma} - \nabla \log \hat{\mu}_{kh}\|_2^2 \leq \varepsilon_k$.
\end{assumption}

\begin{assumption} \label{assp:lipscore}
    The score of the target distribution $\nabla\log \pi$ is $L$-Lipschitz, i.e., for any $x,y \in \mathbb{R}^d$, $\|\nabla\log \pi(x)-\nabla\log \pi(y)\|\le L\|x-y\|$.  
\end{assumption}

\begin{assumption} \label{assp:moment}
    For any $t$, the $\alpha$-th moment of $\mu_t$ is bounded by $m_{\alpha}$ for $\alpha\le 5$. i.e. $\int \|z\|^{\alpha} d\mu_t(z)\leq m_{\alpha}<\infty$.
\end{assumption}

\begin{assumption} \label{assp:bdvel}
    For any $z\in \mathbb{R}^d$,  $\|v_k(z)\|\le A\|z\|+B$.  
\end{assumption}

Assumption \ref{assp:approx} ensures the approximation accuracy of neural networks. 
By the triangular inequality, this assumption implies that 
\begin{equation}
    \E_{\mu_{kh}}\|v_k-v_k^*\|^2=\E_{\mu_{kh}}\|\E_{q_{\sigma}(x|z)}  (f_{\gamma}-\nabla\log\hat{\mu}_{kh})\|^2\le \varepsilon_k.
\end{equation}
In Section \ref{subsec:statgua}, we provide the neural network structure and the sample complexity required to satisfy this assumption.
Assumption \ref{assp:lipscore} addresses the smoothness of the target potential, a standard assumption in variational inference optimization convergence analysis \citep{korba2020non, tom2024theoretical}.
Assumption \ref{assp:moment} is about bounded moments, which is commonly used in stochastic optimization analysis \citep{moulines2011non, tom2024theoretical}.
We also empirically verify this assumption in our Bayesian neural network experiment.
Figure \ref{figure: bnn_moments} (Appendix \ref{sec:detail-bnn}) presents the fifth moments of the particle distributions across iterations for 6 benchmark datasets, based on 10 independent runs.
We see that Assumption \ref{assp:moment} holds for all tasks, i.e., the fifth moment of the particle distribution is bounded throughout the
discrete-time gradient flow.
Please refer to Section \ref{section:exp} and Appendix \ref{appd:exp_setup} for the detailed experimental setup.
Assumption \ref{assp:bdvel} follows \citet{cheng2023gwg} and measures the smoothness of neural nets. 
The rationality behind Assumption \ref{assp:bdvel} is based on the sub-linear growth rate property of the approximation target.

We now present our main results regarding the optimization of the energy functional $\hat{\mathcal{F}}$.

\begin{proposition}\label{prop:descent}
 Suppose Assumption \ref{assp:approx}, \ref{assp:lipscore}, \ref{assp:moment}, \ref{assp:bdvel}
hold. Then the following
inequality holds for $h<\frac{1}{A}$:
\[
\hat{\mathcal{F}}(\mu_{(k+1)h})-\hat{\mathcal{F}}(\mu_{k h})\le -\frac{1}{2}h\|\nabla_{W_2} \hat{\mathcal{F}}(\mu_{k h})\|^2_{L^2(\mu_{k h})}+\frac{1}{2}h\epsilon_k+h^2 [C m_{4}+M]+h^3[C m_{5}+M],
\]
where $C, M$ are constants that depend on $A, B, \sigma, m_5$ and do not depend on $k$ or $h$.
\end{proposition}

Please refer to Appendix 
\ref{append:d_1_2}
for a detailed proof.
This proposition shows that, when away from stationary points where the squared Wasserstein gradient norm exceeds the neural network approximation error, i.e., $\|\nabla_{W_2} \hat{\mathcal{F}}(\mu_{k h})\|^2_{L^2(\mu_{k h})}>\epsilon_k$, the energy functional $\hat{\mathcal{F}}$ decreases at each iteration for sufficiently small step size $h$.
Our result generalizes the findings of \cite{tom2024theoretical}, which were restricted to mixtures of Gaussians, to a broader class of semi-implicit distributions with Gaussian transition kernels.
A key distinction in Proposition \ref{prop:descent} is that, unlike \cite{tom2024theoretical}, the right-hand side now depends on the high-order moment terms (i.e., $m_4$ and $m_5$) rather than the number of mixture components $n$.
This is achieved by expanding the descent of the energy functional to the first order in time $t$ and using Lemma \ref{sec:lem_convolution} to bound the norms of higher-order derivatives.

As a corollary of the descent lemma, we obtain the convergence of the average of squared gradient norms along iterations
by taking a suitable step size $h$.

\begin{theorem}\label{thm:convergence}
    Assume Proposition \ref{prop:descent} holds, for suitable step size $h$ as a function of $\hat{\mathcal{F}}(\mu_0), K, C, M, m_4, m_5$, the average of squared gradient norms satisfies
 \[
 \frac{1}{K}\sum_{k=1}^{K} \|\nabla_{W_2} \hat{\mathcal{F}}(\mu_{k h})\|^2_{L^2(\mu_{k h})}\le \frac{R}{K^{\frac{1}{2}}}+\frac{S}{K^{\frac{2}{3}}}+\frac{1}{K}\sum_{k=1}^{K} \epsilon_k,
 \]
for
 \[
 K>\min\left\{\frac{A^2\hat{\mathcal{F}}(\mu_0)}{Cm_4+M}, \frac{A^3\hat{\mathcal{F}}(\mu_0)}{Cm_5+M}\right\},
 \]
where $R:=4\sqrt{\hat{\mathcal{F}}(\mu_0)(C m_{4}+M)}$ and $S:=4(\hat{\mathcal{F}}(\mu_0))^{\frac{2}{3}}(C m_{5}+M)^{\frac{1}{3}}$.
\end{theorem}

Please refer to Appendix 
\ref{append:d_1_2}
for a detailed proof.
Theorem \ref{thm:convergence} establishes that the convergence of the average of squared gradient norms depends on the iteration number $K$, the approximation error $\epsilon_k$, the initial energy $\mathcal{F}(\hat{\mu}_0)$, the linear rate $A$ and the moments $m_4, m_5$.
Taking $K\to \infty$ implies that the average of the squared gradient norms converges to the order of the average training error.
In contrast, \citet{tom2024theoretical} provide a bound for the average of the squared gradient norms as $\frac{M_n\hat{\mathcal{F}}(\mu_0)}{K}$, where $n$ is the number of Diracs and $M_n$ is a function of $n$ that goes to infinity when $n\to \infty$.
Because we include the additional higher-order error terms $h^2 [C m_{4}+M]+h^3[C m_{5}+M]$, the rate of $K$ drops from 1 to $\frac{1}{2}$ compared to \citet{tom2024theoretical}. However, by considering this higher-order expansion, we can replace $n$ with $m_4, m_5$, allowing our method to apply to cases where $n\to \infty$, e.g., the semi-implicit particle distributions, provided that those moments are finite.

\subsection{Statistical Guarantees}\label{subsec:statgua}

In this section, we focus on the sample complexity and the neural network complexity required to satisfy Assumption \ref{assp:approx}.
Specifically, we establish the statistical guarantees for score matching using a certain family of neural networks via empirical risk minimization (ERM).
We formally define the deep ReLU neural network family as
$\mathcal{S}(M,W,B,L,S):=
\{s(x)=(A_L\sigma(\cdot)+b_L)\circ\cdots\circ(A_1x+b_1): A_i\in\R^{d_i\times d_{i+1}}, b_i\in\R^{d_{i+1}}, \max d_i\leq W, \sum_{i=1}^L (\|A_i\|_0+\|b_i\|_0)\leq S, 
\max\|A_i\|_\infty\vee\|b_i\|_\infty\leq B, \sup_{x}\|s(x)\|_\infty\leq M\}$, where $\sigma(x)=\max\{x,0\}$.

We use $\mu$ and $\hat{\mu}$ to represent any $\mu_{kh}$ and $\hat{\mu}_{kh}$ in Assumption \ref{assp:approx}.
For $n$ independent samples $z_i, i=1,\ldots,M$, drawn from $\mu(z)$, we consider the following ERM problem:
\[
\hat{f_{\gamma}}:=\argmin_{f_{\gamma}\in\mathcal{S}}\frac{1}{n}\sum_{i=1}^n\ell(z_i;f_{\gamma}),
\]
where the loss function $\ell(z;f_{\gamma})$ is defined as
\[
\ell(z;f_{\gamma}):=\E_{q_{\sigma}(x|z)}\|f_{\gamma}(x)-\nabla\log q_{\sigma}(x|z)\|^2.
\]
The population loss is given by 
\[
\ell(f_{\gamma}):=\E_{\mu(z)q_{\sigma}(x|z)}\|f_{\gamma}(x)-\nabla\log q_{\sigma}(x|z)\|^2\\
=\E_{\hat{\mu}(x)}\|f_{\gamma}(x)-\nabla\log \hat{\mu}(x)\|^2+c_*=\ell_{sm}(f_{\gamma})+c_*,
\]
where
\[
\ell_{sm}(f_{\gamma}):=\E_{\hat{\mu}(x)}\|f_{\gamma}(x)-\nabla\log \hat{\mu}(x)\|^2.
\]
is the score matching loss, and $c_*=\E_{\mu(z)q_{\sigma}(x|z)}\|\nabla\log\hat{\mu}(x)-\nabla\log q_\sigma(x|z)\|^2$ is a constant independent of $f_{\gamma}$.
This means that minimizing the population loss is equivalent to minimizing the score matching loss.

\subsubsection{The Bounded Moment Case}
We first proceed with the bounded moment Assumption \ref{assp:moment}, which is weaker than the sub-Gaussian assumption commonly used in complexity analysis of score matching \citep{fu2024unveil, chen2023lowdim}.
To analyze the sample error of the approximation, we follow standard procedures by bounding the empirical Rademacher complexity.

\begin{theorem}\label{thm:generalization}
    Under 
    Assumption \ref{assp:moment}, 
    for any $\delta>0$, it holds with probability no less than $1-2\delta$ that, 
    \begin{equation}
        \ell_{sm}(\hat{f_{\gamma}})
        \leq 2\inf_{f_{\gamma}\in\mathcal{S}}\ell_{sm}(f_{\gamma}) + \mathcal{O}\left(\sqrt{\frac{M^2+\frac{d}{\sigma^2}}{n}\log\mathcal{N}}\right).
    \end{equation}
    where 
    $\log\mathcal{N}=SL\log\big(WL(B\vee 1)\sigma(n/\delta)\log{d}\big)$.
\end{theorem}

Please refer to Appendix 
\ref{appd:d_2_1} 
for a detailed proof. 
The second term on the right hand side of Theorem \ref{thm:generalization} quantifies the sample error as a function of the sample complexity $n$. Theorem \ref{thm:generalization} indicates that the sample complexity depends on the scale parameters of the neural network, specifically $S$ and $L$.

Next, we turn our attention to the approximation error, which is the infimum of the score matching loss, $\inf_{f_{\gamma}\in\mathcal{S}}\ell_{sm}(f_{\gamma})$, for neural networks.
The following theorem presents the approximation complexity bounds for this error, specifically for ReLU neural networks.

\begin{theorem}\label{thm::nn approx}
Under Assumption \ref{assp:moment}, for sufficiently large $N$,
there exists $f_{\gamma} \in \mathcal{S}(M,W,B,L,S)$ such that 
\begin{align*}
    \ell_{sm}(f_{\gamma})=\mathcal{O}(\sigma^{-7+o(1)}N^{-\frac{3-o(1)}{d}}).
\end{align*}
The hyperparameters in the ReLU neural network class $\mathcal{S}$ satisfy
\begin{align*}
    &M = \mathcal{O}(\sigma^{-1+o(1)}N^{\frac{1-o(1)}{d}}),~
     W = \mathcal{O}(N),~
     B =\mathcal{O}(\sigma^{(2d+15-o(1))}N^{\frac{(2d+15-o(1))}{d}}),~\\
    &L = \mathcal{O}(\log^2(\sigma^{-1})\log^2 N),~
     S= \mathcal{O}\left(\log(\sigma^{-1})N\right),
\end{align*}
where $\mathcal{O}$ hides the other factors that are not related to $N$ or $\sigma$.
\end{theorem}

Theorem \ref{thm::nn approx} states that the approximation error scales as $N^{-\frac{3-o(1)}{d}}$ for a given neural network size $N$.
For more general moment conditions with finite $p$-th moment ($p>2\in\mathbb{N}_+$), the approximation error scales as 
$N^{-\frac{p-2-o(1)}{d}}$. 
Please refer to Appendix
\ref{appd:d_3_2} 
for a detailed proof.

Combining Theorem \ref{thm:generalization} and Theorem \ref{thm::nn approx}, we can derive the overall complexity required for Assumption \ref{assp:approx} to hold. Please refer to Appendix \ref{appd:proof4} for a detailed proof.

\begin{corollary}\label{coro1}
    To guarantee that Assumption \ref{assp:approx} holds with $\ell_{sm}(\hat{f_{\gamma}})\le \epsilon_k=\mathcal{O}(\epsilon)$, under Assumption \ref{assp:moment}, it suffices to require that the neural network complexity $N=\mathcal{O}(\sigma^{-\frac{7d}{3-o(1)}}\epsilon^{-\frac{d}{3-o(1)}})$, and the sample size $n=\mathcal{O}(\sigma^{-\frac{7d+20}{3-o(1)}}\epsilon^{-\frac{d+8}{3-o(1)}})$, where $\mathcal{O}$ hides the other terms that are not related to $\epsilon$ or $\sigma$.
\end{corollary}

\subsubsection{The Sub-Gaussian Case}

In this section, we demonstrate that under a stronger sub-Gaussian assumption, we can derive improved results for both sample and approximation complexity.
\begin{assumption} \label{assp:subg}
    For any $t$ and $Z\sim \mu_t$, there exists an absolute constant $C'>0$, such that $P(\|Z\|>R)\le 2\exp(-C' R^2)$ for all $R$.
\end{assumption}

\begin{theorem}\label{thm:generalization2}
    Under 
    Assumption \ref{assp:subg}, 
    for any $\delta>0$, it holds with probability no less than $1-2\delta$ that, 
    \begin{equation}
        \ell_{sm}(\hat{f_{\gamma}})
        \leq 4\inf_{f_{\gamma}\in\mathcal{S}}\ell_{sm}(f_{\gamma}) + \mathcal{O}\left(\frac{M^2}{n}\log\mathcal{N}\right).
    \end{equation}
    where   
    $\log\mathcal{N}=\log^3(nd/\sigma\delta)[\log(1/\delta)+SL\log(nLW(B\vee 1)M\log(1/\delta))]$.
\end{theorem}

Please refer to Appendix 
\ref{appd:d_2_2} 
for a detailed proof.
Theorem \ref{thm:generalization2} shows that under the sub-Gaussian assumption, the generalization bound improves from $n^{-\frac{1}{2}}$ to $n^{-1}$ if we ignore the logarithmic terms.
Moreover, the neural network approximation error bound also improves.

\begin{theorem}\label{thm::nn approx2}
Under Assumption \ref{assp:subg}, for sufficiently large $N$, there exists $f_{\gamma} \in \mathcal{S}(M,W,B,L,S)$ such that 
\begin{align*}
    \ell_{sm}(f_{\gamma})=\mathcal{O}(\sigma^{-4}\exp(-C'\sigma^{\frac{d-2}{d+1}} N^{\frac{d-2}{d(d+1)}})).
\end{align*}
The hyperparameters in the ReLU neural network class $\mathcal{S}$ satisfy
\begin{align*}
    &M = \mathcal{O}(\sigma^{-\frac{3(d+2)}{2(d+1)}}N^{\frac{d-2}{2d(d+1)}}),~
     W = \mathcal{O}(N),~
     B =\mathcal{O}(\sigma^{\frac{3\sigma^{\frac{d}{d+1}}N^{\frac{1}{d+1}}+3d^2}{d}}N^{\frac{3\sigma^{\frac{d}{d+1}}N^{\frac{1}{d+1}}}{d(d+1)}}),~
    L = \mathcal{O}(\sigma^{\frac{2d}{d+1}}N^{\frac{2}{d+1}}),\\
    ~
     &S= \mathcal{O}(\sigma^{\frac{d}{d+1}}N^{\frac{d+2}{d+1}}).
\end{align*}
where $\mathcal{O}$ hides the other factors that are not related to $N$ or $\sigma$.
\end{theorem}

Please see Appendix 
\ref{appd:d_3_2} 
for a detailed proof. 
Combining Theorem \ref{thm:generalization2} and Theorem \ref{thm::nn approx2}, we obtain the overall complexity needed for Assumption \ref{assp:approx} to hold.

\begin{corollary}\label{coro2}
    To guarantee that Assumption \ref{assp:approx} holds with $\ell_{sm}(\hat{f_{\gamma}})\le \epsilon_k=\mathcal{O}(\epsilon)$, under Assumption \ref{assp:subg}, it suffices to require that the neural network complexity $N=\mathcal{O}(\sigma^{-d}\text{polylog}(\epsilon^{-1}))$, and the sample complexity $n=\mathcal{O}(\sigma^{-d}\text{polylog}(\epsilon^{-1}) \epsilon^{-1})$, where $\mathcal{O}$ hides the other factors that are not related to $\epsilon$ or $\sigma$.
\end{corollary}

Please refer to Appendix \ref{appd:proof4} for a detailed proof. \cite{zhang2024minimax} study the sample complexity of score matching using kernel density estimation, which is also 
$\mathcal{O}(\sigma^{-d}\text{polylog}(\epsilon^{-1}) \epsilon^{-1})$
for $\sigma<1$. 
Corollary \ref{coro2} shows that our complexity bound is consistent with theirs. 

\section{Numerical Experiments}\label{section:exp}

In this section, we compare SIFG and Ada-SIFG with other ParVI methods including SVGD \citep{Liu2016SVGD}, $L_2$-GF \citep{di2021neural} and Ada-GWG \citep{cheng2023gwg} on both synthetic and real data problems.
In Bayesian neural network (BNN) experiments, we also compare to SGLD \citep{SGLD}.
Throughout this section, the initial particle distribution is the standard Gaussian distribution unless otherwise specified.
We parameterize $f_{\gamma}$ a 3-layer neural network and use SGD optimizer to train $f_{\gamma}$.
For Ada-SIFG, the noise magnitude $\sigma$ is clipped to be greater than $0.001$ unless otherwise specified.
Please refer to Appendix \ref{appd:exp_setup} for more detailed setups of our experiments.
The code is available at \url{https://github.com/ShiyueZhang66/Semi-Implicit_Functional_Gradient_Flow}.

\subsection{Gaussian Mixture}

\begin{figure*}[t]
   \centering
   \subfigure{
   \begin{minipage}[t]{0.32\linewidth}
   \centering
   \includegraphics[width=1\textwidth]{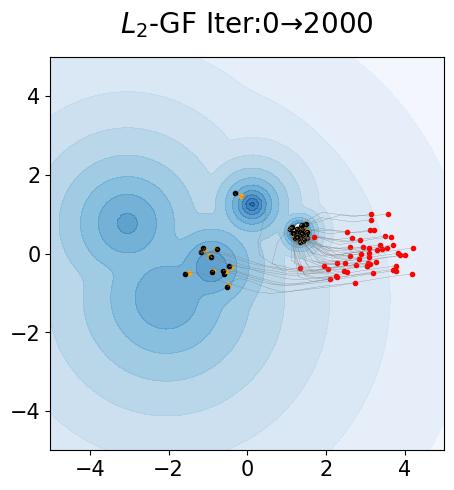}
   \end{minipage}%
   }%
   \hfill
   \subfigure{
   \begin{minipage}[t]{0.32\linewidth}
   \centering
   \includegraphics[width=1\textwidth]{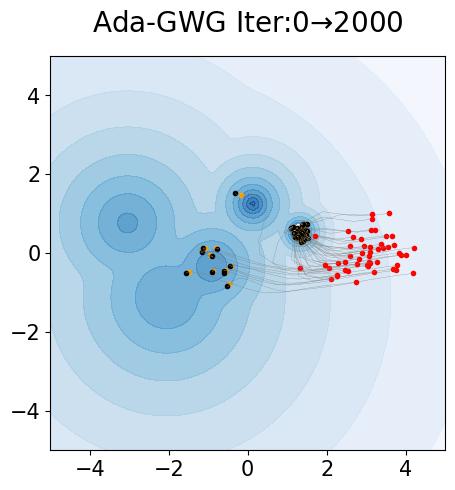}
   \end{minipage}%
   }%
   \hfill
   \subfigure{
   \begin{minipage}[t]{0.32\linewidth}
   \centering
   \includegraphics[width=1\textwidth]{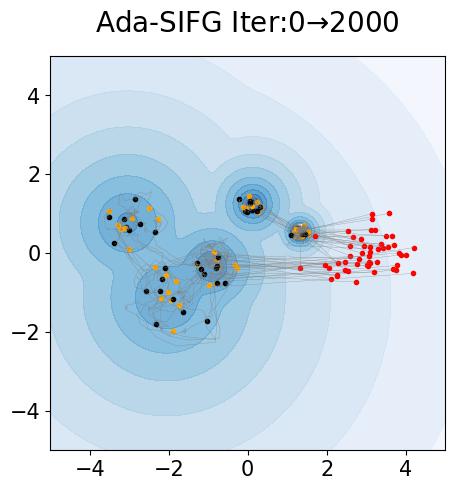}
   \end{minipage}%
   }%
   \centering
   \caption{The trajectory of particle movements during the first 2000 iterations for different methods. The red dot represents the initial location, the orange dot shows the particle location at iteration 1600 and the black dot shows the location at iteration 2000. We randomly selected 50 particles for illustration.
   }
   \label{figure:toy2ddiv}
\end{figure*}

We first test the sampling efficiency and sample diversity of different methods on toy Gaussian mixture models.
Specifically, we consider a 5-cluster, 2-dimensional Gaussian mixture model with equal weights for each component, where the standard deviations are $0.1, 0.2, 0.3, 0.4, 0.5$, respectively.
All methods share the same initial particle positions, which are sampled from the Gaussian distribution $\mathcal{N}((3,0),0.25\mathbf{I}_2)$.
The number of particles is set to 1000.
Figure \ref{figure:toy2dexpl} shows the scatter plots of the sampled particles from different methods.
We see that both $L_2$-GF and Ada-GWG get stuck at modes near the initial positions, whereas SIFG successfully discovers all the modes due to the improved exploration ability from the semi-implicit construction.
To further highlight the exploration ability and the sample diversity of SIFG compared to the other methods, we randomly select 50 particles of each method and trace their trajectories from iteration 0 to 2000.
The particles at iteration 1600, where all methods converge in distribution, are marked in orange.
From Figure \ref{figure:toy2ddiv}, we can see that the particles from $L_2$-GF and Ada-GWG barely after iteration 1600, while SIFG continue to generate diverse particles that match the target distribution well.

To evaluate the approximation accuracy of different methods and investigate the effectiveness of the adaptive procedure, we also consider a 5-cluster, 10-dimensional Gaussian mixture model with equal weights for each component, and standard deviations of $0.1, 0.2, 0.3, 0.4, 0.5$, respectively.
The number of particles is set to 1000, and the initial noise magnitude for both SIFG and Ada-SIFG is $0.1$.
The left plot of Figure \ref{figure: samples-toy} shows the KL divergence for different methods as a function of the number of iterations
in 5 independent runs.
We see that with a fixed noise level 0.1, SIFG converges significantly faster than $L_2$-GF, and performs similarly--though faster at the start--compared to Ada-GWG, which uses an adaptive procedure.
When the adaptive procedure is enabled, Ada-SIFG further improves upon SIFG and outperforms Ada-GWG.

\subsection{Monomial Gamma}

\begin{figure}[!t]
    \centering
    \includegraphics[width=0.48\linewidth]{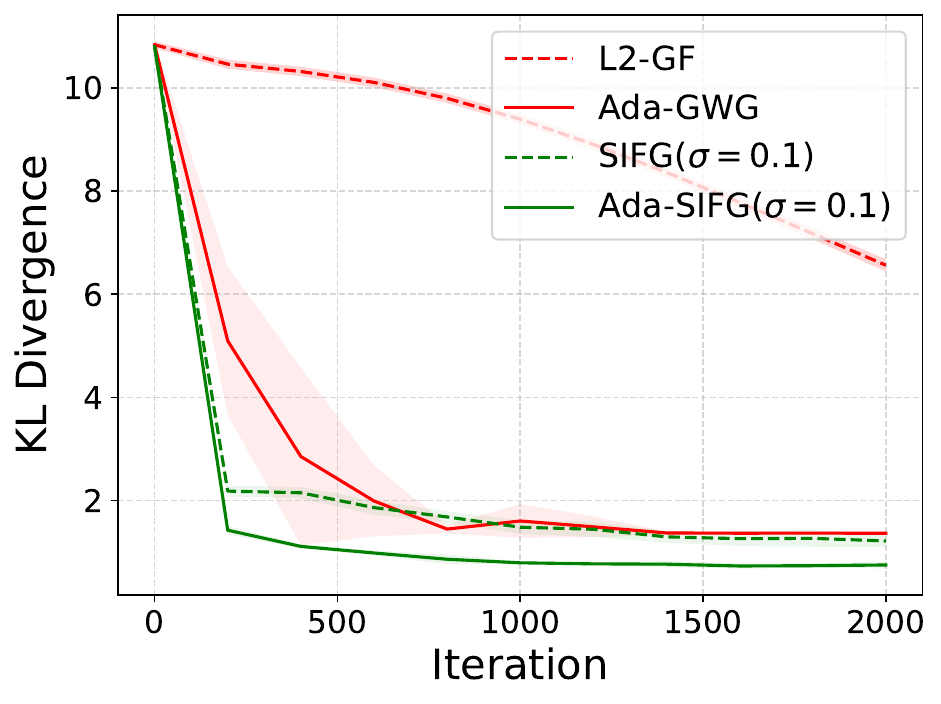}
    \includegraphics[width=0.49\linewidth]{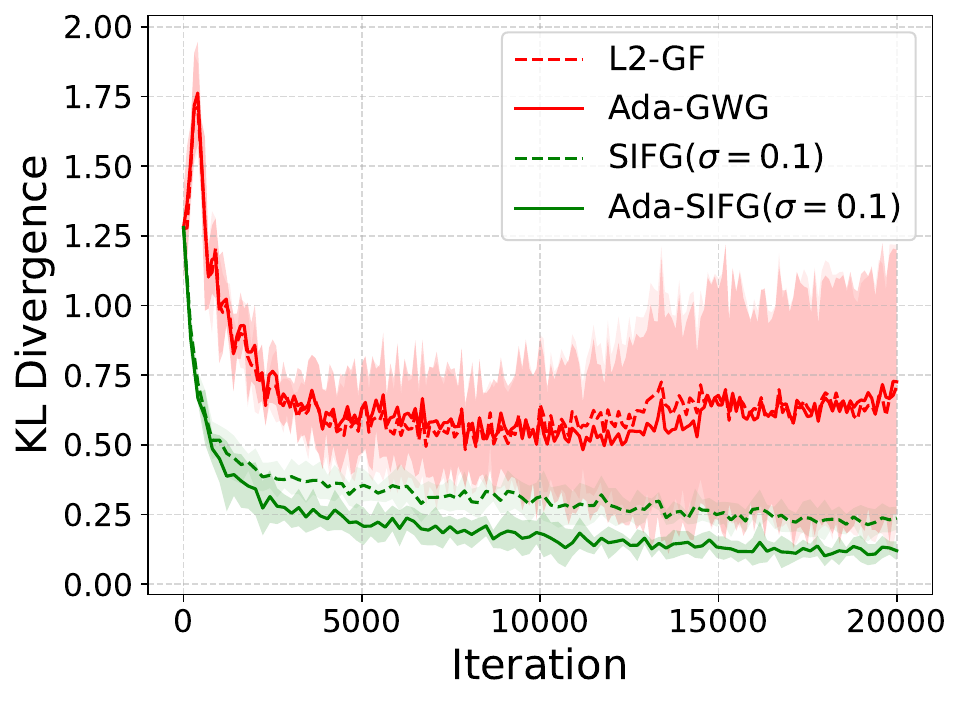}  
    \caption{KL divergence of different methods versus the number of iterations. $\textbf{Left:}$ Gaussian mixture distribution. $\textbf{Right:}$ Monomial gamma distribution.
    }
    \label{figure: samples-toy}
\end{figure}

To further investigate the exploration advantage of SIFG, we follow \citet{cheng2023gwg} and consider the heavy-tailed 2-dimensional Monomial Gamma distribution, where the target distribution is given by $\pi(x_1,x_2) \propto \exp(-0.3(|x_1|^{0.9}+|x_2|^{0.9}))$.
The number of particles is set to 1000, and the initial noise magnitude for both SIFG and Ada-SIFG is $0.1$.
The right plot of Figure \ref{figure: samples-toy} shows the KL divergence as a function of the number of iterations.
We see that SIFG converges faster than both $L_2$-GF and Ada-GWG, and with the adaptive procedure enabled, Ada-SIFG converges even faster than SIFG.
Furthermore, both SIFG and Ada-SIFG exhibit more stable convergence compared to $L_2$-GF and Ada-GWG, as evidenced by significantly smaller variances in the training curves.
Faster convergence suggests a stronger exploration ability, allowing the method to efficiently navigate the heavy-tailed distribution.
This advantage arises from the semi-implicit construction, which provides noisy gradient directions during training, while denoising score matching further reduces variance in the score function estimates, leading to more stable training.

\subsection{Bayesian Independent Component Analysis}

The task of independent component analysis (ICA) is to infer the unknown unmixing matrix $W\in \mathbb{R}^{d\times d}$, given the observations $x\in\mathbb{R}^d$, which is related to the latent sources $s\in\mathbb{R}^d$ through the equation $x=W^{-1}s$, where the components of $s$ are independent. Let $x=\{x_n\in \mathbb{R}^d\}$ be the observed samples. Assuming that each component $s_i$ follows the same distribution, $s_i\sim p_s$, the likelihood of the data is given by
\begin{equation*}
    p(x|W)=|\text{det} W| \prod_{i=1}^{d} p_s([Wx]_i).
\end{equation*}
Following \citet{korba2021ksdd}, we choose
$p_s$ such that $p'_s(\cdot)/p_s(\cdot) =-\text{tanh}(\cdot)$. The posterior distribution of $W$ given the observation $x$ is $p(W|x)\propto p(W)p(x|W)$.

We compare our methods, SIFG and Ada-SIFG, with SVGD, $L_2$-GF and Ada-GWG
on the MEG dataset \citep{ricardo2022meg}, which has 122 channels and 17730 observations.
Following \citet{zhang18}, we extract the first 5 channels, hence the dimension of matrix $W$ is 25. We choose a Gaussian prior over the unmixing matrix
$p(W_{ij})\sim \mathcal{N}(0,100)$.
To access convergence, we compute the Amari distance \citep{amari1995} between the ground truth $W_0$ and the estimates generated by different methods, where the ground truth is obtained from a long run ($10^5$ samples) of Hamiltonian Monte Carlo. 
To evaluate sample efficiency, we consider two settings with 10 and 100 particles. Each experiment is repeated 50 and 5 times, respectively, yielding 500 samples per method.
The results are reported in the left and middle plots of Figure \ref{figure:bnns}.
We see that SIFG outperforms $L_2$-GF and Ada-GWG in both settings and remains robust even with fewer particles. 
In contrast, SVGD is relatively more affected by small particle sizes. 

\begin{figure*}[t]
   \centering
   \subfigure{
   \begin{minipage}[t]{0.32\linewidth}
   \centering
   \includegraphics[width=1\textwidth]{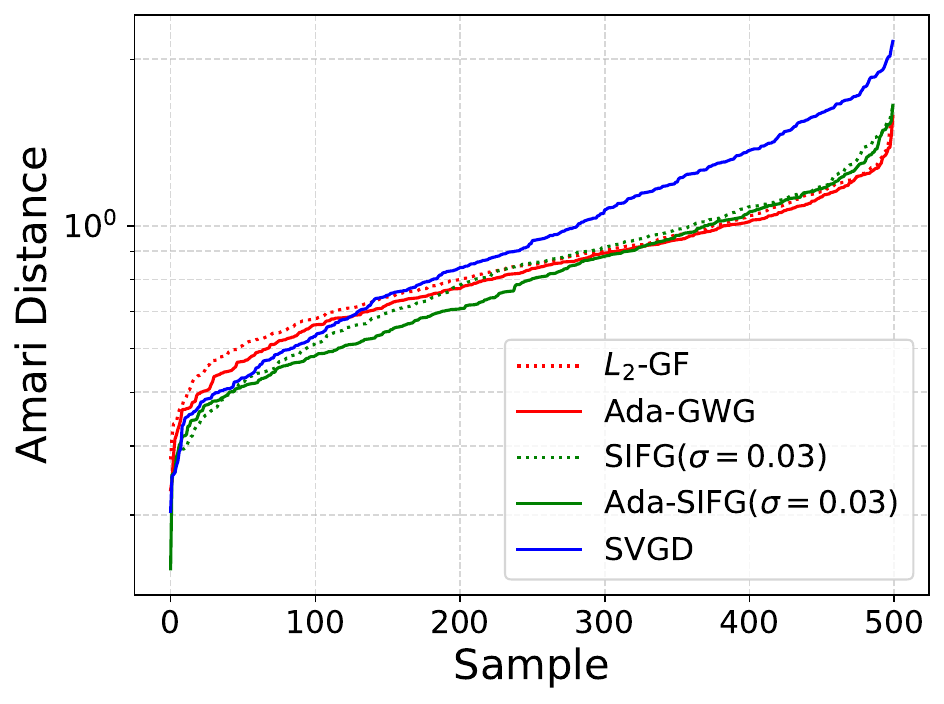}
   \end{minipage}%
   }%
   \hfill
   \subfigure{
   \begin{minipage}[t]{0.32\linewidth}
   \centering
   \includegraphics[width=1\textwidth]{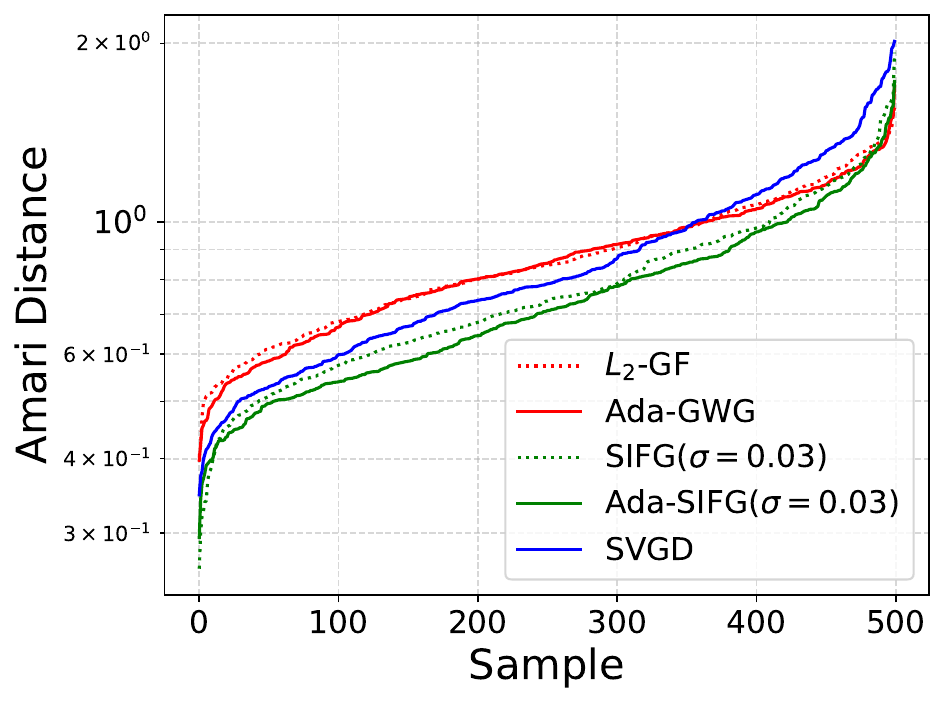}
   \end{minipage}%
   }%
   \hfill
   \subfigure{
   \begin{minipage}[t]{0.32\linewidth}
   \centering
   \includegraphics[width=1\textwidth]{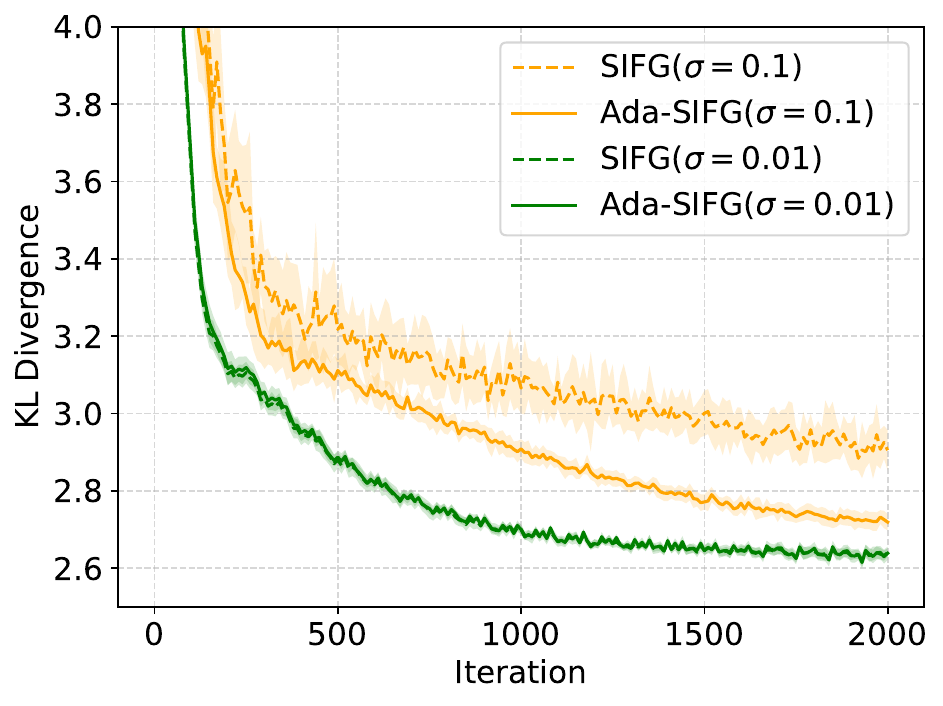}
   \end{minipage}%
   }%
   \centering
   \caption{$\textbf{Left and Middle:}$ Amari distances of different methods on MEG dataset. On the left is the experiment for 10 particles and 50 random repetitions. In the middle is the experiment for 100 particles and 5 random repetitions. $\textbf{Right:}$ Test RMSE for BNN on Boston dataset. The number in parentheses
specifies the initial value of $\sigma$.
   }
   \label{figure:bnns}
\end{figure*}

\subsection{Bayesian Neural Networks}

Our final experiment is on Bayesian neural networks (BNN), where we compare our algorithm with SGLD, SVGD, $L_2$-GF and Ada-GWG.
Following \citet{cheng2023gwg}, we use a two-layer network with 50 hidden units and ReLU
activation, with a $\mathrm{Gamma}(1, 0.1)$ prior for the inverse covariances.
Each dataset is randomly split into $90\%$ for training and $10\%$ for testing, and all methods use a mini-batch size of 100.
The results are averaged over 10 independent trials.

Table \ref{tab:bnn_results} reports the average test rooted mean square error (RMSE) and negative log-likelihood (NLL), along with their standard deviations.
We see that Ada-SIFG achieves comparable or better results than the other methods.
The right plot in Figure \ref{figure:bnns} shows the test RMSE over iterations for both SIFG and Ada-SIFG on the Boston dataset.
Specifically, we observe that $\sigma=0.01$ produces better results than $\sigma=0.1$.
While the latter is sub-optimal, our adaptive method Ada-SIFG significantly improves performance, converging towards the optimal setting of $\sigma=0.01$.   
This demonstrates that the adaptive scheme effectively adjusts the suboptimal noise magnitude during training, enhancing robustness.
Additional experimental details can be found in Appendix \ref{appd:exp_setup}.

\begin{table*}[t]
\caption{Test RMSE and test NLL of Bayesian neural networks on several UCI datasets. The results are averaged from 10 independent runs with standard deviations in subscripts.
  }
\label{tab:bnn_results}
\centering
\vskip2em
\setlength\tabcolsep{6.3pt}
\resizebox{\linewidth}{!}{
\begin{tabular}{lcccccccccc}
\toprule
 \multirow{2}{*}{Dataset}&\multicolumn{5}{c}{Test RMSE ($\downarrow$)} &\multicolumn{5}{c}{Test NLL ($\downarrow$)}  \\
\cmidrule(l){2-6}\cmidrule(l){7-11}
& SGLD& SVGD& $L_2$-GF& Ada-GWG& Ada-SIFG& SGLD& SVGD& $L_2$-GF& Ada-GWG& Ada-SIFG\\
\midrule
\textsc{Boston}        &$2.917_{\pm 0.04}$ & $2.944_{\pm 0.03}$& $ 3.016_{\pm 0.12}$ & $2.662_{\pm 0.06}$ & $\bm{2.641}_{\pm 0.02}$&$2.563_{\pm 0.01}$& $2.567_{\pm 0.01}$& $2.706_{\pm 0.04}$    & $2.680_{\pm 0.04}$ & $\bm{2.496}_{\pm 0.02}$   \\
\textsc{Concrete}     &$8.218_{\pm 0.27}$ & $7.640_{\pm 0.03}$& $7.556_{\pm 0.08}$ & $\bm{6.590}_{\pm 0.09}$ & $6.619_{\pm 0.03}$& $3.539_{\pm 0.02}$ & $3.455_{\pm 0.01}$&$3.449_{\pm 0.01}$    & $3.341_{\pm 0.01}$ & $\bm{3.323}_{\pm 0.01}$  \\
\textsc{Diabetes}      &$0.382_{\pm0.00}$ & $0.386_{\pm 0.00}$& $0.381_{\pm 0.00}$ & $0.384_{\pm 0.00}$ & $\bm{0.379}_{\pm 0.00}$ &$0.466_{\pm0.00}$& $0.484_{\pm 0.00}$& $0.471_{\pm 0.00}$    & $0.479_{\pm 0.00}$ & $\bm{0.449}_{\pm 0.00}$ \\
\textsc{Power}       &$4.176_{\pm 0.01}$ & $4.940_{\pm 0.00}$&$4.258_{\pm 0.06}$ & $ 4.074_{\pm 0.05}$ & $\bm{4.017}_{\pm 0.00}$& $2.863_{\pm 0.00}$ & $3.016_{\pm 0.00}$& $3.070_{\pm 0.02}$    & $2.967_{\pm 0.01}$ & $\bm{2.829}_{\pm 0.00}$   \\
\textsc{Protein}      &$\bm{4.609}_{\pm 0.00} $  &$4.940_{\pm 0.00}$ &$5.024_{\pm 0.02}$ & $4.967_{\pm 0.01}$ & $4.863_{\pm 0.00}$&$\bm{2.946}_{\pm 0.00}$ & $3.016_{\pm 0.00}$ & $3.075_{\pm 0.02}$    & $3.040_{\pm 0.01}$ & $2.999_{\pm 0.00}$  \\
\textsc{Wine}     &$0.431_{\pm 0.00}$ & $0.422_{\pm 0.00}$& $0.419_{\pm 0.00}$ & $0.416_{\pm 0.00}$ & $\bm{0.413}_{\pm 0.00}$& $0.576_{\pm 0.00}$& $0.559_{\pm 0.00}$& $0.552_{\pm 0.01}$    & $0.552_{\pm 0.00}$ & $\bm{0.535}_{\pm 0.00}$  \\
\bottomrule
\end{tabular}
}
\end{table*}

\section{Conclusion}\label{sec:concl}

We introduce a novel functional gradient flow method, Semi-Implicit Functional Gradient Flow (SIFG), which leverages a semi-implicit variational family represented by particles.
In SIFG, Gaussian noise is injected into the particles before updating them according to the gradient flow, promoting better exploration and diversity in the approximation.
Under accurate gradient flow estimation and mild smoothness conditions, we establish a convergence guarantee for SIFG that generalizes prior results on variational inference with finite Gaussian mixtures \citep{tom2024theoretical}.
Furthermore, we derive the sample complexity of the neural network-based gradient flow estimation within the empirical risk minimization (ERM) framework.
We also introduced an adaptive version, called Ada-SIFG, which can automatically tune the noise magnitude to improve sampling quality.
Extensive numerical experiments demonstrate that Ada-SIFG consistently outperforms existing particle-based variational inference (ParVI) methods.

Several promising directions remain for further exploration. Beyond Gaussian kernels, alternative kernels that better capture the geometry of the target distribution—especially in constrained domains—could be investigated. In this context, SIFG is closely related to the mirror Wasserstein gradient flow introduced by \citet{bonet2024mirror}, and a deeper exploration of this connection may yield valuable insights.
Additionally, our theoretical analysis relies on certain regularity assumptions about the sampling trajectory. While such assumptions are commonly made in the literature and supported by empirical evidence, a rigorous justification requires further investigation into the partial differential equations (PDEs) governing the law of the corresponding gradient flows.

\section*{Acknowledgements}
This work was supported by National Natural Science Foundation of China (grant no. 12201014 and grant no. 12292983).
The research of Cheng Zhang was supported in part by National Engineering Laboratory for Big Data Analysis and Applications, the Key Laboratory of Mathematics and Its Applications (LMAM) and the Key Laboratory of Mathematical Economics and Quantitative Finance (LMEQF) of Peking University. 
The authors are grateful for the computational resources provided by the High-performance Computing Platform of Peking University.

\bibliography{Bibliography}
\bibliographystyle{apalike}

\newpage
\appendix

\section{The Geometry of the Gradient Flows}\label{app:wgf_basics}

In this section, we first develop the geometric framework of gradient flows by introducing the general definition of a gradient system on a manifold. We then provide examples of gradient flows with different geometries, including the well-known Wasserstein gradient flow.

A gradient flow describes a dynamical system that is driven towards the fastest
dissipation of a certain energy functional. The dissipation potential depicts the geometry of this gradient flow.
Given a smooth state space $X$, a dissipation potential $\mathfrak{R}=\mathfrak{R}(\mu,\dot{\mu})$ is a
function on the tangent bundle $TX:=\{(\mu,\dot{\mu})|\mu\in X,\dot{\mu}\in T_{\mu}X\}$ that satisfies the following conditions: (i) $\mathfrak{R}(\mu,\cdot)$ is non-negative and convex, and (ii) $\mathfrak{R}(\mu,0)=0$. 
If $\mathfrak{R}(\mu,\dot{\mu})=\frac{1}{2}\langle\mathbb{G}(\mu)\dot{\mu},\dot{\mu}\rangle$, where $\mathbb{G}$ is the Riemannian tensor, we say that $\mathfrak{R}(\mu,\cdot)$ is quadratic. 
In practice, the dissipation potential is usually complex, making direct analysis challenging. To address this, we introduce the dual dissipation $\mathfrak{R}^*(\mu,\xi)$, which describes the local geometry in terms of its dual formulation:
\begin{equation}
     \mathfrak{R}^*(\mu, \xi) = \sup \{ \langle \xi, v \rangle -\mathfrak{R}(\mu, v) | v \in T_{\mu}X\}.
\end{equation}
For quadratic $\mathfrak{R}(\mu,\dot{\mu})=\frac{1}{2}\langle\mathbb{G}(\mu)\dot{\mu},\dot{\mu}\rangle$, the corresponding dual dissipation potential takes the form $\mathfrak{R}^*(\mu, \xi)=\frac{1}{2}\langle\xi, \mathbb{K}(\mu)\xi\rangle$, where $\mathbb{K}=\mathbb{G}^{-1}$ is the Onsager operator.

We now present the formal definition of the gradient system \citep{zhu2024rao}. 
\begin{definition}
    A triple $(X, \mathcal{F}, \mathfrak{R})$ is called a gradient system, if $X$ is a manifold or a subset of a Banach space, $\mathcal{F} : X\to\mathbb{R}$ is a differentiable function,
and $\mathfrak{R}$ is a dissipation potential. The associate gradient flow equation is
\begin{equation}\label{gsystem}
     \text{D}_{\dot{\mu}}\mathfrak{R}(\mu,\dot{\mu})+\text{D}_{\mu}\mathcal{F}(\mu)=0\Longleftrightarrow \dot{\mu}=\text{D}_{\xi}\mathfrak{R}^*(\mu,-\text{D}_{\mu}\mathcal{F}(\mu))
\end{equation}
If $\mathfrak{R}$ is quadratic, the gradient flow equations are
\begin{equation}\label{gradientflow_equation}
     \mathbb{G}(\mu)\dot{\mu}+\text{D}\mathcal{F}(\mu)=0\Longleftrightarrow  \dot{\mu}=-\mathbb{K}(\mu)\text{D}\mathcal{F}(\mu)
\end{equation}
\end{definition}
Equation (\ref{gsystem}) aligns with the definition of Euclidean space gradient flow. Under this gradient flow, the dissipation of the objective energy functional is 
$\dot{\mathcal{F}}=-\langle \text{D}\mathcal{F}(\mu), \mathbb{K}(\mu)\text{D}\mathcal{F}(\mu)\rangle=-2\mathfrak{R}^*(\mu,\text{D}\mathcal{F}(\mu))\le 0$, 
which is a basic property of the gradient system.

After establishing the general framework for the gradient system, we present several typical gradient flows as illustrative examples.

\begin{enumerate}[label=\textbf{Example \arabic*.}]

\item 
One of the most well-known is the 2-Wasserstein gradient flow \citep{benamou2000wass}.
The dissipation potential of 2-Wasserstein geometry is defined as $\mathfrak{R}_{W_2}(\mu,\dot{\mu})=\inf \{\frac{1}{2}\|v\|^2_{L^2(\mu)}:\dot{\mu}=-\nabla\cdot(\mu v)\}$. By Fenchel duality, the dual dissipation potential is
$\mathfrak{R}_{W_2}^*(\mu, \xi)=\frac{1}{2}\|\nabla\xi\|^2_{L^2(\mu)}=\frac{1}{2}\langle\xi, -\nabla\cdot(\mu\nabla\xi)\rangle$.
The Onsager operator of 2-Wasserstein gradient flow is $\mathbb{K}_{W_2}(\mu)\xi=-\nabla\cdot(\mu\nabla\xi)$. By equation (\ref{gradientflow_equation}), we have $\xi=\text{D}\mathcal{F}=\frac{\delta\mathcal{F}}{\delta\mu}$, $v=-\nabla\frac{\delta\mathcal{F}}{\delta\mu}:=-\nabla_{W_2}\mathcal{F}$, and
the 2-Wasserstein gradient flow equation is $\dot{\mu}=\nabla\cdot(\mu\nabla\frac{\delta\mathcal{F}}{\delta\mu})$.

\item 

The second example is the generalization of the 2-Wasserstein geometry-the GWG geometry proposed by \cite{cheng2023gwg}.
Here the dissipation potential is not quadratic. For a Young function $g$ \citep{cheng2023gwg}, define the dissipation potential of GWG geometry to be $\mathfrak{R}_{\text{GWG}}(\mu,\dot{\mu})=\inf \{\int g(v)d\mu:\dot{\mu}=-\nabla\cdot(\mu v)\}$. 
By Fenchel duality, the dual dissipation potential is $\mathfrak{R}_{\text{GWG}}^*(\mu, \xi)=\int g^*(\nabla \xi)d\mu$, where $g^*$ is the convex conjugate of $g$. Since $\int [g^*(\nabla (\xi+d\xi))-g^*(\nabla\xi)]d\mu=\langle \nabla g^*(\nabla\xi),\nabla d\xi\rangle_{L^2(\mu)}=\langle d\xi, -\nabla\cdot(\mu \nabla g^*(\nabla\xi))$, by equation (\ref{gsystem}), we have 
the generalized Wasserstein gradient flow equation is $\dot{\mu}=\nabla\cdot(\mu\nabla g^*(\nabla\frac{\delta \mathcal{F}}{\delta \mu}))$, $\xi=-\text{D}\mathcal{F}=-\frac{\delta\mathcal{F}}{\delta\mu}$, $v=-\nabla g^*(\nabla\frac{\delta \mathcal{F}}{\delta \mu})$, arriving at equation (\ref{eq:sm_full}).

\item The gradient flow interpretation of SVGD \citep{Liu2016SVGD} is SVGF \citep{Liu2017SVGF}. Its geometry is the kernelized Wasserstein geometry, also named as Stein geometry \citep{zhu2024rao}. The dissipation potential of Stein geometry is defined as $\mathfrak{R}_{\text{Stein}}(\mu,\dot{\mu})=\inf \{\frac{1}{2}\|\mathcal{K}_{\mu}v\|^2_{\mathcal{H}}:\dot{\mu}=-\nabla\cdot(\mu\cdot\mathcal{K}_{\mu}v)\}$. By Fenchel duality, the dual dissipation potential is $\mathfrak{R}_{\text{Stein}}^*(\mu, \xi)=\frac{1}{2}\|\mathcal{K}_{\mu}\nabla\xi\|^2_{\mathcal{H}}=\frac{1}{2}\langle \nabla\xi, \mathcal{K}_{\mu}\nabla\xi\rangle$, and the Onsager operator is $\mathbb{K}_{\text{Stein}}(\mu)\xi=-\nabla\cdot(\mu\mathcal{K}_{\mu}\nabla\xi)$. The gradient flow of SVGD is $\dot{\mu}=\nabla\cdot(\mu\mathcal{K}_{\mu}\nabla\frac{\delta \mathcal{F}}{\delta \mu})$.

\item Finally, we state that the two gradient flows 
in Section \ref{subsec:sifggf}:
\begin{equation}    \dot{\hat{\mu}}=\mathcal{K}\nabla\cdot(\mathcal{K}^{-1}\hat{\mu}\cdot\mathcal{K}\nabla\frac{\delta \mathcal{F}}{\delta \hat{\mu}})\Longleftrightarrow \dot{\mu}=-\nabla\cdot(\mu \nabla\frac{\delta \hat{\mathcal{F}}}{\delta \mu})
\end{equation}
where $\hat{\mathcal{F}}(\mu)=\mathcal{F}(\mathcal{K}\mu)$, 
$\hat{\mu}=\mathcal{K}\mu$,
can be regarded as an interconversion between standard Wasserstein gradient flow with the energy functional $\hat{\mathcal{F}}$ and the SIFG gradient flow with the energy functional $\mathcal{F}$. The kernel needs to satisfy $\langle f,-\mathcal{K}\nabla\cdot g\rangle=\langle \mathcal{K}\nabla f, g\rangle$, and the Gaussian kernel we used satisfies this property.
The dissipation potential of SIFG geometry is defined as $\mathfrak{R}_{\text{SIFG}}(\hat{\mu},\dot{\hat{\mu}})=\inf \{\frac{1}{2}\|v\|^2_{L^2(\mathcal{K}^{-1}\hat{\mu})}:\dot{\hat{\mu}}=-\nabla\cdot(\mathcal{K}^{-1}\hat{\mu}\cdot\mathcal{K}v)\}$. This potential is nonnegative, convex,
and satisfies $\mathfrak{R}_{\text{SIFG}}(\mu,0)=0$. By Fenchel duality, the dual dissipation potential is $\mathfrak{R}_{\text{SIFG}}^*(\hat{\mu}, \xi)=\frac{1}{2}\langle \xi,-\mathcal{K}\nabla\cdot(\mathcal{K}^{-1}\hat{\mu}\cdot\mathcal{K}\nabla\xi)\rangle= \frac{1}{2}\langle \mathcal{K}\nabla\xi,\mathcal{K}^{-1}\hat{\mu}\cdot\mathcal{K}\nabla\xi\rangle=\frac{1}{2}\|\mathcal{K}\nabla\xi\|^2_{L^2_{\mathcal{K}^{-1}\hat{\mu}}}$ (we use the property $\langle f,-\mathcal{K}\nabla\cdot g\rangle=\langle \mathcal{K}\nabla f, g\rangle$ at the second equation), and the Onsager operator is $\mathbb{K}_{\text{SIFG}}(\hat{\mu})\xi=-\mathcal{K}\nabla\cdot(\mathcal{K}^{-1}\hat{\mu}\cdot\mathcal{K}\nabla\xi)$. The dual dissipation potential $\mathfrak{R}_{\text{SIFG}}$ resembles $\mathfrak{R}_{\text{Stein}}^*(\mu, \xi)=\frac{1}{2}\|\mathcal{K}_{\mu}\nabla\xi\|^2_{\mathcal{H}}$.

\end{enumerate}

In conclusion, the general gradient system framework includes previous gradient flow methods and our new method, providing potential for designing new algorithms.

\section{Derivation of the Wasserstein Gradient Flow}\label{app:wgf_proof}

We prove Theorem \ref{mol_wgf} as the following.

\begin{theorem}
If the gradient of the transition kernel is skew-symmetric, i.e., $\nabla_x k(x,z)=-\nabla_z k(x,z)$, then the Wasserstein gradient flow of the energy functional $\mathcal{\hat{F}}$ is 
\begin{equation}
     \dot{\mu}_t(z)=-\nabla\cdot\Big(\mu_t(z) \E_{k(x, z)} \nabla\log \frac{\pi(x)}{\hat{\mu}_t(x)}\Big)=-\nabla\cdot(\mu_t \mathcal{K}\nabla\log \frac{\pi}{\mathcal{K}\mu_t}),
\end{equation}
and the Wasserstein gradient is $\nabla_{W_2}\mathcal{\hat{F}}(\mu_t)(z)=-\E_{k(x,z)} \nabla\log \frac{\pi(x)}{\hat{\mu}_t(x)}=-\mathcal{K}\nabla\log \frac{\pi}{\mathcal{K}\mu_t}$.

\end{theorem}

\begin{proof}

Consider the dynamics of the original particles
\begin{equation}
    \frac{d}{d t}z_{t}= v_{t}(z_{t}), \ z_{t}\sim \mu_t,
\end{equation}
The continuity equation is 
\begin{equation}\label{sifggradflow}
    \frac{d \mu_t}{d t}+ \nabla \cdot (\mu_t v_t)=0,
\end{equation}
By the chain rule, the time derivative of the KL divergence is
\begin{align}
    \frac{d \mathcal{F}}{d t}&=\langle \frac{\delta}{\delta \hat{\mu}_t }\mathcal{F}(\hat{\mu}_t), \frac{d \hat{\mu}_t}{d t}\rangle\\
    &=\langle \log \frac{\hat{\mu}_t(x)}{\pi(x)}, \int k(x,z)\frac{d\mu_t(z)}{dt} \rangle
\end{align}
Plugging in the continuity equation gives
\begin{align}
    \frac{d \mathcal{F}}{d t}=&\int \log \frac{\hat{\mu}_t(x)}{\pi(x)}\int k(x,z)\frac{d \mu_t(z)}{d t}dx\\
    =&-\int\log\frac{\hat{\mu}_t(x)}{\pi(x)}\int k(x,z)\nabla_z\cdot(\mu_t(z) v_t(z))dz dx\\
    =&\int\log\frac{\hat{\mu}_t(x)}{\pi(x)}\int \nabla_z k(x,z) \mu_t(z) v_t(z)dz dx\\
    =&\int \left [ \int\log\frac{\hat{\mu}_t(x)}{\pi(x)} \nabla_z k(x,z) dx  \right ]  v_t(z) \mu_t(z)dz \label{term:wass}
\end{align}
The dissipation property of the Wasserstein gradient
flow states that 
$\frac{d}{d t}\mathcal{\hat{F}}=\frac{d}{d t}\mathcal{F}=-\E_{\mu_t}\|\nabla_{W_2}\mathcal{\hat{F}}\|^2$, and the optimal velocity $v_t^*$ is the negative of the Wasserstein gradient \citep{Chewi2020chi-squared}. From (\ref{term:wass}) we have 
\begin{align}
v_t^*(z)
=&-\nabla_{W_2}\hat{\mathcal{F}}(\mu_t)(z)\\
=&-\int\log\frac{\hat{\mu}_t(x)}{\pi(x)} \nabla_z k(x,z) dx \\
=&\int\log\frac{\hat{\mu}_t(x)}{\pi(x)} \nabla_x k(x,z) dx \label{term:sym}\\
=&\int \nabla \log \frac{\pi(x)}{\hat{\mu}_t(x)} k(x,z) dx\label{term:result}\\
=&\E_{k(x,z)} \nabla\log \frac{\pi(x)}{\hat{\mu}_t(x)}\\
=&\mathcal{K}\nabla\log \frac{\pi}{\mathcal{K}\mu_t}
\end{align}
where the equation of (\ref{term:sym}) is guaranteed by the skew-symmetry of the gradient of the Gaussian transition kernel. 
Hence, the Wasserstein gradient flow equation for the energy functional $\mathcal{\hat{F}}$ is $ \dot{\mu}_t=\nabla\cdot(\mu_t \nabla_{W_2}\mathcal{\hat{F}}(\mu_t))=-\nabla\cdot(\mu_t \mathcal{K}\nabla\log \frac{\pi}{\mathcal{K}\mu_t})$.

\end{proof}

\section{Details of Ada-SIFG}\label{appd:adasifg}

We follow the derivation idea of \citet{luo2024diff} to prove the gradient estimate result (\ref{adasifg:estimate}).
\begin{align}
    \frac{d}{d \sigma} \mathcal{F}_{\sigma}(\mu_t) 
    &= \E _{z\sim \mu_t(z), w\sim \mathcal{N}(0,I)} \frac{\partial}{\partial \sigma} \log \frac{\hat{\mu}_{\sigma,t}(z+\sigma w)}{\pi(z+\sigma w)}\\
    &= \E _{z\sim \mu_t(z), w\sim \mathcal{N}(0,I)} \nabla \log \frac{\hat{\mu}_{\sigma,t}(z+\sigma w)}{\pi(z+\sigma w)}\cdot w +\E _{x\sim \hat{\mu}_{\sigma,t}} \frac{\partial}{\partial \sigma} \log \hat{\mu}_{\sigma,t}(x)\\ 
    &=\E _{z\sim \mu_t(z), w\sim \mathcal{N}(0,I)} \nabla \log \frac{\hat{\mu}_{\sigma,t}(z+\sigma w)}{\pi(z+\sigma w)}\cdot w \\ \label{appd:adasifgzero}
    &\approx \E _{z\sim \mu_t(z), w\sim \mathcal{N}(0,I)} [f_{\gamma}(z+\sigma w)- \nabla \log \pi(z+\sigma w)]\cdot w
\end{align}

where the equality of (\ref{appd:adasifgzero}) is because
\begin{equation}
    \E_{x\sim \hat{\mu}_{\sigma,t}} \frac{\partial}{\partial \sigma} \log \hat{\mu}_{\sigma,t}(x)=\int \frac{\partial}{\partial \sigma} \hat{\mu}_{\sigma,t}(x) dx=\frac{\partial}{\partial \sigma}  \int \hat{\mu}_{\sigma,t}(x) dx=\frac{\partial}{\partial \sigma} 1=0.
\end{equation}

By this gradient descent scheme on the noise magnitude $\sigma$, we present the full procedure of Ada-SIFG in Algorithm \ref{alg:adasifg}.

\begin{algorithm}[!thbp]
    \caption{Ada-SIFG: adaptive semi-implicit functional gradient flow}
    \label{alg:adasifg}
    \begin{algorithmic}
        \REQUIRE{Unnormalized target distribution $\pi$, initial particles $\{z_0^i\}_{i=1}^{n}$, initial parameter $\gamma_0$, iteration number $N, N'$, particle step size $h$, parameter step size $\eta$}, initial noise magnitude $\sigma_0$, noise magnitude step size $\Tilde{\eta}$, lower and upper bounds on $\sigma$: $lb, ub$.
        \FOR{$k=0, \cdots, N-1$}
            \STATE{Assign $\gamma_k^0 = \gamma_k$}
            \STATE{Obtain perturbed samples $x_k^i = z_k^i+\epsilon_k^i$}, where $\epsilon_k^i\sim \mathcal{N}(0,\sigma_k^2)$
            \FOR{$t=0, \cdots, N'-1$}
                \STATE{Compute 
                    \begin{equation}
                        \widehat{\mathcal{L}}(\gamma) = \frac{1}{n}\sum_{i=1}^{n} \|f_\gamma(x_k^i)-\frac{x_k^i-z_k^i}{\sigma_k^2}\|^2 
                    \end{equation}
                }
                \STATE{Update $\gamma_k^{t+1} = \gamma_k^{t} + \eta \nabla_\gamma \widehat{\mathcal{L}}(\gamma_k^t)$}
            \ENDFOR 
            \STATE{Update $\gamma_{k+1} = \gamma_k^{N'}$}
            
            \STATE{Compute $\widehat{grad}(\sigma_k)=\frac{1}{n} \sum_{i=1}^{n} [\nabla \log{\pi(x_k^i)}-f_{\gamma_{k+1}}(x_k^{i})] \cdot\epsilon_k^i$}
            \STATE{Update $\sigma_{k+1} = \textbf{clip}(\sigma_{k} + \Tilde{\eta}\cdot\widehat{grad}(\sigma_k), lb, ub)$}
            
            \STATE{Update particles $z_{k+1}^{i} = z_{k}^{i} + h(\nabla \log{\pi(x_k^i)}-f_{\gamma_{k+1}}(x_{k}^{i}))$ for $i=1, \cdots, n$}
        \ENDFOR
        \STATE{Obtain perturbed samples $x_N^i = z_N^i+\epsilon_N^i$}, where $\epsilon_N^i\sim \mathcal{N}(0,\sigma_N^2)$
        \RETURN{Particles $\{x_N^i\}_{i=1}^{n}$}
    \end{algorithmic}
\end{algorithm}

\section{Proofs}\label{appd:proof}

\subsection{Optimization Guarantees: Proof of Proposition \ref{prop:descent} and Theorem \ref{thm:convergence}} \label{appd:proof1}

\subsubsection{Main Lemmas}\label{appd:d_1_1}

\begin{lemma}\label{lem:convolution}
    Assuming the second moment of distribution $\rho$ is finite, $k_\sigma(x)=Z_0 e^{-\frac{\|x\|^2}{2\sigma^2}}$, where $Z_0$ satisfies $\int k_\sigma(x) dx=1$, then for all $\theta\in \mathbb{R}^d$, $k\in \mathbb{N}_+$, there exist constants $\tilde{C}, \tilde{M}>0$ only depending on $k$, $\sigma$ and the finite second moment, such that
    \begin{equation}
        \frac{\int \|\theta-y\|^k k_\sigma (\theta-y) d \rho(y)}{\int k_{\sigma}(\theta-y)d\rho(y)}\le \tilde{C}\|\theta\|^k+\tilde{M}
    \end{equation}
\end{lemma}

\begin{proof}
Since $\int y^2d\rho(y)<\infty$, there exists a sufficiently large constant $R_0$ and a positive constant $C_0>0$, both not related to $\theta$, such that $\int_{\|y\|<R} y^2d\rho(y)>C_0$ for all $R\ge R_0$, then $\int_{\|y\|<R} d\rho(y)>\frac{C_0}{R^2_0} $  for all $R\ge R_0$.
Note that $\{y:\|\theta-y\|<R_0+\|\theta\|\}\supseteq \{y:\|y\|<R_0\} $ by triangular inequality, 
we have $\int_{\|\theta-y\|<R_0+\|\theta\|} y^2d\rho(y)>C_0$.

From the above analysis, the denominator
\begin{equation}\label{appendix:denominator}
\begin{aligned}
\int k_{\sigma}(\theta-y)d\rho(y)&> \int_{\|\theta-y\|<R_0+\|\theta\|} k_{\sigma}(\theta-y)d\rho(y)\\&>Z_0\int_{\|\theta-y\|<R_0+\|\theta\|} e^{-\frac{(R_0+\|\theta\|)^2}{2\sigma^2}}d\rho(y)\\&>Z_0\int_{\|y\|<R_0} e^{-\frac{(R_0+\|\theta\|)^2}{2\sigma^2}}d\rho(y)\\&> Z_0 e^{-\frac{(R_0+\|\theta\|)^2}{2\sigma^2}}\cdot \frac{C_0}{R^2_0}
\end{aligned}
\end{equation}

Then
\begin{align}
\frac{\int \|\theta-y\|^k k_{\sigma}(\theta-y)d\rho(y)}{\int k_{\sigma}(\theta-y)d\rho(y)}&=\frac{\int_{\|\theta-y\|>R} \|\theta-y\|^k k_{\sigma}(\theta-y)d\rho(y)+\int_{\|\theta-y\|<R} \|\theta-y\|^k k_{\sigma}(\theta-y)d\rho(y)}{\int k_{\sigma}(\theta-y)d\rho(y)}\\
&\le \frac{\int_{\|\theta-y\|>R} \|\theta-y\|^k k_{\sigma}(\theta-y)d\rho(y)}{\int k_{\sigma}(\theta-y)d\rho(y)}+ R^k\\
&\le \frac{Z_0 R^k e^{-\frac{R^2}{2\sigma^2}}}{\int k_{\sigma}(\theta-y)d\rho(y)}+ R^k \label{appendix:monotonicity}\\
&\le  \frac{R^k e^{-\frac{R^2}{2\sigma^2}}}{ e^{-\frac{(R_0+\|\theta\|)^2}{2\sigma^2}}}\cdot\frac{R^2_0}{C_0}+ R^k\label{appendix:monotonicity2}
\end{align}

The inequality (\ref{appendix:monotonicity}) uses monotonically decreasing property for $R^k e^{-\frac{R^2}{2\sigma^2}} $ when $R> R_1=\sqrt{k}\sigma$, and plugging in (\ref{appendix:denominator}) gives the inequality (\ref{appendix:monotonicity2}).

Choosing $R=\max\{R_0+\|\theta\|, R_1\}$, we have 
$\frac{\int \|\theta-y\|^k k_{\sigma}(\theta-y)d\rho(y)}{\int k_{\sigma}(\theta-y)d\rho(y)}\le (\frac{R^2_0}{C_0}+1)R^k\le \tilde{C}(\|\theta\|^k+\tilde{M})$, where $\tilde{C}=(\frac{R_0^2}{C_0}+1)2^{k-1}, \tilde{M}=(\frac{R_0^2}{C_0}+1)\cdot\max\{2^{k-1}R_0^k,R_1^k\}$, the proof is complete.

\end{proof}

\begin{lemma}\label{lem:linear}
Denote $\psi_t=Id+t\phi$, then for any $\rho_0\in \mathcal{P}(\R^d)$ and $\rho_t=(\psi_t)_{\#}\rho_0$, we have
    \begin{equation}
        \frac{d}{d t}\mathcal{G}(\hat{\rho_t})-\frac{d}{d t}\mathcal{G}(\hat{\rho_t})\Bigg|_{t=0}\le Lt\|\phi\|^2_{L^2(\rho_0)}
    \end{equation}
where $\mathcal{G}(\hat{\rho}_t):=\int -\log \pi(x) d\hat{\mu}_t(x)$, L is the Lipschitz constant of $\nabla \log \pi$.
\end{lemma}

\begin{proof}
Analogous to (\ref{term:result}),
\begin{equation}
        \frac{d}{d t}\mathcal{G}(\hat{\rho_t})=\langle \int -\nabla\log\pi(x) q_{\sigma}(x|z)dx, \phi(\psi_t^{-1}(z)) \rangle_{L^2(\rho_t)}
    \end{equation}

Then 
\begin{equation}
\begin{aligned}
    &\frac{d}{d t}\mathcal{G}(\hat{\rho_t})-\frac{d}{d t}\mathcal{G}(\hat{\rho_t})\Bigg|_{t=0}\\&=\langle \int -\nabla\log\pi(x) q_{\sigma}(x|z)dx, \phi(\psi_t^{-1}(z)) \rangle_{L^2(\rho_t)}-\langle \int -\nabla\log\pi(x) q_{\sigma}(x|z)dx, \phi(z) \rangle_{L^2(\rho_0)}\\
    &=\langle [\int -\nabla\log\pi(x) q_{\sigma}(x|\psi_t(z))dx-\int -\nabla\log\pi(x) q_{\sigma}(x|z)dx], \phi(z) \rangle_{L^2(\rho_0)}\\
    &=\langle \int [\nabla\log\pi(x)-\nabla\log\pi(x+t\phi(z))] q_{\sigma}(x|z)dx , \phi(z) \rangle_{L^2(\rho_0)}\\
    &\le \langle L\|t\phi(z)\|,\|\phi(z)\|\rangle_{L^2(\rho_0)}=Lt\|\phi\|^2_{L^2(\rho_0)}
\end{aligned}        
\end{equation}

\end{proof}

\begin{lemma}\label{lem:entropy}
    Denote $\psi_t=Id+t\phi$, where $\|\phi(z)\|\le A\|z\|+B$, then for any fifth absolute moment finite $\rho_0\in \mathcal{P}(\R^d)$ and $\rho_t=(\psi_t)_{\#}\rho_0$, $t<\frac{1}{A}$, there exists constants $C, M$ such that
    \begin{equation}
        \frac{d}{d t}\mathcal{E}(\hat{\rho_t})-\frac{d}{d t}\mathcal{E}(\hat{\rho_t})\Bigg|_{t=0}\le t [C m_4(\rho_0)+M]+ t^2 [C m_5(\rho_0)+M]
    \end{equation}
where $\mathcal{E}(\hat{\mu}_t):=\int \log \hat{\mu}_t(x) d\hat{\mu}_t(x)$, $m_4(\rho_0):=\int \|z\|^4 d\rho_0$ and $m_5(\rho_0):=\int \|z\|^5 d\rho_0$ are the fourth and fifth absolute moments of $\rho_0$.
    
\end{lemma}

\begin{proof}

Analogous to (\ref{term:result}),
\begin{equation}
\begin{aligned}
    \frac{d}{d t}\mathcal{E}(\hat{\rho_t})&=\langle \int \nabla\log\hat{\rho}_t(x) q_{\sigma}(x|z)dx, \phi(\psi_t^{-1}(z)) \rangle_{L^2(\rho_t)}\\
    &=\langle \int q_{\sigma}(x|z)  \frac{\int \nabla_x q_{\sigma}(y|x) d\rho_t(y) }{\int q_{\sigma}(y|x) d\rho_t(y)}dx, \phi(\psi_t^{-1}(z)) \rangle_{L^2(\rho_t)}
\end{aligned}        
    \end{equation}

Then 
\begin{equation}\label{entropy_diff}
\begin{aligned}
    &\frac{d}{d t}\mathcal{E}(\hat{\rho_t})-\frac{d}{d t}\mathcal{E}(\hat{\rho_t})\Bigg|_{t=0}\\=&\langle \int q_{\sigma}(x|z)  \frac{\int \nabla_x q_{\sigma}(y|x) d\rho_t(y) }{\int q_{\sigma}(y|x) d\rho_t(y)}dx, \phi(\psi_t^{-1}(z)) \rangle_{L^2(\rho_t)}-\langle \int q_{\sigma}(x|z)  \frac{\int \nabla_x q_{\sigma}(y|x) d\rho_0(y) }{\int q_{\sigma}(y|x) d\rho_0(y)}dx, \phi(z) \rangle_{L^2(\rho_0)}\\
    =&\langle [\int q_{\sigma}(x|\psi_t(z))  \frac{\int \nabla_x q_{\sigma}(y|x) d\rho_t(y) }{\int q_{\sigma}(y|x) d\rho_t(y)}dx-\int q_{\sigma}(x|z)  \frac{\int \nabla_x q_{\sigma}(y|x) d\rho_0(y) }{\int q_{\sigma}(y|x) d\rho_0(y)}dx], \phi(z) \rangle_{L^2(\rho_0)}\\
    =&\frac{1}{\sigma^2}\langle \int [q_{\sigma}(x|\psi_t(z))  \frac{\int y q_{\sigma}(y|x) d\rho_t(y) }{\int q_{\sigma}(y|x) d\rho_t(y)}- q_{\sigma}(x|z)  \frac{\int y q_{\sigma}(y|x) d\rho_0(y) }{\int q_{\sigma}(y|x) d\rho_0(y)}]dx, \phi(z) \rangle_{L^2(\rho_0)}+\frac{t}{\sigma^2}\|\phi\|^2_{L^2(\rho_0)}
\end{aligned}        
\end{equation}

Denote $A_t:=q_{\sigma}(x|\psi_t(z))  \frac{\int y q_{\sigma}(y|x) d\rho_t(y) }{\int q_{\sigma}(y|x) d\rho_t(y)}$, then $A_0= q_{\sigma}(x|z)  \frac{\int y q_{\sigma}(y|x) d\rho_0(y) }{\int q_{\sigma}(y|x) d\rho_0(y)}$.

For simplicity, we denote $k_\sigma(w)=Z_0 e^{-\frac{\|w\|^2}{2\sigma^2}}$, where $Z_0$ satisfies $\int k_\sigma(w)dw=1$, then $q_{\sigma}(x|z)=k_{\sigma}(x-z)$.

By Lagrangian mean value theorem, there exists $t',t''\in [0,t]$, such that:
\begin{equation}\label{mean:1}
q_{\sigma}(x|\psi_t(z)) = Z_0 e^{-\frac{\|x-(z+t\phi(z))\|^2}{2\sigma^2}}=Z_0[e^{-\frac{\|x-z\|^2}{2\sigma^2}}+t\cdot e^{-\frac{\|x-(z+t'\phi(z))\|^2}{2\sigma^2}}\frac{(x-(z+t'\phi(z)))}{\sigma^2}\phi(z)]
\end{equation}

and 
\begin{equation}\label{mean:2}
\frac{\int y k_{\epsilon}(x-y)d\rho_t(y)}{\int k_{\sigma}(x-y)d\rho_t(y)}=\frac{\int y k_{\sigma}(x-y)d\rho(y)}{\int k_{\sigma}(x-y)d\rho(y)}+t\cdot \frac{\partial}{\partial t}\Big|_{t=t''}(\frac{\int y k_{\sigma}(x-y)d\rho_t(y)}{\int k_{\sigma}(x-y)d\rho_t(y)})
\end{equation}

Using continuity equation $\frac{d\rho_t(y)}{d t}=-\nabla\cdot(\rho_t(y) \phi(\psi_t^{-1}(y))) $ for $y\sim \rho_t$, we have:
\begin{align}
    &\frac{\partial}{\partial t} \frac{\int y k_{\sigma}(x-y)d\rho_t(y)}{\int k_{\sigma}(x-y)d\rho_t(y)}\\=&\frac{\int y k_{\sigma}(x-y)\frac{d\rho_t(y)}{d t}\int k_{\sigma}(x-y)d\rho_t(y)- \int y k_{\sigma}(x-y)d\rho_t(y)\int k_{\sigma}(x-y)\frac{d\rho_t(y)}{d t}}{(\int k_{\sigma}(x-y)d\rho_t(y))^2}\\
    =&\frac{\int \phi(\psi_t^{-1}(y))\nabla( y k_{\sigma}(x-y))d\rho_t(y) }{\int k_{\sigma}(x-y)d\rho_t(y)}- \frac{\int y k_{\sigma}(x-y)d\rho_t(y) \int \phi(\psi_t^{-1}(y))\nabla k_{\sigma}(x-y)d\rho_t(y) }{(\int k_{\sigma}(x-y)d\rho_t(y))^2}\\
    =&\frac{\int \phi(\psi_t^{-1}(y))(1+y\frac{x-y}{\sigma^2}) k_{\sigma}(x-y)d\rho_t(y)}{\int k_{\sigma}(x-y)d\rho_t(y)}-\frac{\int y k_{\sigma}(x-y)d\rho_t(y)}{\int k_{\sigma}(x-y)d\rho_t(y)}\frac{\int \phi(\psi_t^{-1}(y))\frac{x-y}{\sigma^2} k_{\sigma}(x-y)d\rho_t(y) }{\int k_{\sigma}(x-y)d\rho_t(y)}
\end{align}

Since $\psi_t(y)=y+t\phi(y)$ and $\|\phi(y)\|\le A\|y\|+B$, then $(1-tA)\|y\|-tB\le \|y\|-t\|\phi(y)\|\le \|\psi_t(y)\|$.

For $t<\frac{1}{A}$, we have $\|y\|\le \frac{1}{1-tA}\|\psi_t(y)\|+\frac{tB}{1-tA}$, which means $\|\psi_t^{-1}(y)\|\le \frac{1}{1-tA}\|y\|+\frac{tB}{1-tA}$.

This means that $\|\phi(\psi_t^{-1}(y))\|\le \frac{A}{1-tA}\|y\|+\frac{B}{1-tA}$.

By triangular inequality, we have 
\begin{equation}\label{appendix:C1_const}
\begin{aligned}
    &\| \phi(\psi_t^{-1}(y))(1+y\frac{x-y}{\sigma^2})\| \\ \le &\frac{A}{(1-tA)\sigma^2}\|x-y\|^3+(2\frac{A\|x\|}{(1-tA)\sigma^2}+\frac{B}{1-tA})\|x-y\|^2\\&+(\frac{\|x\|}{\sigma^2}(\frac{A}{1-tA}\|x\|+\frac{B}{1-tA})+\frac{A}{1-tA})\|x-y\|+(\frac{A}{1-tA}\|x\|+\frac{B}{1-tA})
\end{aligned}    
\end{equation}

The right hand side is a polynomial of $(\|x\|,\|x-y\|)$, the order of which is at most 3. 
Using basic inequalities $\|x\|^3+\frac{4}{27}d\ge \|x\|^2, \|x\|^3+\frac{8}{27}d\ge \|x\|$ and Lemma \ref{lem:convolution}, plugging in (\ref{appendix:C1_const}), we obtain:
\begin{equation}
    \| \frac{\int \phi(\psi_t^{-1}(y))(1+y\frac{x-y}{\sigma^2}) k_{\sigma}(x-y)d\rho_t(y)}{\int k_{\sigma}(x-y)d\rho_t(y)}\| \le C_1\|x\|^3+M_1
\end{equation}

where $C_1, M_1>0$ are constants that are not related to $x$.

Analogously, 
\begin{equation}
    \|\frac{\int y k_{\sigma}(x-y)d\rho_t(y)}{\int k_{\sigma}(x-y)d\rho_t(y)}\frac{\int \phi(\psi_t^{-1}(y))\frac{x-y}{\sigma^2} k_{\sigma}(x-y)d\rho_t(y) }{\int k_{\sigma}(x-y)d\rho_t(y)}\| \le C_2\|x\|^3+M_2
\end{equation}

To conclude, for $t<\frac{1}{A}$, we obtain
\begin{equation}
    \|\frac{\partial}{\partial t}\Big|_{t=t''}(\frac{\int y k_{\sigma}(x-y)d\rho_t(y)}{\int k_{\sigma}(x-y)d\rho_t(y)})\|\le (C_1+C_2)\|x\|^3+(M_1+M_2):=C_3\|x\|^3+M_3
\end{equation}

From (\ref{mean:1}) and (\ref{mean:2}) we have
\begin{equation}
\begin{aligned}
    &\frac{1}{Z_0}\|A_t-A_0\|\\=&\|t [ \frac{\int y k_{\sigma}(x-y)d\rho(y)}{\int k_{\sigma}(x-y)d\rho(y)}\cdot e^{-\frac{\|x-(z+t'\phi(z))\|^2}{2\sigma^2}}\frac{(x-(z+t'\phi(z)))}{\sigma^2}\phi(z)+ e^{-\frac{\|x-z\|^2}{2\sigma^2}}\cdot \frac{\partial}{\partial t}\Big|_{t=t''}(\frac{\int y k_{\sigma}(x-y)d\rho_t(y)}{\int k_{\sigma}(x-y)d\rho_t(y)})]
    \\&+t^2 e^{-\frac{\|x-(z+t'\phi(z))\|^2}{2\sigma^2}}\frac{(x-(z+t'\phi(z)))}{\sigma^2}\phi(z)\cdot\frac{\partial}{\partial t}\Big|_{t=t''}(\frac{\int y k_{\sigma}(x-y)d\rho_t(y)}{\int k_{\sigma}(x-y)d\rho_t(y)})\|
    \\\le & t [ \|\phi(z)\| (C_4\|x\|+M_4)\|\frac{(x-(z+t'\phi(z)))}{\sigma^2}\|e^{-\frac{\|x-(z+t'\phi(z))\|^2}{2\sigma^2}}+e^{-\frac{\|x-z\|^2}{2\sigma^2}}\cdot(C_3\|x\|^3+M_3)]\\
    &+t^2  \|\phi(z)\| \|\frac{(x-(z+t'\phi(z)))}{\sigma^2}\|e^{-\frac{\|x-(z+t'\phi(z))\|^2}{2\sigma^2}}\cdot (C_3\|x\|^3+M_3)
\end{aligned}    
\end{equation}

The constants $C_4, M_4$ come from using Lemma \ref{lem:convolution} on the term $\frac{\int y k_{\sigma}(x-y)d\rho(y)}{\int k_{\sigma}(x-y)d\rho(y)}$.

Then
\begin{equation}
\begin{aligned}
    &\int \|q_{\sigma}(x|\psi_t(z))  \frac{\int y q_{\sigma}(y|x) d\rho_t(y) }{\int q_{\sigma}(y|x) d\rho_t(y)}- q_{\sigma}(x|z)  \frac{\int y q_{\sigma}(y|x) d\rho_0(y) }{\int q_{\sigma}(y|x) d\rho_0(y)}\|dx\\
    =&\int \|A_t-A_0\| dx\\
    \le &t Z_0\int [ \|\phi(z)\| (C_4\|x\|+M_4)\|\frac{(x-(z+t'\phi(z)))}{\sigma^2}\|e^{-\frac{\|x-(z+t'\phi(z))\|^2}{2\sigma^2}}+e^{-\frac{\|x-z\|^2}{2\sigma^2}}\cdot(C_3\|x\|^3+M_3)] dx\\
    &+t^2 Z_0 \int  \|\phi(z)\| \|\frac{(x-(z+t'\phi(z)))}{\sigma^2}\|e^{-\frac{\|x-(z+t'\phi(z))\|^2}{2\sigma^2}}\cdot (C_3\|x\|^3+M_3) dx\\
    \le& t[(C_5\|z+t'\phi(z)\|+M_5)\|\phi(z)\|+C_6\|z\|^3+M_6]+t^2(C_7\|z+t'\phi(z)\|^3+M_7)\|\phi(z)\|
\end{aligned}    
\end{equation}

Plugging in $t'\le t<\frac{1}{A}$ and $\|\phi(z)\|\le A\|z\|+B $, we obtain

\begin{equation}
\int \|q_{\sigma}(x|\psi_t(z))  \frac{\int y q_{\sigma}(y|x) d\rho_t(y) }{\int q_{\sigma}(y|x) d\rho_t(y)}- q_{\sigma}(x|z)  \frac{\int y q_{\sigma}(y|x) d\rho_0(y) }{\int q_{\sigma}(y|x) d\rho_0(y)}\|dx\le t(C_8\|z\|^3+M_8)+t^2(C_9\|z\|^4+M_9)
\end{equation}

Using this result and  $\|\phi(z)\|\le A\|z\|+B $ in (\ref{entropy_diff}), we finally obtain:
\begin{equation}
\begin{aligned}
&\frac{d}{d t}\mathcal{E}(\hat{\rho_t})-\frac{d}{d t}\mathcal{E}(\hat{\rho_t})\Bigg|_{t=0}\\
\le &\frac{1}{\sigma^2}\int [t(C_8\|z\|^3+M_8)+t^2(C_9\|z\|^4+M_9)]\|\phi(z)\|d\rho_0(z)+\frac{t}{\sigma^2}\|\phi\|^2_{L^2(\rho_0)}\\
\le &t [C m_4(\rho_0)+M]+ t^2 [C m_5(\rho_0)+M]
\end{aligned}
\end{equation}

The proof is complete.

\end{proof}

\subsubsection{Proof of Proposition \ref{prop:descent} and Theorem \ref{thm:convergence}}\label{append:d_1_2}

\begin{proposition}
 Suppose Assumption \ref{assp:approx}, \ref{assp:lipscore}, \ref{assp:moment}, \ref{assp:bdvel}
hold. Then the following
inequality holds for $h<\frac{1}{A}$:
\[
\hat{\mathcal{F}}(\mu_{(k+1)h})-\hat{\mathcal{F}}(\mu_{k h})\le -\frac{1}{2}h\|\nabla_{W_2} \hat{\mathcal{F}}(\mu_{k h})\|^2_{L^2(\mu_{k h})}+\frac{1}{2}h\epsilon_k+h^2 [C m_{4}+M]+h^3[C m_{5}+M], 
\]
where $C, M$ are constants that depend on $A, B, \sigma, m_5$ and do not depend on $k$ or $h$.
\end{proposition}

\begin{proof}

The KL divergence can be split into two terms:
\begin{equation}
    \hat{\mathcal{F}}(\mu_t):=\int \log \frac{\hat{\mu}_t(x)}{\pi(x)} d\hat{\mu}_t(x)=\int -\log \pi(x) d\hat{\mu}_t(x)+\int \log \hat{\mu}_t(x) d\hat{\mu}_t(x):=\mathcal{G}(\hat{\mu}_t)+\mathcal{E}(\hat{\mu}_t)
\end{equation}

Since
\begin{equation}\label{propproof:comp}
    \hat{\mathcal{F}}(\mu_{(k+1)h})-\hat{\mathcal{F}}(\mu_{k h})= h \frac{d}{d t}\Bigg |_{t=0}\mathcal{F}(\hat{\mu}_{k h+t})+\int_0^{h}\left[  \frac{d}{d t}\mathcal{F}(\hat{\mu}_{k h+t})- \frac{d}{d t}\Bigg |_{t=0}\mathcal{F}(\hat{\mu}_{k h+t})\right]dt
\end{equation}

and
\begin{equation}
\frac{d}{d t}\Bigg |_{t=0}\mathcal{F}(\hat{\mu}_{k h+t})=\langle \nabla_{W_2} \hat{\mathcal{F}}(\mu_{k h}), v_{kh}\rangle_{\mu_{kh}}\le -\frac{1}{2}\|\nabla_{W_2} \hat{\mathcal{F}}(\mu_{k h})\|^2_{L^2(\mu_{k h})} +\frac{1}{2}\epsilon_k
\end{equation}

Setting $\phi=v_{kh}, \rho_0=\mu_{k h}$ and $\rho_t=\mu_{k h+t}$ in Lemma \ref{lem:linear}, we have
\begin{equation}
 \frac{d}{d t}\mathcal{G}(\hat{\mu}_{k h+t})- \frac{d}{d t}\Bigg |_{t=0}\mathcal{G}(\hat{\mu}_{k h+t})\le Lt\|\phi\|^2_{L^2(\hat{\mu}_{k h})}\le t(C_1 m_2 +M_1)
\end{equation}

where $m_2$ is the upper bound of the second moment of $\mu_{k h}$.

By Lemma \ref{lem:entropy}, we have
\begin{equation}
 \frac{d}{d t}\mathcal{E}(\hat{\mu}_{k h+t})- \frac{d}{d t}\Bigg |_{t=0}\mathcal{E}(\hat{\mu}_{k h+t})\le t [C_2 m_4 +M_2]+ t^2 [C_2 m_5 +M_2]
\end{equation}

where $m_4$ and $m_5$ are the upper bounds of the fourth and fifth absolute moments of $\mu_{k h}$.

Note that $\|z\|^4+\frac{1}{4}d\ge \|z\|^2$, we have $m_4+\frac{1}{4}d\ge m_2$.

Then 
\begin{equation}
\begin{aligned}
    &\int_0^{h}\left[  \frac{d}{d t}\mathcal{F}(\hat{\mu}_{k h+t})- \frac{d}{d t}\Bigg |_{t=0}\mathcal{F}(\hat{\mu}_{k h+t})\right]dt\\
    =& \int_0^{h}\left[  \frac{d}{d t}\mathcal{G}(\hat{\mu}_{k h+t})- \frac{d}{d t}\Bigg |_{t=0}\mathcal{G}(\hat{\mu}_{k h+t})+ \frac{d}{d t}\mathcal{E}(\hat{\mu}_{k h+t})- \frac{d}{d t}\Bigg |_{t=0}\mathcal{E}(\hat{\mu}_{k h+t})\right]dt\\
    \le& h^2[C m_4+M]+h^3[C m_5+M]
\end{aligned} 
\end{equation}

Plugging this back to (\ref{propproof:comp}) gives the desired result. The proof is complete.

\end{proof}

\begin{theorem}
    Assume Proposition \ref{prop:descent} holds, for suitable step size $h$ as a function of $\hat{\mathcal{F}}(\mu_0), K, C, M, m_4, m_5$, the average of squared gradient norms satisfies
 \[
 \frac{1}{K}\sum_{k=1}^{K} \|\nabla_{W_2} \hat{\mathcal{F}}(\mu_{k h})\|^2_{L^2(\mu_{k h})}\le \frac{R}{K^{\frac{1}{2}}}+\frac{S}{K^{\frac{2}{3}}}+\frac{1}{K}\sum_{k=1}^{K} \epsilon_k,
 \]
for
 \[
 K>\min\left\{\frac{A^2\hat{\mathcal{F}}(\mu_0)}{Cm_4+M}, \frac{A^3\hat{\mathcal{F}}(\mu_0)}{Cm_5+M}\right\},
 \]
where $R:=4\sqrt{\hat{\mathcal{F}}(\mu_0)(C m_{4}+M)}$ and $S:=4(\hat{\mathcal{F}}(\mu_0))^{\frac{2}{3}}(C m_{5}+M)^{\frac{1}{3}}$.
\end{theorem}

\begin{proof}
From Proposition \ref{prop:descent} we have:
\begin{equation}
    \hat{\mathcal{F}}(\mu_{(k+1)h})-\hat{\mathcal{F}}(\mu_{k h})\le -\frac{1}{2}h\|\nabla_{W_2} \hat{\mathcal{F}}(\mu_{k h})\|^2_{L^2(\mu_{k h})}+\frac{1}{2}h\epsilon_k+h^2 [C m_{4}+M]+h^3[C m_{5}+M]
\end{equation}

Taking the sum from $k=0$ to $k=K-1$, we obtain
\begin{equation}
    \frac{1}{K}\sum_{k=1}^{K}\|\nabla_{W_2}\hat{\mathcal{F}}(\mu_{k h})\|^2_{L^2(\mu_{k h})}\le \frac{2\mathcal{F}(\hat{\mu}_0)}{h K}+\frac{1}{K}\sum_{k=1}^K\epsilon_k+h (2(C m_{4}+M))+h^2 (2(C m_{5}+M))
\end{equation}

Denote $a=\frac{2\hat{\mathcal{F}}(\mu_0)}{K}, b=2(C m_{4}+M), c= 2(C m_{5}+M)$, choose $h=\min \{(\frac{a}{b})^{\frac{1}{2}}, (\frac{a}{c})^{\frac{1}{3}}\}$, then choosing $K>\min\{\frac{A^2\hat{\mathcal{F}}(\mu_0)}{Cm_4+M}, \frac{A^3\hat{\mathcal{F}}(\mu_0)}{Cm_5+M}\} $ leads to $h<\frac{1}{A}$. In this case, we have $\frac{a}{h}+bh+c h^2\le 2((ab)^{\frac{1}{2}}+a^{\frac{2}{3}}c^{\frac{1}{3}})$. Plugging back $a,b,c$ gives the desired result.

\end{proof}

\subsection{Sample Complexities: Proof of Theorem \ref{thm:generalization} and Theorem \ref{thm:generalization2}}  \label{appd:proof2}

\subsubsection{Proof of Theorem \ref{thm:generalization}}\label{appd:d_2_1}

Consider ERM $\hat{f_{\gamma}}:=\argmin_{f_{\gamma}\in\mathcal{S}}\frac{1}{n}\sum_{i=1}^n\ell(z_i;f_{\gamma})$, where $\ell(z;f_{\gamma}):=\E_{q(x|z)}\|f_{\gamma}(x)-\nabla\log q(x|z)\|^2$ and $\{z_i\}$ are i.i.d. samples of $\mu(z)$. 
The population loss is given by $\ell(f_{\gamma}):=\E_{\mu(z)q(x|z)}\|f_{\gamma}(x)-\nabla\log q(x|z)\|^2=\E_{\hat{\mu}(x)}\|f_{\gamma}(x)-\nabla\log \hat{\mu}(x)\|^2+c_*$, where $c_*=\E_{\mu(z)q(x|z)}\|\nabla\log\hat{\mu}(x)-\nabla\log q(x|z)\|^2$ is a constant independent of $f_{\gamma}$.

\begin{theorem}\label{append:thm:generalization}
    Under 
    Assumption \ref{assp:moment}, 
    for any $\delta>0$, it holds with probability no less than $1-2\delta$ that, 
    \begin{equation}
        \ell_{sm}(\hat{f_{\gamma}})
        \leq 2\inf_{f_{\gamma}\in\mathcal{S}}\ell_{sm}(f_{\gamma}) + \mathcal{O}\left(\sqrt{\frac{M^2+\frac{d}{\sigma^2}}{n}\log\mathcal{N}}\right).
    \end{equation}
    where  
    $\log\mathcal{N}=SL\log\big(WL(B\vee 1)\sigma(n/\delta)\log{d}\big)$.
\end{theorem}

\begin{proof}

    Recall that $\ell(z;f_{\gamma})=\E_{q(x|z)}\left[\|f_{\gamma}(x)-\nabla\log q(x|z)\|^2-\|\nabla\log\hat{\mu}(x)-\nabla\log q(x|z)\|^2\right]$.
    Let $R>0$ be some constant defined later and consider the corresponding truncated loss $\ell^{tr}(z;f_{\gamma}):=\ell(z;f_{\gamma})\mathbbm{1}_{\{\|z\|_\infty\leq R\}}$.
    For any $f_{\gamma}\in \mathcal{S}, z\in\R^d$, we have
    \begin{equation}
        \ell(z;f_{\gamma})\leq 2\E_{q(x|z)}\big[\|f_{\gamma}(x)\|^2+\|\nabla\log q(x|z)\|^2\big]\leq 2(M^2+d/\sigma^2)=:b.
    \end{equation}

    For the minimizer of empirical score matching $\hat{f_{\gamma}}$, 
    \begin{equation}\label{eq:decomposition}
        \begin{aligned}
            \ell(\hat{f_{\gamma}})
            &=\E_{\mu(z)}\ell^{tr}(\hat{f_{\gamma}};z)+\E_{\mu(z)}\ell(\hat{f_{\gamma}};z)\mathbbm{1}(\|z\|_\infty>R) \\
            &\leq \E_{\mu(z)}\ell^{tr}(\hat{f_{\gamma}};z)+\mathbb{P}(\|z\|_\infty>R)b
        \end{aligned}
    \end{equation}

    By Assumption \ref{assp:moment}, $\mathbb{P}(\|z\|_\infty>R)\leq C R^{-5}$ for some constant $C$.

    Hence with probability no less than $1-CnR^{-5}$, for all $1\leq i\leq n$, we have $\|z_i\|_\infty\leq R$ and consequently $\hat{f_{\gamma}}=\argmin_{f_{\gamma}\in\mathcal{S}}\frac{1}{n}\sum_{i=1}^n\ell(z_i;f_{\gamma})=\argmin_{f_{\gamma}\in\mathcal{S}}\frac{1}{n}\sum_{i=1}^n\ell^{tr}(z_i;f_{\gamma})=:\argmin_{f_{\gamma}\in\mathcal{S}}\hat{\ell}^{tr}(f_{\gamma})$.

    Define the empirical Rademacher complexity of a function class $\mathcal{F}$ as
    \begin{equation}
        \mathcal{R}_n(\mathcal{F}):=\E_{\bm{\sigma}}\sup_{f\in\mathcal{F}}\Big|\frac{1}{n}\sum_{i=1}^n\sigma_if(z_i)\Big|,\ \bm{\sigma}\sim \text{Unif}(\{-1,1\}^n).
    \end{equation}
    For $r>0$, let $\mathcal{S}_r:=\{f_{\gamma}\in\mathcal{S}:\hat{\ell}^{tr}(f_{\gamma})\leq r\}$ and $\mathcal{F}_r=\big\{\ell^{tr}(\cdot;f_{\gamma}):f_{\gamma}\in\mathcal{S}_r\big\}$.
    Note that for any $f,g\in\mathcal{F}_r$,
    \begin{equation}
        \Big\|\frac{1}{\sqrt{n}}\sum_{i=1}^n\sigma_if(z_i)-\frac{1}{\sqrt{n}}\sum_{i=1}^n\sigma_ig(z_i)\Big\|_{\psi_2}\lesssim \sqrt{\frac{1}{n}\sum_i \|f(z_i)-g(z_i)\|^2}\lesssim \|f-g\|_{L^2(\mathbb{P}_n)}.
    \end{equation}
    and $\textbf{diam}(\mathcal{F}_r,\|\cdot\|_{L^2(\mathbb{P}_n)})\leq 2\sqrt{br}$.
    Then by Dudley's bound \citep{srebro2010smoothness, wainwright2019high}, we have
    \begin{equation}
            \mathcal{R}_n(\mathcal{F}_r)
            \lesssim \inf_{\alpha\geq 0} \left\{\alpha+\int_\alpha^{2\sqrt{br}}\sqrt{\frac{\log\mathcal{N}(\mathcal{F}_r,\|\cdot\|_{L^2(\mathbb{P}_n)},\varepsilon)}{n}}d\varepsilon\right\}.
    \end{equation}
    By Lemma \ref{lem:cover_number} and Lemma \ref{lem:covering_num}, for any $K\geq 2R$ and $\alpha_K:=CM\sqrt{dr}\exp(-C'K^2/\sigma^2)$,
    \begin{equation}
        \begin{aligned}
            \mathcal{R}_n(\mathcal{F}_r)
            &\lesssim \inf_{\alpha\geq \alpha_K} \left\{\alpha+\int_\alpha^{2\sqrt{br}}\sqrt{\frac{\log\mathcal{N}(\mathcal{F}_r,\|\cdot\|_{L^2(\mathbb{P}_n)},\varepsilon)}{n}}d\varepsilon\right\} \\
            &\lesssim \inf_{\alpha\geq \alpha_K} \left\{\alpha+\int_\alpha^{2\sqrt{br}}\sqrt{\frac{\log\mathcal{N}(\mathcal{S}_r,\|\cdot\|_{L^\infty([-K,K]^d)},\varepsilon/\sqrt{4r})}{n}}d\varepsilon\right\} \\
            &\lesssim \inf_{\alpha\geq \alpha_K} \left\{\alpha+ \int_\alpha^{2\sqrt{br}}\sqrt{\frac{SL\log(\frac{LWK(B\vee 1)r}{\epsilon})}{n}}d\varepsilon\right\} \\
            &\lesssim \sqrt{\frac{brSL\log(LWK(B\vee 1)r)}{n}} + \alpha_K.
        \end{aligned}
    \end{equation}
    Let $K=\max\Big\{2R,C\sigma\log^{\frac{1}{2}}\big(\frac{Mdn}{bSL}\big)\Big\}$ and we finally obtain
    \begin{equation}
        \mathcal{R}_n(\mathcal{F}_r)\lesssim \sqrt{\frac{brSL\log\big(LWR(B\vee 1)r\sigma\log{(dn)}\big)}{n}}.
    \end{equation}
    By \citet[Theorem 6.1]{bousquet2002concentration}, with probability at least $1-\delta$, for all $f_{\gamma}\in\mathcal{S}$,
    \begin{equation}
        \ell^{tr}(f_{\gamma})\leq \hat{\ell}^{tr}(f_{\gamma}) + \mathcal{O}\left(\sqrt{\hat{\ell}^{tr}(f_{\gamma})\cdot\frac{\gamma_R+b\log(1/\delta)}{n}}+\frac{\gamma_R+b\log(1/\delta)}{n}\right),
    \end{equation}
    where $\gamma_R=bSL\log\big(LWR(B\vee 1)\sigma\log{(dn)}\big)$. 

Hence with probability no less than $1-CnR^{-5}-\delta$,   
    we have
    \begin{equation}
        \begin{aligned}
            \ell^{tr}(\hat{f_{\gamma}})
            &\leq \inf_{f_{\gamma}\in\mathcal{S}}\ell^{tr}(f_{\gamma})+\mathcal{O}\left(\sqrt{\inf_{f_{\gamma}\in\mathcal{S}}\ell^{tr}(f_{\gamma})\cdot\frac{\gamma_R+b\log(1/\delta)}{n}}+\frac{\gamma_R+b\log(1/\delta)}{n}\right) \\
            &\leq \inf_{f_{\gamma}\in\mathcal{S}}\ell(f_{\gamma})+\mathcal{O}\left(\sqrt{\inf_{f_{\gamma}\in\mathcal{S}}\ell(f_{\gamma})\cdot\frac{\gamma_R+\log(1/\delta)}{n}}+\frac{\gamma_R+b\log(1/\delta)}{n}\right).
        \end{aligned}
    \end{equation}
    Combining it with \eqref{eq:decomposition} and  
    defining $R=(Cn/\delta)^{\frac{1}{5}}$,     
    we conclude that with probability at least $1-2\delta$,
    \begin{equation}
        \begin{aligned}
            \ell(f_{\gamma})\leq \inf_{f_{\gamma}\in\mathcal{S}}\ell(f_{\gamma})+\mathcal{O}\left(\sqrt{\inf_{f_{\gamma}\in\mathcal{S}}\ell(f_{\gamma})\cdot\frac{bSL\log\big(LW(B\vee 1)\sigma(n/\delta)^{\frac{1}{5}}\log{dn}\big)+b\log(1/\delta)}{n}}\right).
        \end{aligned}
    \end{equation}
    Or equivalently,
    \begin{equation}
        \begin{aligned}
            \ell_{sm}(f_{\gamma})
            &\leq \inf_{f_{\gamma}\in\mathcal{S}}\ell_{sm}(f_{\gamma})+\mathcal{O}\left(\sqrt{(\inf_{f_{\gamma}\in\mathcal{S}}\ell_{sm}(f_{\gamma})+c_*)\cdot\frac{bSL\log\big(LW(B\vee 1)\sigma(n/\delta)\log{d}\big)}{n}}\right) \\
            &\leq 2\inf_{f_{\gamma}\in\mathcal{S}}\ell_{sm}(f_{\gamma})+\mathcal{O}\left(\sqrt{c_*\cdot\frac{bSL\log\big(LW(B\vee 1)\sigma(n/\delta)\log{d}\big)}{n}}\right)\\
            &= 2\inf_{f_{\gamma}\in\mathcal{S}}\ell_{sm}(f_{\gamma})+\mathcal{O}\left(\sqrt{\frac{bSL\log\big(LW(B\vee 1)\sigma(n/\delta)\log{d}\big)}{n}}\right).
        \end{aligned}
    \end{equation}    
\end{proof}

\begin{lemma}\label{lem:cover_number}
    There exists constants $C,C'>0$ such that, for any $K\geq 2R, \varepsilon\geq C\sqrt{dr}M\exp(-C'K^2/\sigma^2)$,
    \begin{equation}
        \mathcal{N}(\mathcal{F}_r,\|\cdot\|_{L^2(\mathbb{P}_n)},\varepsilon)
        \leq \mathcal{N}(\mathcal{S}_r,\|\cdot\|_{L^\infty([-K,K]^d)},\varepsilon/\sqrt{4r})
    \end{equation}
\end{lemma}

\begin{proof}
    Given an $(\varepsilon/\sqrt{4r})$-net $\{\{s_j\}_{j=1}^N$ in $(\mathcal{S}_r,\|\cdot\|_{L^\infty([-K,K]^d)})$, we aim to show that $\{\ell^{tr}(\cdot;s_j)\}_{j=1}^N$ is an $\varepsilon$-net in $(\mathcal{F}_r,\|\cdot\|_{L^2(\mathbb{P}_n)})$.
    In fact, for any $s\in \mathcal{S}_r$, there exists $s_j$ such that $\|s_j-s\|_{L^\infty([-K,K]^d)}\leq \varepsilon/\sqrt{4r}$. 
    Therefore, for any $z\in [-R,R]^d$,
    \begin{equation}
        \begin{aligned}
            \E_{q(x|z)}\big\|s(x)-s_j(x)\big\|^2
            &\leq \E_{q(x|z)}\big\|s(x)-s_j(x)\big\|^2\mathbbm{1}_{\{\|x-z\|_\infty\leq K/2\}} + M^2\mathbb{P}(\|x-z\|_\infty> K/2) \\
            &\leq \|s-s_j\|^2_{L^\infty([-K,K]^d)} + CdM^2\exp(-K^2/C\sigma^2) \\
            &\leq \varepsilon^2/4r + CdM^2\exp(-C'K^2/\sigma^2) \\
            &\leq \varepsilon^2/2r.
        \end{aligned}
    \end{equation}
    Note that in the second inequality we use $K\geq 2R$ and probability of Gaussian-tail. Hence,
    \begin{equation}
        \begin{aligned}
            \|\ell^{tr}(\cdot;s)-\ell^{tr}(\cdot;s_j)\|_{L^2(\mathbb{P}_n)}
            &\leq \sqrt{\frac{1}{n}\sum_{i=1}^n\left[\E_{q(x|z_i)}\big(s(x)-s_j(x)\big)^T\big(s(x)+s_j(x)-2\nabla \log q(x|z_i)\big)\right]^2} \\
            &\leq \sqrt{\frac{1}{n}\sum_{i=1}^n\E_{q(x|z_i)}\big\|s(x)-s_j(x)\big\|^2 \E_{q(x|z_i)}\big\|s(x)+s_j(x)-2\nabla \log q(x|z_i)\big\|^2} \\
            &\leq \sqrt{\varepsilon^2/2r\cdot 2r} = \varepsilon.
        \end{aligned}
    \end{equation}
    This concludes the proof.
\end{proof}

\subsubsection{Proof of Theorem \ref{thm:generalization2}}\label{appd:d_2_2}

\begin{theorem}\label{append:thm:generalization2}
    Under 
    Assumption \ref{assp:subg}, 
    for any $\delta>0$, it holds with probability no less than $1-2\delta$ that, 
    \begin{equation}
        \ell_{sm}(\hat{f_{\gamma}})
        \leq 4\inf_{f_{\gamma}\in\mathcal{S}}\ell_{sm}(f_{\gamma}) + \mathcal{O}\left(\frac{M^2}{n}\log\mathcal{N}\right).
    \end{equation}
    where 
    $\log\mathcal{N}=\log^3(nd/\sigma\delta)[\log(1/\delta)+SL\log(nLW(B\vee 1)M\log(1/\delta))]$.
\end{theorem}

\begin{proof}
    Define the denoising score matching loss as
    \begin{equation}
        \ell(z,f_{\gamma}):=\E_{q(x|z)}\left[\|f_{\gamma}(x)-\nabla\log q(x|z)\|^2-\|\nabla\log \hat{\mu}(x)-\nabla\log q(x|z)\|^2\right].
    \end{equation}
    Consider the truncated function class defined on $\R^d$,
    \begin{equation}
        \Phi=\{z\mapsto \widetilde{\ell}(z,f_{\gamma}):=\ell(z,f_{\gamma})\cdot\mathbbm{1}_{\|z\|_\infty\leq R}:f_{\gamma}\in\mathcal{S}\},
    \end{equation}
    Here the truncation radius $R>0$ will be defined later.
    Since $\mu(z)$ is $C'$-sub-Gaussian, it is easy to show that with probability no less than $1-2dn\exp(-C'R^2)$, it holds that $\|z_i\|_\infty\leq R$ for all $1\leq i \leq n$.
    Hence by definition, the empirical minimizer also satisfies $\hat{f_{\gamma}}= \argmin_{f_{\gamma}\in\mathcal{S}} \frac{1}{n}\sum_{i=1}^n \widetilde{\ell}(z_i,f_{\gamma})$.
    Below we reason conditioned on this event and verify the conditions required in Lemma \ref{lem:local_rademacher}. 
    \begin{enumerate}[label=\textbf{Step \arabic*.}]
        \item To bound the individual loss,
        \begin{equation}
            \widetilde{\ell}(z,f_{\gamma})
            \leq \E_{q(x|z)}\|f_{\gamma}(x)-\nabla\log q(x|z)\|^2
            \lesssim M^2+d/\sigma^2. 
        \end{equation}
        And according to Lemma \ref{lem:lip_score},
        \begin{equation}
            \begin{aligned}
                -\widetilde{\ell}(z,f_{\gamma})
                &\leq \E_{q(x|z)}\|\nabla\log \hat{\mu}(x)-\nabla\log q(x|z)\|^2\cdot \mathbbm{1}_{\|z\|_\infty\leq R} \\
                &\lesssim \E_{q(x|z)}\|\nabla\log \hat{\mu}(x)\|^2\cdot \mathbbm{1}_{\|z\|_\infty\leq R}+d/\sigma^2 \\
                &\lesssim \E_{q(x|z)}\frac{\|x\|^6+1}{\sigma^8}\cdot \mathbbm{1}_{\|z\|_\infty\leq R}+d/\sigma^2 \\
                &\lesssim \frac{R^6+d^3\sigma^6}{\sigma^8}.
            \end{aligned}
        \end{equation}
        Let $M_R:=C_x'\left(M^2+\frac{R^6+d^3\sigma^6}{\sigma^8}\right)$ and thus $|\widetilde{\ell}(z,f_{\gamma})|\leq M_R$.
        \item To bound the second order moment, we have
            \begin{equation}
                \begin{aligned}
                    &\E_{\mu(z)} \left[\mathbbm{1}_{\|z\|_\infty\leq R} \ell(z,f_{\gamma})^2\right] \\
                    &= \E_{\mu(z)}\left[\mathbbm{1}_{\|z\|_\infty\leq R} \left(\E_{q(x|z)}\|f_{\gamma}(x)-\nabla q(x|z)\|^2-\|\nabla\log\hat{\mu}(x)-\nabla\log q(x|z)\|^2\right)^2\right] \\
                    &\leq \E_{\mu(z)} \left[\mathbbm{1}_{\|z\|_\infty\leq R}\left(\E_{q(x|z)}\|f_{\gamma}(x)-\nabla\log\hat{\mu}(x)\|^2\right)\right.\\
                    &\qquad\qquad\qquad \left.\cdot \left(\E_{q(x|z)}\|f_{\gamma}(x)+\nabla\log\hat{\mu}(x)-2\nabla\log q(x|z)\|^2\right)\right] \\
                    &\leq 4M_R\E_{\mu(z)} \left[\mathbbm{1}_{\|z\|_\infty\leq R}\left(\E_{q(x|z)}\|f_{\gamma}(x)-\nabla\log\hat{\mu}(x)\|^2\right)\right] \\
                    &\leq 4M_R \E_{\mu(z)}\ell(z,f_{\gamma}) \\
                    &\leq 4M_R\E_{\mu(z)} [\widetilde{\ell}(z,f_{\gamma})] + 8M_R^2\exp(-C'R^2).
                \end{aligned}
            \end{equation}
        \item To bound the local Rademacher complexity, note that
            \begin{equation}
                \Big\|\frac{1}{\sqrt{n}} \sum_{i=1}^n\sigma_i\widetilde{\ell}(z_i,f_{\gamma,1}) - \frac{1}{\sqrt{n}}\sum_{i=1}^n\sigma_i\widetilde{\ell}(z_i,f_{\gamma,2}) \Big\|_{\psi_2} \leq 4\|\widetilde{\ell}(\cdot,f_{\gamma,1})-\widetilde{\ell}(\cdot,f_{\gamma,2})\|_{L^2(\widehat{\P}_n)},
            \end{equation}
            where $\widehat{\P}_n:=\frac{1}{n}\sum_{i=1}^n\delta_{z_i}$.
            Define $\Phi_r:=\{\varphi\in\Phi:\frac{1}{m}\sum_{i=1}^m\varphi(z_i)^2\leq r\}$
            and it is easy to show that $\textbf{diam}\big(\Phi_r,\|\cdot\|_{L^2(\widehat{\P}_n)}\big)\leq 2\sqrt{r}$.
            By Dudley's bound \citep{van2014probability,wainwright2019high}, there exists an absolute constant $C_0$ such that for any $\theta>0$,
            \begin{equation}\label{eq:rademacher_bound_2}
                \mathcal{R}_n(\Phi_r)\leq C_0\left(\theta+\int_\theta^{2\sqrt{r}}\sqrt{\frac{\log\mathcal{N}(\Phi_r,\|\cdot\|_{L^2(\widehat{\P}_n)},\varepsilon)}{n}}\ d\varepsilon\right).
            \end{equation}
            Since $\|z_i\|_\infty\leq R$,
            \begin{equation}
                \begin{aligned}
                    \frac{1}{n}\sum_{i=1}^n(\widetilde{\ell}(z_i,f_{\gamma,1})-\widetilde{\ell}(z_i,f_{\gamma,2}))^2
                    &= \frac{1}{n}\sum_{i=1}^n(\ell(z_i,f_{\gamma,1})-\ell(z_i,f_{\gamma,2}))^2 \\
                    &\leq \frac{1}{n}\sum_{i=1}^n \left[\E_{q(x|z_i)}\|f_{\gamma,1}-f_{\gamma,2}\|^2\right]\cdot\left[\E_{q(x|z_i)}\|f_{\gamma,1}+f_{\gamma,2}-2\nabla_x\log q\|^2\right] \\
                    &\leq \frac{4M_R}{n}\sum_{i=1}^n \E_{q(x|z_i)}\|f_{\gamma,1}(x)-f_{\gamma,2}(x)\|^2.
                \end{aligned}
            \end{equation}
            Let $R_1=2R$. Since $x|z_i\sim\mathcal{N}(x;z_i,\sigma^2I)$, we have $\P(\|x\|_\infty\geq R_1)\leq d\P(|\mathcal{N}(0,1)|\leq R)\leq 2d\exp(-C_0'R^2)$ for some absolute constant $C_0'$.
            Therefore,
            \begin{equation}
                \begin{aligned}
                    &\E_{q(x|z_i)}\|f_{\gamma,1}(x)-f_{\gamma,2}(x)\|^2 \\
                    &\qquad \leq \E_{q(x|z_i)}[\mathbbm{1}_{\|x_t\|\leq R_1}] [\|f_{\gamma,1}(x)-f_{\gamma,2}(x)\|^2] + 8dM_R^2\exp(-C_0'R^2) \\
                    &\qquad \leq \|f_{\gamma,1}-f_{\gamma,2}\|^2_{L^\infty(\Omega_{R_1})} + 8dM^2\exp(-C_0'R^2)
                \end{aligned}
            \end{equation}
            where $\Omega_{R_1}:=[-R_1,R_1]^{d}$. Plug in the bound above,
            \begin{equation}
                \sqrt{\frac{1}{n}\sum_{i=1}^n(\widetilde{\ell}(z_i,f_{\gamma,1})-\widetilde{\ell}(z_i,f_{\gamma,2}))^2}
                \leq 4M_R^{\frac{1}{2}}\|f_{\gamma,1}-f_{\gamma,2}\|_{L^\infty(\Omega_{R_1})} + 8d^{\frac{1}{2}}M_R\exp(-C_0'R^2/2).
            \end{equation}
            For any $\varepsilon\geq 16d^{\frac{1}{2}}M_R\exp(-C_0'R^2/2)$, according to Lemma \ref{lem:covering_num},
            \begin{equation}
                \begin{aligned}
                    \log\mathcal{N}(\Phi_r,\|\cdot\|_{L^2(\widehat{\P}_n)},\varepsilon)
                    &\leq \log\mathcal{N}(\mathcal{S},\|\cdot\|_{L^\infty(\Omega_{R_1})},\varepsilon/(8M^{\frac{1}{2}})) \\
                    &\leq C_4SL\log\left(\frac{LW(B\vee 1)RM}{\varepsilon}\right).
                \end{aligned}
            \end{equation}
            Plug in \eqref{eq:rademacher_bound_2} and let $\theta=16d^{\frac{1}{2}}M_R\exp(-C_0'R^2/2)$,
            \begin{equation}
                \begin{aligned}
                    \mathcal{R}_n(\Phi_r)
                    &\leq C_0\left(\theta+\int_\theta^{2\sqrt{r}}\sqrt{\frac{C_4SL\log\left(\frac{LW(B\vee 1)RM}{\varepsilon}\right)}{n}}d\varepsilon\right) \\
                    &\leq C_0\left(16d^{\frac{1}{2}}M_R\exp(-C_0'R^2/2)+\sqrt{\frac{C_4'SL\log\left(\frac{LW(B\vee 1)RM}{r}\right)\cdot r}{n}}\right) \\
                    &=: \widetilde{\mathcal{R}}_n(r)
                \end{aligned}
            \end{equation}
    \end{enumerate}
    Combine the three steps above, by Lemma \ref{lem:local_rademacher} with $B_0=8M_R^2\exp(-C'R^2),B=4M_R,b=M_R$, it holds that with probability no less than $1-2n\exp(-C'R^2)-\delta/2$, for any $f\in\mathcal{F}$,
    \begin{equation}\label{eq:bound_erm_1}
        \begin{aligned}
            \E_{\mu(z)} [\widetilde{\ell}(z,f_{\gamma})]
            &\leq \frac{2}{n}\sum_{i=1}^n \widetilde{\ell}(z_i,f_{\gamma}) + C_5M_R\left(r_n^*+\frac{\log(\log(n)/\delta)}{n}\right) \\
            &\qquad + C_5\sqrt{\frac{M_R^2\log(\log(n)/\delta)}{n}}\exp(-C'R^2),
        \end{aligned}
    \end{equation}
    \begin{equation}\label{eq:bound_erm_2}
        \begin{aligned}
            \frac{1}{n}\sum_{i=1}^n \widetilde{\ell}(z_i,f_{\gamma})
            &\leq 2\E_{\mu(z)} [\widetilde{\ell}(z,f_{\gamma})] + C_5M_R\left(r_n^*+\frac{\log(\log(n)/\delta)}{n}\right) \\
            &\qquad + C_5\sqrt{\frac{M_R^2\log(\log(n)/\delta)}{n}}\exp(-C'R^2).
        \end{aligned}
    \end{equation}
    where $r_n^*$ is the largest fixed point of $\widetilde{\mathcal{R}}_n$, and it can be bounded as
    \begin{equation}
        r_n^*\leq C_6\left(d^{\frac{1}{2}}M_R\exp(-C'R^2/2)+\frac{SL\log\left(nLW(B\vee 1)RM_R\right)}{n}\right),
    \end{equation}
    for some absolute constant $C_6$.
    Moreover, we have
    \begin{equation}
        \left|\E_{\mu(z)}[\ell(z,f_{\gamma})]
        - \E_{\mu(z)}[\widetilde{\ell}(z,f_{\gamma})]\right| \leq 2M_R\exp(-C'R^2).
    \end{equation}
    Combine this with \eqref{eq:bound_erm_1},\eqref{eq:bound_erm_2},
    \begin{equation}\label{eq:bound_erm_3}
        \E_{\mu(z)} [\ell(z,f_{\gamma})]
        \leq \frac{2}{n}\sum_{i=1}^n \ell(z_i,f_{\gamma}) + C_5M_R\left(r_n^*+\frac{\log(\log(n)/\delta)}{n}+\exp(-C'R^2)\right),
    \end{equation}
    \begin{equation}\label{eq:bound_erm_4}
        \frac{1}{n}\sum_{i=1}^n \ell(z_i,f_{\gamma})
        \leq 2\E_{\mu(z)} [\ell(z,f_{\gamma})]
        + C_5M_R\left(r_n^*+\frac{\log(\log(n)/\delta)}{n}+\exp(-C'R^2)\right)
    \end{equation}
    
    Plug in the definition of $M_R=C_x'\left(M^2+\frac{R^6+d^3\sigma^6}{\sigma^8}\right)$ and let $R=C_6\log^{\frac{1}{2}}(ndM/\sigma\delta)$ for some large constant $C_6$. Hence \eqref{eq:bound_erm_3} and \eqref{eq:bound_erm_4} reduce to
    \begin{align}
        \E_{\mu(z)} [\ell(z,f_{\gamma})]
        \leq \frac{2}{n}\sum_{i=1}^n \ell(z_i,f_{\gamma}) + C_7\sigma^{-8}M^2\log^3(nd/\sigma\delta)\left(r_n^\dagger+\frac{\log(\log(n)/\delta)}{n}\right), \\
        \frac{1}{n}\sum_{i=1}^n \ell(z_i,f_{\gamma})
        \leq 2\E_{\mu(z)} [\ell(z,f_{\gamma})]
        + C_7\sigma^{-8}M^2\log^3(nd/\sigma\delta)\left(r_n^\dagger+\frac{\log(\log(n)/\delta)}{n}\right),
    \end{align}
    where $r_n^\dagger:=\frac{SL\log\left(nLW(B\vee 1)M\log(1/\delta)\right)}{n}$.

    Therefore, we obtain that with probability no less than $1-\delta$, the population loss of the empirical minimizer $\hat{f_{\gamma}}$ can be bounded by
    \begin{equation}
        \begin{aligned}
            &\E_{\hat{\mu}(x)}[\|\hat{f_{\gamma}}(x)-\nabla\log\hat{\mu}(x)\|^2]\\
            =& \E_{\mu(z)}[\ell(z,\hat{f_{\gamma}})] \\
            \leq& \frac{2}{n}\sum_{i=1}^n \ell(z_i,\hat{f_{\gamma}}) + 2C_7\sigma^{-8}M^2\log^3(nd/\sigma\delta)\left(r_n^\dagger+\frac{\log(1/\delta)}{n}\right) \\
            \leq& \inf_{f_{\gamma}\in\mathcal{S}}\frac{2}{n}\sum_{i=1}^n \ell(z_i,\hat{f_{\gamma}}) + 2C_7\sigma^{-8}M^2\log^3(nd/\sigma\delta)\left(r_n^\dagger+\frac{\log(1/\delta)}{n}\right) \\
            \leq& 4\inf_{f_{\gamma}\in\mathcal{S}}\E_{\mu(z)}[\ell(z,f_{\gamma})] + 6C_7\sigma^{-8}M^2\log^3(nd/\sigma\delta)\left(r_n^\dagger+\frac{\log(1/\delta)}{n}\right) \\
            =& 4\inf_{f_{\gamma}\in\mathcal{S}}\E_{\hat{\mu}(x)}[\|f_{\gamma}(x)-\nabla\log\hat{\mu}(x)\|^2] + 6C_7\sigma^{-8}M^2\log^3(nd/\sigma\delta)\left(r_n^\dagger+\frac{\log(1/\delta)}{n}\right).
        \end{aligned}
    \end{equation}
\end{proof}

\begin{lemma}\label{lem:lip_score}
    Suppose that $\mu(\cdot)$ is $C$-sub-Gaussian.
    There exists some constant $C_x$ such that the score $\nabla\log\hat{\mu}(x)$ is $\frac{C_x}{\sigma^4}(\|x\|^2+1)$-Lipschitz in $x$.
\end{lemma}

\begin{proof}
    Define the posterior density $\mu(z|x)\propto \mu(z)q(x|z)$.
    We rewrite the score function as $\nabla\log\hat{\mu}(x)=\int \nabla\log q(x|z)\frac{q(x|z)\mu(z)dz}{\int q(x|z)\mu(z)dz}$, and it yields
    \begin{equation}
        \begin{aligned}
            \nabla_x^2\log\hat{\mu}(x)
            &= \E_{\mu(z|x)} \left[\nabla_x^2\log q(x|z)\right] + \Var_{\mu(z|x)}(\nabla_x\log q(x|z)) \\
            &= -\frac{I}{\sigma^2} + \Var_{\mu(z|x)}\Big(\frac{z-x}{\sigma^2}\Big).
        \end{aligned}
    \end{equation}
    For any $R>0$, we have
    \begin{equation}
        \begin{aligned}
            \Var_{\mu(z|x)}\Big(\frac{z-x}{\sigma^2}\Big)
            &\preceq \frac{1}{\sigma^2}\int \big\|\frac{z-x}{\sigma}\big\|^2\frac{q(x|z)\mu(z)}{\int q(x|z)\mu(z)d z}d z \\
            &\leq \frac{R^2}{\sigma^2} + \frac{\int_{\|\frac{z-x}{\sigma}\|\geq R} \|\frac{z-x}{\sigma}\|^2\exp\left(-\frac{\|z-x\|^2}{2\sigma^2}\right)\mu(z)d z}{\sigma^2\int \exp\left(-\frac{\|z-x\|^2}{2\sigma^2}\right)\mu(z)d z} \\
            &\leq \frac{R^2}{\sigma^2} + \frac{\int_{\|\frac{z-x}{\sigma}\|\geq R} \exp(-\frac{R^2}{4})\mu(z)dz}{\sigma^2\int_{\|\frac{z-x}{\sigma}\|\leq R/2} \exp(-\frac{R^2}{8})\mu(z)d z}.
        \end{aligned}
    \end{equation}
    Let $R=\frac{2\|x\|+2C_0}{\sigma}$, then the domain $\Big\{z:\|\frac{z-x}{\sigma}\|\leq R/2\Big\}$ includes $\Big\{z:\|z\|\leq C_0\Big\}$, indicating
    \begin{equation}
        \begin{aligned}
            &\int_{\|\frac{z-x}{\sigma}\|\leq R/2} \mu(z)dz\geq  \int_{\|z\|\leq C_0} \mu(z)d z \geq 1-2\exp(-C'C_0^2)\geq \frac{1}{2},\\
            &\int_{\|\frac{z-x}{\sigma}\|\geq R} \mu(z)d z\leq  \int_{\|z\|\geq C_0} \mu(z)dz \leq \frac{1}{2}.
        \end{aligned}
    \end{equation}
    and
    \begin{equation}\label{eq:lip_x_large}
        \|\nabla_x^2\log\hat{\mu}(x)\|\leq \frac{1}{\sigma^2}+\big\|\Var_{\mu(z|x)}\Big(\frac{z-x}{\sigma^2}\Big)\big\|
        \leq \frac{R^2}{\sigma^2} + \frac{2}{\sigma^2} \leq \frac{8\|x\|^2+8C_0^2+2\sigma^2}{\sigma^4}.
    \end{equation}
\end{proof}

\begin{lemma}[Lemma 7, \citet{chen2022nonparametric}]\label{lem:covering_num}
    The covering number of $\mathcal{F}=\mathcal{S}(M,W,B,L,S)$ can be bounded by   
    \begin{equation}
        \log \mathcal{N}(\mathcal{F},\|\cdot\|_{L^\infty([-R,R]^{d})},\varepsilon) \lesssim SL\log\left(\frac{LW(B\vee 1)R}{\varepsilon}\right).
    \end{equation}
\end{lemma}

\begin{lemma}[Lemma A.11, \citet{cheng2025provable}]\label{lem:local_rademacher}
    Let $\Phi$ be a class of functions on domain $\Omega$ and $\P$ be a probability distribution over $\Omega$. 
    Suppose that for any $\varphi\in\Phi$, $\|\varphi\|_{L^\infty(\Omega)}\leq b$, $\E_\P [\varphi]\geq 0$, and $\E_\P [\varphi^2] \leq B\E_\P [\varphi]+B_0$ for some $b,B,B_0\geq 0$. 
    Let $x_1,\cdots,x_n\overset{\textit{i.i.d.}}{\sim}\P$ and $\phi_n$ be a positive, non-decreasing and sub-root function such that
    \begin{equation}
        \mathcal{R}_n(\Phi_r):=\E_{\bm{\sigma}} \sup_{\varphi\in\Phi_r}\Big|\frac{1}{n}\sum_{i=1}^n\sigma_i\varphi(x_i)\Big|\leq \phi_n(r).
    \end{equation}
    where $\Phi_r:= \Big\{\varphi\in\Phi: \frac{1}{n}\sum_{i=1}^n{(\varphi(x_i))^2}\leq r\Big\}$.
    Define the largest fixed point of $\phi_n$ as $r_n^*$.
    Then for some absolute constant $C'$, with probability no less than $1-\delta$, it holds that for any $\varphi\in\Phi$,
    \begin{align}
        &\E_\P[\varphi]\leq \frac{2}{n}\sum_{i=1}^n\varphi(x_i) + C'(B\vee b)\left(r_n^* + \frac{\log\big((\log n)/\delta\big)}{n}\right)+C'\sqrt{\frac{B_0\log\big((\log n)/\delta\big)}{n}},\\
        &\frac{1}{n}\sum_{i=1}^n\varphi(x_i)\leq 2\E_\P[\varphi] + C'(B\vee b)\left(r_n^* + \frac{\log\big((\log n)/\delta\big)}{n}\right)+C'\sqrt{\frac{B_0\log\big((\log n)/\delta\big)}{n}}. 
    \end{align}
\end{lemma}

\subsection{Approximation Bounds: Proof of Theorem \ref{thm::nn approx} and Theorem \ref{thm::nn approx2}}  \label{appd:proof3}

\subsubsection{Main Lemmas}

Lemma \ref{lem:nnproduct} and Lemma \ref{lem:nnreciprocal}
from \cite{fu2024unveil} focus on the complexity of the neural network to approximate the basic operations.

\begin{lemma}\label{lem:nnproduct}
    (Approximating the product) Let $d\geq 2$, $C \geq 1$.
    For any $\epsilon_{\text{product}}>0$, there exists $f_{\text{mult}}(x_1,x_2,\cdots,x_d)\in \mathcal{S}( W, B, L, S)$ with
   $L = \mathcal{O}(\log d(\log \epsilon_{\text{product}}^{-1}+ d \log C)), W = 48d, S = \mathcal{O}(d \log \epsilon_{\text{product}}^{-1} + d\log C)), B=C^d$
    such that
    \begin{align}
     &   \left|f_{\text{mult}}(x_1',x_2',\cdots,x_d') - \prod_{i=1}^d x_{i}\right| \leq \epsilon_{\text{product}} + d C^{d-1} \epsilon_1.
    \end{align}
    for all $x\in [-C,C]^d\text{ and } x'\in \R$ with $\|x-x'\|_\infty \leq \epsilon_1$. 
    $|f_{\text{mult}}(x)|\leq C^d$ for all $x\in \R^d$, and $f_{\text{mult}}(x_1',x_2',\cdots,x_d')=0$ if at least one of $x_i'$ is $0$.
\end{lemma}

\begin{lemma}\label{lem:nnreciprocal}
    (Approximating the reciprocal function) For any $0<\epsilon_{\text{inv}} <1$, there exists $f_{-1} \in \mathcal{S}( W, B, L, S)$ with $L= \mathcal{O}(\log^2 \epsilon_{\text{inv}}^{-1}), W = \mathcal{O}(\log^3 \epsilon_{\text{inv}}^{-1}), S = \mathcal{O}(\log^4 \epsilon_{\text{inv}}^{-1})$, and $B= \mathcal{O}(\epsilon_{\text{inv}}^{-2})$ such that
    \begin{align}
        \left|f_{-1}(x') - \frac{1}{x}\right| \leq \epsilon_{\text{inv}} + \frac{|x'-x|}{\epsilon_{\text{inv}}^2}, \quad \text{for all }x\in [\epsilon_{\text{inv}},\epsilon_{\text{inv}}^{-1}] \text{ and }x'\in \R.
    \end{align}
\end{lemma}

Additionally, we state the smoothness of $\hat{\mu}$ due to Gaussian convolution by the following lemma.

\begin{lemma}\label{lem:holder}
    (Smoothness of Gaussian convolution) Given Gaussian transition kernel $k_\sigma(x)=Z_0 e^{-\frac{\|x\|^2}{2\sigma^2}}$, $\sigma>0$, where $Z_0$ satisfies $\int k_\sigma(x) dx=1$. For any probability distribution $\mu(z)$, define $\hat{\mu}(x)=\int k_\sigma(x-z)\mu(z)dz$. For any $\beta\in\mathbb{N}$, the norm of $\beta$-th derivative of $\hat{\mu}$ is bounded by 
    $\mathcal{O}(\sigma^{-(\beta+d)}\beta^{\frac{\beta}{2}})$, where $\mathcal{O}$ hides the other polynomial factors.
\end{lemma}

\begin{proof}
For $\sum_{i=1}^d n_i=\beta$, by the derivative property of Gaussian convolution, we have
\begin{equation}
\begin{aligned}
\frac{\partial^{\beta}}{\partial x_1^{n_1}\cdots x_d^{n_d}}\hat{\mu}=Z_0\int \mu(z) \prod_{i=1}^d \frac{\partial^{n_i} }{\partial x_i^{n_i}}e^{\frac{-(x_i-z_i)^2}{2\sigma^2}}d z
\end{aligned}
\end{equation}

The derivative of Gaussian distribution is related to the Hermite polynomials. From \cite{kra2004hermite}, we have $|\frac{\partial^{n_i} }{\partial x_i^{n_i}}e^{\frac{-(x_i-z_i)^2}{2\sigma^2}}|\le \mathcal{O}(\sigma^{-n_i}(n_i!)^{\frac{1}{2}})$ for any $x_i, z_i$. Hence, $|\prod_{i=1}^d \frac{\partial^{n_i} }{\partial x_i^{n_i}}e^{\frac{-(x_i-z_i)^2}{2\sigma^2}}|\le \mathcal{O}(\sigma^{-\beta}\prod_{i=1}^d (n_i !)^{\frac{1}{2}})\le \mathcal{O}(\sigma^{-\beta}(\beta !)^{\frac{1}{2}})\le \mathcal{O}(\sigma^{-\beta}\beta^{\frac{\beta}{2}})$. Since $Z_0\le \mathcal{O}(\sigma^{-d})$, then $|\frac{\partial^{\beta}}{\partial x_1^{n_1}\cdots x_d^{n_d}}\hat{\mu}|\le \mathcal{O}(\sigma^{-(\beta+d)}\beta^{\frac{\beta}{2}})$.

\end{proof}

\subsubsection{Proof of Theorem \ref{thm::nn approx} and Theorem \ref{thm::nn approx2}}\label{appd:d_3_2}

We prove a general version of Theorem \ref{thm::nn approx} by extending the exponent $\alpha$ in Assumption \ref{assp:moment} to $\alpha\le p$, where the constant $p> 2\in \mathbb{N}_{+}$. Theorem \ref{thm::nn approx} is the special case when $p=5$.

\begin{theorem}\label{appendix:thm::nn approx}
Under generalized version of Assumption \ref{assp:moment} (where $\alpha\le 5$ is replaced with $\alpha\le p$, the constant $p> 2\in \mathbb{N}_{+}$), for sufficiently large $N$,
there exists $f_{\gamma} \in \mathcal{S}(M,W,B,L,S)$ such that 
\begin{align*}
    \ell_{sm}(f_{\gamma})=\mathcal{O}(\sigma^{-(p+2-o(1))}N^{-\frac{(p-2-o(1))}{d}}).
\end{align*}
The hyperparameters in the ReLU neural network class $\mathcal{S}$ satisfy
\begin{align*}
    &M = \mathcal{O}(\sigma^{-(1-o(1))}N^{\frac{(1-o(1))}{d}}),~
     W = \mathcal{O}(N),~
     B =\mathcal{O}(\sigma^{(2d+3p-o(1))}N^{\frac{(2d+3p-o(1))}{d}}),~\\
    &L = \mathcal{O}(\log^2 (\sigma^{-1})\log^2 N),~
     S= \mathcal{O}\left(\log(\sigma^{-1})N\right).
\end{align*}
where $\mathcal{O}$ hides the other factors that are not related to $N$ or $\sigma$.
\end{theorem}

\begin{proof}

For $R>0$, we have
\begin{equation}\label{app: three_terms}
\begin{aligned}
&\int_{\R^d}\|f_{\gamma}(x)-\nabla\log\hat{\mu}_t(x)\|_2^2 \hat{\mu}_t(x)d x \\
=& \int_{\|x\|_{2}> R}\|f_{\gamma}(x)-\nabla\log\hat{\mu}_t(x)\|_2^2 \hat{\mu}_t(x)d x  \\
+&\int_{\|x\|_{2}\le R} \textbf{1}\{|\hat{\mu}_t(x)|<\epsilon_{\rm low}\} \|f_{\gamma}(x)-\nabla\log\hat{\mu}_t(x)\|_2^2 \hat{\mu}_t(x)d x  \\
+&\int_{\|x\|_{2}\le R} \textbf{1}\{|\hat{\mu}_t(x)|\ge \epsilon_{\rm low}\} \|f_{\gamma}(x)-\nabla\log\hat{\mu}_t(x)\|_2^2 \hat{\mu}_t(x)d x  \\
:=&A_1+A_2+A_3
\end{aligned}
\end{equation}

\begin{enumerate}[label=\textbf{Step \arabic*.}]

\item 
For term $A_1$, we use the assumption that the $p$-th moment of $\mu_t$ is finite:
\begin{equation}
\int \|z\|^p d\mu_t(z)\le m_{p}<\infty.
\end{equation}

Denote the fifth moment of the Gaussian distribution $\E_{w\sim \mathcal{N}(0,\sigma^2\mathbf{I})} \|w\|^{p}:= \sigma_{p}$.

We have 
\begin{equation}
\int \|x\|^{p} d\hat{\mu}_t(x)\le 16(m_{p}+\sigma_{p})<\infty.
\end{equation}

Then by Chebyshev inequality:
\begin{equation}
    \int_{\|x\|_2\ge R} d \hat{\mu}_t(x) \lesssim R^{-p}
\end{equation}

Since $\int \|z\|^2 d\mu_t(z)\le m_2<\infty$, note that
\begin{equation}
    \nabla\log \hat{\mu}_t(x)=-\frac{1}{\sigma^2}\frac{\int (x-y) k_\sigma (x-y) d \mu_t(y)}{\int k_{\sigma}(x-y)d\mu_t(y)}
\end{equation}

From Lemma \ref{lem:convolution}, we have $\|\nabla\log \hat{\mu}_t(x)\|\le \sigma^{-2}(C_0\|x\|+M_0)$.

Then by Chebyshev inequality:
\begin{equation}
    \int_{\|x\|_{2}\ge R} \|\nabla \log \hat{\mu}_t(x)\|_2^2 \hat{\mu}_t(x) d x\lesssim \int_{\|x\|_{2}\ge R} (\sigma^{-4}\|x\|^2_{2}+1) \hat{\mu}_t(x) d x\lesssim  \sigma^{-4}R^{-(p-2)}
\end{equation}

Summing up the above, and plugging in the upper bound of the neural network $M$, we obtain the upper bound for $A_1$:
\begin{equation}
\begin{aligned}
    A_1&\le 2\int_{\|x\|_{2}> R}\|f_{\gamma}(x)\|_2^2 \hat{\mu}_t(x)d x +2\int_{\|x\|_{2}> R} \|\nabla\log\hat{\mu}_t(x)\|_2^2 \hat{\mu}_t(x)d x \\
    &\lesssim d M^2 \int_{\|x\|_{2}> R} \hat{\mu}_t(x)d x  + \sigma^{-4}R^{-(p-2)}\\
    &\lesssim d M^2 R^{-p} + \sigma^{-4}R^{-(p-2)}
\end{aligned}
\end{equation}

\item 
As for $A_2$, using $\|\nabla\log \hat{\mu}_t(x)\|\le C_0\|x\|+M_0$, we obtain
\begin{equation}
\begin{aligned} &\int_{\|x\|_{2}\le R} \textbf{1}\{|\hat{\mu}_t(x)|<\epsilon_{\rm low}\} \hat{\mu}_t(x)dx \lesssim R^{d}\epsilon_{\rm low},
    \\ &\int_{\|x\|_{2}\le R} \textbf{1}\{|\hat{\mu}_t(x)|<\epsilon_{\rm low}\}\|\nabla\log\hat{\mu}_t(x)\|^2\hat{\mu}_t(x)dx \lesssim \sigma^{-4}R^{d+2}\epsilon_{\rm low},
    \end{aligned}
\end{equation}

Then 
\begin{equation}
\begin{aligned}
    A_2&\le 2\int_{\|x\|_{2}\le R} \textbf{1}\{|\hat{\mu}_t(x)|<\epsilon_{\rm low}\} (dM^2+\|\nabla\log\hat{\mu}_t(x)\|^2) \hat{\mu}_t(x)d x \\
    &\lesssim (d M^2 R^{d}+ \sigma^{-4}R^{d+2})\epsilon_{\rm low}
\end{aligned}
\end{equation}

\item  
Finally, for $A_3$, since $\|x\|_2\le R$ induces $\|x\|_{\infty}\le R$, we first propose the local Taylor polynomial approximation for $\|x\|_{\infty}\le R$, then calculate the complexity of the neural network to approximate the above Taylor polynomial approximation.

By reparameterization, we denote $f(x)=\hat{\mu}_t(2R(x-\frac{1}{2})), x\in[0,1]^{d}$, then $\hat{\mu}_t(x)=f(\frac{x}{2R}+\frac{1}{2}) $.

Then from Lemma \ref{lem:holder}, $\|f\|_{\mathcal{H}^{\beta}([0,1]^d)}\le B_0 (2R)^{\beta}$, where $B_0= \mathcal{O}(\sigma^{-(\beta+d)}\beta^{\frac{\beta}{2}})$.

We use Taylor expansion $q$ to approximate $f$:
\begin{equation}
    q(x)=\sum_{v\in[N]^d}\sum_{\|n\|_1\le \beta}\frac{1}{n!}\frac{\partial^n f}{\partial x^n}\Bigg | _{x= \frac{v}{N}} \textbf{1}\{x\in(\frac{v-1}{N},\frac{v}{N}]\}\prod_{i=1}^{d} (x_i-\frac{v_i}{N})^{n_i}
\end{equation}

Since for any $x\in(\frac{v-1}{N},\frac{v}{N}]$, $\theta\in[0,1]$:
\begin{equation}
\begin{aligned}
    &\Bigg |\sum_{\|n\|_1= \beta}\frac{1}{n!}\frac{\partial^n f}{\partial x^n}\Bigg | _{\frac{v}{N}} \prod_{i=1}^{d} (x_i-\frac{v_i}{N})^{n_i}-\sum_{\|n\|_1= \beta}\frac{1}{n!}\frac{\partial^n f}{\partial x^n}\Bigg | _{(1-\theta)\frac{v}{N}+\theta x} \prod_{i=1}^{d} (x_i-\frac{v_i}{N})^{n_i} \Bigg|\\
    \le& 2B_0 (2R)^{\beta} N^{-\beta} \frac{d^{\beta}}{\beta !}
\end{aligned}    
\end{equation}

Then by Lagrangian mean-value theorem, we have
\begin{equation}
    |f(x)-q(x)|\lesssim B_0 (2R)^{\beta} N^{-\beta} \frac{d^{\beta}}{\beta !}
\end{equation}

We denote the approximation polynomial for $\hat{\mu}_t(x)$ is $f_1(x)=q(\frac{x}{2R}+\frac{1}{2})$, and its corresponding RELU neural network is $f_1^{RELU}$. 
For any $\beta_0\in [1,\beta]$, consider the case $\|n\|_1=\beta_0$. Assume the approximation error of a single term $ \prod_{i=1}^{d} (x_i-\frac{v_i}{N})^{n_i}$ is $\epsilon_{\text{product}}$ and the corresponding complexity is $\kappa$ (here $\kappa$ denotes the complexity like $S$ that add up in neural network summation like Lemma F.3 in \cite{fu2024unveil}), then the error of $\sum_{v\in[N]^d}\sum_{\|n\|_1= \beta_0}\frac{1}{n!}\frac{\partial^n f}{\partial x^n}\Bigg | _{x= \frac{v}{N}} \textbf{1}\{x\in(\frac{v-1}{N},\frac{v}{N}]\}\prod_{i=1}^{d} (x_i-\frac{v_i}{N})^{n_i}$ is less than $\beta_0 ! d^{\beta_0}R^{\beta_0}\sigma^{-(d+\beta_0)}\beta_0^{\frac{\beta_0}{2}}N^{d}\epsilon_{\text{product}}$. The total polynomial error is less than $\sum_{\beta_0=1}^{\beta}\beta_0! d^{\beta_0}R^{\beta_0}\sigma^{-(d+\beta_0)}\beta_0^{\frac{\beta_0}{2}}N^{d}\epsilon_{\text{product}}\le (\beta+1)! d^{\beta}R^{\beta}\sigma^{-(d+\beta)}\beta^{\frac{\beta}{2}}N^{d}\epsilon_{\text{product}}\le \epsilon_{\text{poly}}$ and the total complexity is
\begin{equation}\label{append:taylorsumcomplexity}    N^d\sum_{\beta_0=1}^{\beta}\binom{d+\beta_0-1}{d-1} \kappa\le N^d \binom{d+\beta}{d}\kappa\le \mathcal{O}(N^d\beta^d\kappa).
\end{equation}
Since $\|n\|_1\le \beta$ and $|x_i-\frac{v_i}{N}|\le \frac{1}{N}$. From Lemma \ref{lem:nnproduct}, we have the complexity of approximating the single term $ \prod_{i=1}^{d} (x_i-\frac{v_i}{N})^{n_i}$ is not greater than:
\begin{equation}\label{append:taylorsinglecomplexity}
    L = \mathcal{O}(\log \beta(\log \epsilon_{\text{product}}^{-1}-\beta\log N)), W = 48\beta, S = \mathcal{O}(\beta\log \epsilon_{\text{product}}^{-1}-\beta\log N), B=\mathcal{O}(1)
\end{equation}

The structure is the same for approximating $\sigma \nabla \hat{\mu}_t$ by separately approximating each dimension $[\sigma \nabla \hat{\mu}_t]_i$, the only difference is that $\beta$ changes to $\beta-1$, which does not affect the order of the complexity for $\beta>1$. Denote the approximation polynomial for $\sigma \nabla \hat{\mu}_t$ is $f_2$, and its corresponding RELU neural network is $f_2^{RELU}$.

From Lemma \ref{lem:nnreciprocal}, we have the neural network complexity to approximate the reciprocal of $f_1^{RELU}$. There exists the inverse $f_{-1} \in \mathcal{S}( W, B, L, S)$ with $L= \mathcal{O}(\log^2 \epsilon_{\text{inv}}^{-1}), W = \mathcal{O}(\log^3 \epsilon_{\text{inv}}^{-1}), S = \mathcal{O}(\log^4 \epsilon_{\text{inv}}^{-1})$, and $B= \mathcal{O}(\epsilon_{\text{inv}}^{-2})$ such that
\begin{equation}\label{append:inversecomplexity}
        \left|f_{-1}(f_1^{RELU}) - \frac{1}{f_1}\right| \leq \epsilon_{\text{inv}} + \frac{|f_1^{RELU}-f_1|}{\epsilon_{\text{inv}}^2}, \quad \text{for all }f_1\in [\epsilon_{\text{inv}},\epsilon_{\text{inv}}^{-1}] \text{ and }f_1^{RELU}\in \R.
    \end{equation}

Since the deviation between $\hat{\mu}_t$ and $f_1$ is upper bounded
by $B_0 (2R)^{\beta} N^{-\beta} \frac{d^{\beta}}{\beta !}$, then
there exists a constant $C_{\text{low}}>1$ that when $\epsilon_{\text{low}}:= C_{\text{low}} B_0 (2R)^{\beta} N^{-\beta} \frac{d^{\beta}}{\beta !}$,
we have $f_1(x)>\frac{1}{2}\hat{\mu}_t(x)$.

Denote clipped $f_{1,\text{clip}}=\max\{f_1,\epsilon_{\text{low}}\}$ (note that this does not affect the approximation effect for $ \hat{\mu}_t>\epsilon_{\text{low}}$), and the score estimator $f_3:=\min \{\frac{f_2}{\sigma f_{1,\text{clip}}}, \frac{1}{\sigma^2}\mathcal{O}(R)\}$, note that $\|\nabla\log\hat{\mu}_t\|_{\infty}\le \mathcal{O}(\sigma^{-2}R)$, hence $M=\mathcal{O}(\sigma^{-2}R)$.

We have
\begin{align}
    |\nabla\log\hat{\mu}_{t,1}-f_3 |&\le |\nabla\log\hat{\mu}_{t,1}-\frac{f_2}{\sigma f_{1,\text{clip}}} |\\
    &=|\frac{[\nabla\hat{\mu}_t]_1}{\hat{\mu}_t}-\frac{[\nabla\hat{\mu}_t]_1}{f_{1,\text{clip}}}+\frac{[\nabla\hat{\mu}_t]_1}{f_{1,\text{clip}}}-\frac{f_2}{\sigma f_{1,\text{clip}}}|\\
    &\le |[\nabla\hat{\mu}_t]_1||\frac{1}{\hat{\mu}_t}-\frac{1}{f_{1,\text{clip}}}|+\frac{|\sigma [\nabla\hat{\mu}_t]_1-f_2|}{\sigma f_{1,\text{clip}}}
\end{align}

By $\|\nabla\log\hat{\mu}_t\|_{\infty}\le \mathcal{O}(\sigma^{-2}R)$ we have $|[\nabla\hat{\mu}_t]_1|\le \mathcal{O}(\sigma^{-2}R)\hat{\mu}_t$. 

Then
\begin{align}
    |\nabla\log\hat{\mu}_{t,1}-f_3 |&\le \mathcal{O}(\sigma^{-2}R)\hat{\mu}_t|\frac{1}{\hat{\mu}_t}-\frac{1}{f_{1,\text{clip}}}|+\frac{|\sigma [\nabla\hat{\mu}_t]_1-f_2|}{\sigma f_{1,\text{clip}}}\\
    &=\frac{1}{f_{1,\text{clip}}} [\mathcal{O}(\sigma^{-2}R)|\hat{\mu}_t-f_{1,\text{clip}}|+\frac{|\sigma [\nabla\hat{\mu}_t]_1-f_2|}{\sigma}]\\
    &\lesssim \frac{1}{\hat{\mu}_t} [\mathcal{O}(\sigma^{-2}R)|\hat{\mu}_t-f_{1,\text{clip}}|+\frac{|\sigma [\nabla\hat{\mu}_t]_1-f_2|}{\sigma}]
\end{align}

We have
\begin{align}
    |\nabla\log\hat{\mu}_{t,1}-f_3 |\hat{\mu}_t \lesssim  \sigma^{-2} B_0 (2R)^{\beta+1} N^{-\beta} \frac{d^{\beta}}{\beta !}
\end{align}

We set
\begin{align}
    |f_3 -f_3^{RELU} |\lesssim  (B_0 (2R)^{\beta} N^{-\beta} \frac{d^{\beta}}{\beta !}R^d)^{\frac{1}{2}}
\end{align}

Let $f_{\gamma}=f_3^{RELU}$ on $\|x\|_{2}\le R$, then $A_3$:
\begin{equation}
    \begin{aligned}
        &\int_{\|x\|_{2}\le R} \textbf{1}\{|\hat{\mu}_t(x)|\ge \epsilon_{\rm low}\} \|f_{\gamma}(x)-\nabla\log\hat{\mu}_t(x)\|_2^2 \hat{\mu}_t(x)d x  \\
        \lesssim&\int_{\|x\|_{2}\le R} \textbf{1}\{|\hat{\mu}_t(x)|\ge \epsilon_{\rm low}\} \|\nabla\log\hat{\mu}_t(x)-f_3(x)\|_2^2 \hat{\mu}_t(x)d x  \\
        &+\int_{\|x\|_{2}\le R} \textbf{1}\{|\hat{\mu}_t(x)|\ge \epsilon_{\rm low}\} \|f_3(x)-f_{\gamma}(x)\|_2^2 \hat{\mu}_t(x)d x  \\
        \lesssim &\frac{1}{\epsilon_{\text{low}}} \sigma^{-4}B_0^2 (2R)^{2\beta+2} N^{-2\beta} \frac{d^{2\beta}}{(\beta !)^2} R^{d}+B_0 (2R)^{\beta} N^{-\beta} \frac{d^{\beta}}{\beta !} R^{d}\\
        \lesssim &\sigma^{-4}B_0 (2R)^{\beta+2} N^{-\beta} \frac{d^{\beta}}{\beta !} R^{d}
    \end{aligned}
\end{equation}

The last $\lesssim$ is obtained by plugging in $\epsilon_{\text{low}}:= C_{\text{low}} B_0 (2R)^{\beta} N^{-\beta} \frac{d^{\beta}}{\beta !}$.

\item 
\textbf{(summing up error terms)}

We have $M=\mathcal{O}(\sigma^{-2}R)$,
the overall approximation error is $\mathcal{O}(d\sigma^{-4}(R^{-(p-2)}+R^{d+\beta+2}(\frac{N}{2})^{-\beta}B_0\frac{d^{\beta}}{\beta !}))$, the optimal $R=\mathcal{O}((\frac{p-2}{d+\beta+2})^{\frac{1}{d+\beta+p}}(\frac{N}{2})^{\frac{\beta}{d+\beta+p}}(B_0\frac{d^{\beta}}{\beta !})^{-\frac{1}{d+\beta+p}})$, the optimal error is 
$\\ \mathcal{O}(\sigma^{-\frac{(d+\beta)(p+2)+4p}{d+\beta+p}} N^{-\frac{(p-2)\beta}{d+\beta+p}})$.
We set $\beta=\mathcal{O}((d+p)^2)$.
The ReLU approximation error $|f_3 -f_3^{RELU} |\lesssim  (B_0 (2R)^{\beta} N^{-\beta} \frac{d^{\beta}}{\beta !}R^d)^{\frac{1}{2}}=\mathcal{O}(\sigma^{-\frac{(d+\beta)p}{2(d+\beta+p)}}N^{-\frac{p\beta}{2(d+\beta+p)}})$.
$\epsilon_{\text{low}}=\mathcal{O}(\sigma^{-\frac{(d+\beta)(d+p)}{d+\beta+p}}N^{-\frac{(d+p)\beta}{d+\beta+p}})$.

Now we calculate how much complexity needed for 
\begin{align}
    |f_3 -f_3^{RELU} |\lesssim \epsilon_{\text{total}}:= \sigma^{-\frac{(d+\beta)p}{2(d+\beta+p)}}N^{-\frac{p\beta}{2(d+\beta+p)}}
\end{align}

The estimator of $\frac{f_2}{\sigma f_{1,\text{clip}}}$ is $f_2^{RELU}\cdot  (\sigma f^{RELU}_{1,\text{clip}})^{-1}$. Since $|f_2^{RELU}-f_2|\le \epsilon_{\text{poly}}, |(\sigma f^{RELU}_{1,\text{clip}})^{-1}-\sigma f_{1,\text{clip}}^{-1}|\le \epsilon_{\text{inv}}+\frac{\epsilon_{\text{poly}}}{\epsilon^2_\text{inv}}$, the upper bound of $f_2$ and $(\sigma f_{1,\text{clip}})^{-1}$ is $ u_{\text{upper}}:=\max\{\mathcal{O}(R)\epsilon_{\text{low}}, \frac{1}{\sigma \epsilon_{\text{low}}}\}$. Denote $\epsilon_\text{all}:=\max\{\epsilon_{\text{poly}}, \epsilon_{\text{inv}}+\frac{\epsilon_{\text{poly}}}{\epsilon_\text{inv}^2} \}=2(\epsilon_{\text{poly}})^{\frac{1}{3}}$  (with $\epsilon_{\text{inv}}=\mathcal{O}(\epsilon_{\text{poly}})^{\frac{1}{3}}$), then by Lemma \ref{lem:nnreciprocal},
there exists $f_{\text{mult}}(x_1,x_2)\in \mathcal{S}( W, B, L, S)$ with
   $L = \mathcal{O}(\log 2(\log \epsilon_{\text{product}}^{-1}+ 2 \log u_{\text{upper}})), W = 96, S = \mathcal{O}(2 \log \epsilon_{\text{product}}^{-1} + 2\log u_{\text{upper}})), B=u_{\text{upper}}^2$
    such that
    \begin{align}
     &   \left|f_{\text{mult}}(f_2^{RELU},\sigma f^{RELU}_{1,\text{clip}})^{-1}) - \frac{f_2}{\sigma f_{1,\text{clip}}}\right| \leq \epsilon_{\text{product}} + 2 u_{\text{upper}} \epsilon_{\text{all}}.
    \end{align}
    for all $x\in [-u_{\text{upper}},u_{\text{upper}}]^2\text{ and } x'\in \R$ with $\|x-x'\|_\infty \leq \epsilon_{\text{all}}$.

We set 
 $ 2 u_{\text{upper}} \epsilon_{\text{all}}\lesssim \frac{1}{\epsilon_{\text{low}}}(\epsilon_{\text{poly}})^{\frac{1}{3}}\lesssim \epsilon_{\text{total}} $, plugging in 
$\epsilon_{\text{low}}=\mathcal{O}(\sigma^{-\frac{(d+\beta)(d+p)}{d+\beta+p}}N^{-\frac{(d+p)\beta}{d+\beta+p}})$,
we obtain 
$\epsilon_{\text{poly}}=\mathcal{O}(\sigma^{-\frac{3(d+\beta)(d+\frac{3p}{2})}{d+\beta+p}}N^{-\frac{3(d+\frac{3p}{2})\beta}{d+\beta+p}})$, hence $\epsilon_{\text{inv}}=\mathcal{O}(\epsilon_{\text{poly}})^{\frac{1}{3}}=\mathcal{O}(\sigma^{-\frac{(d+\beta)(d+\frac{3p}{2})}{d+\beta+p}}N^{-\frac{(d+\frac{3p}{2})\beta}{d+\beta+p}})$.

Note that Lemma F.3 in \cite{fu2024unveil} indicates that $W,S$ are complexities that add up in neural network summation, then $W=\mathcal{O}(N^d)$ and $S$ is determined by Equation (\ref{append:taylorsumcomplexity}) and (\ref{append:taylorsinglecomplexity}). On the other hand, $B,L$ are maximum complexities in the Taylor polynomial approximation part according to Lemma F.3 in \cite{fu2024unveil}, and $L$ is dominated by the largest $L_{i}$ during concatenation when $k=2$ is a constant in Lemma F.1 of \cite{fu2024unveil}. Hence, $B,L$ are determined by the largest possible in the neural network structure, which is the inverse part (\ref{append:inversecomplexity}).
Using Lemma F.1-F.4 in \cite{fu2024unveil} and hiding factors that are not related to 
$N$ and $\sigma$,
the final complexity is:
\begin{equation}
\begin{aligned}    
    &M=\mathcal{O}(\sigma^{-2}R)=\mathcal{O}(\sigma^{-\frac{d+\beta+2p}{d+\beta+p}}N^{\frac{\beta}{d+\beta+p}})=\mathcal{O}(\sigma^{-(1-o(1))}N^{(1-o(1))})\\
    &W=\mathcal{O}(N^d)\\
    &B=\mathcal{O}(\epsilon_{\text{inv}}^{-2})=\mathcal{O}(\sigma^{\frac{(d+\beta)(2d+3p)}{d+\beta+p}}N^{\frac{(2d+3p)\beta}{d+\beta+p}})=\mathcal{O}(\sigma^{(2d+3p-o(1))}N^{(2d+3p-o(1))})\\
    &L=\mathcal{O}(\log^2\epsilon_{\text{inv}}^{-1})=\mathcal{O}(\log^2 (\sigma^{-1})\log^2 N)\\
    &S=\mathcal{O}(\log (\sigma^{-1})N^d\log N)=\mathcal{O}(\log (\sigma^{-1})N^d)
\end{aligned}
\end{equation}

Substituting $N$ with $N^{\frac{1}{d}}$ obtains the desired result.

\end{enumerate}

\end{proof}

If replacing the $p$-th moment finite Assumption \ref{assp:moment} by the sub-Gaussian Assumption \ref{assp:subg}, the main difference is at the $\textbf{Step 1.}$ in the proof of Theorem \ref{appendix:thm::nn approx}. The following is the proof for the sub-Gaussian case Theorem \ref{thm::nn approx2}.

\begin{theorem}\label{appendix:thm::nn approx2}
Under Assumption \ref{assp:subg}, for sufficiently large $N$, there exists $f_{\gamma} \in \mathcal{S}(M,W,B,L,S)$ such that 
\begin{align*}
    \ell_{sm}(f_{\gamma})=\mathcal{O}(\sigma^{-4}\exp(-C'\sigma^{\frac{d-2}{d+1}} N^{\frac{d-2}{d(d+1)}})).
\end{align*}
The hyperparameters in the ReLU neural network class $\mathcal{S}$ satisfy
\begin{align*}
    &M = \mathcal{O}(\sigma^{-\frac{3(d+2)}{2(d+1)}}N^{\frac{d-2}{2d(d+1)}}),~
     W = \mathcal{O}(N),~
     B =\mathcal{O}(\sigma^{\frac{3\sigma^{\frac{d}{d+1}}N^{\frac{1}{d+1}}+3d^2}{d}}N^{\frac{3\sigma^{\frac{d}{d+1}}N^{\frac{1}{d+1}}}{d(d+1)}}),~
    L = \mathcal{O}(\sigma^{\frac{2d}{d+1}}N^{\frac{2}{d+1}}),\\
    ~
     &S= \mathcal{O}(\sigma^{\frac{d}{d+1}}N^{\frac{d+2}{d+1}}).
\end{align*}
where $\mathcal{O}$ hides the other factors that are not related to $N$ or $\sigma$.
\end{theorem}

\begin{proof}
    The proof of Theorem \ref{appendix:thm::nn approx2} is similar to the proof of Theorem \ref{appendix:thm::nn approx}. Only some details need modification.

    The main focus is $\textbf{Step 1.}$ of the proof of Theorem \ref{appendix:thm::nn approx}.
    By sub-Gaussian properties and integration by parts, we have:
\begin{equation}
    \int_{\|x\|_2\ge R} d \hat{\mu}_t(x) \lesssim d\exp(-C'R^2)
\end{equation}

Since sub-Gaussian implies second moment finiteness, we still have $\|\nabla\log \hat{\mu}_t(x)\|\le C_0\|x\|+M_0$ from Lemma \ref{lem:convolution}.
Then by sub-Gaussian property:
\begin{equation}
    \int_{\|x\|_{2}\ge R} \|\nabla \log \hat{\mu}_t(x)\|_2^2 \hat{\mu}_t(x) d x\lesssim \sigma^{-4}\int_{\|x\|_{2}\ge R} (\|x\|^2_{2}+1) \hat{\mu}_t(x) d x\lesssim  d\sigma^{-4}R^2 \exp(-C'R^2)
\end{equation}

Summing up the above, and plugging in the upper bound of the neural network $M$, we obtain the new upper bound for $A_1$:
\begin{equation}
\begin{aligned}
    A_1&\le 2\int_{\|x\|_{2}> R}\|f_{\gamma}(x)\|_2^2 \hat{\mu}_t(x)d x +2\int_{\|x\|_{2}> R} \|\nabla\log\hat{\mu}_t(x)\|_2^2 \hat{\mu}_t(x)d x \\
    &\lesssim d M^2 \int_{\|x\|_{2}> R} \hat{\mu}_t(x)d x  + d\sigma^{-4}R^2 \exp(-C'R^2)\\
    &\lesssim d M^2 \exp(-C'R^2) + d\sigma^{-4}R^2 \exp(-C'R^2)
\end{aligned}
\end{equation}

For $M=\mathcal{O}(\sigma^{-2}R)$, we have $A_1\le \mathcal{O}(\sigma^{-4}R^2 \exp(-C'R^2))$.

 $\textbf{Step 2.}$ and $\textbf{Step 3.}$ of the proof of Theorem \ref{appendix:thm::nn approx} do not related to moment or sub-Gaussian assumptions. From $\textbf{Step 2.}$ and $\textbf{Step 3.}$, we obtain $A_2+A_3\le\mathcal{O}(\sigma^{-4}B_0R^{d+2}(2R)^{\beta}N^{-\beta}\frac{d^{\beta}}{\beta!})$.

 Set $N=\mathcal{O}(\sigma^{-1})$,
 $\beta\gg d^2$ and $R= \beta^{(\frac{1}{2}-\frac{1}{d})}$,
 then $A_2+A_3\le\mathcal{O}(\sigma^{-(d+4)}\beta^{-\frac{2\beta+4-d^2}{2d}})$,
 $A_1\le \mathcal{O}(\sigma^{-4}\exp(-C'\beta^{(1-\frac{2}{d})}))$. 
 The overall error is $\mathcal{O}(\sigma^{-4}\exp(-C'\beta^{(1-\frac{2}{d})}))$. 
 The ReLU approximation error $|f_3 -f_3^{RELU} |\lesssim  (B_0 (2R)^{\beta} N^{-\beta} \frac{d^{\beta}}{\beta !}R^d)^{\frac{1}{2}}=\mathcal{O}(\sigma^{-\frac{d}{2}}\beta^{(-\frac{\beta}{2d}+\frac{1}{4}d-\frac{1}{2})})$. $\epsilon_{\text{low}}=\mathcal{O}(\sigma^{-d}\beta^{-\frac{\beta}{d}})$.

Now we calculate how much complexity needed for 
\begin{align}
    |f_3 -f_3^{RELU} |\lesssim \epsilon_{\text{total}}:= \sigma^{-\frac{d}{2}}\beta^{(-\frac{\beta}{2d}+\frac{1}{4}d-\frac{1}{2})}
\end{align}

From here we follow the same procedure as $\textbf{Step 4.}$ of the proof of Theorem \ref{appendix:thm::nn approx}. We set 
 $ 2 u_{\text{upper}} \epsilon_{\text{all}}\lesssim \frac{1}{\epsilon_{\text{low}}}(\epsilon_{\text{poly}})^{\frac{1}{3}}\lesssim \epsilon_{\text{total}} $  and $\epsilon_{\text{inv}}=\mathcal{O}(\epsilon_{\text{poly}})^{\frac{1}{3}}$. Plugging in 
$\epsilon_{\text{low}}=\mathcal{O}(\sigma^{-d}\beta^{-\frac{\beta}{d}})$
, we obtain 
$\epsilon_{\text{poly}}=\mathcal{O}(\sigma^{-\frac{9d}{2}}\beta^{(-\frac{9\beta}{2d}+\frac{3}{4}d-\frac{3}{2})})$, 
$\epsilon_{\text{inv}}=\mathcal{O}(\sigma^{-\frac{3d}{2}}\beta^{(-\frac{3\beta}{2d}+\frac{1}{4}d-\frac{1}{2})})$.

Using Lemma F.1-F.4 in \cite{fu2024unveil} and hiding factors that do not depend on lower order terms of $\sigma$ and $\beta$, the final complexity is:
\begin{equation}
\begin{aligned}    
    &M=\mathcal{O}(\sigma^{-2} R)=\mathcal{O}(\sigma^{-2}\beta^{(\frac{1}{2}-\frac{1}{d})})\\
    &W=\mathcal{O}(N^d\beta^{d+1})=\mathcal{O}(\sigma^{-d}\beta^{d+1})\\
    &B=\mathcal{O}(\epsilon_{\text{inv}}^{-2})=\mathcal{O}(\sigma^{3d}\beta^{\frac{3\beta}{d}})\\
    &L=\mathcal{O}(\log^2\epsilon_{\text{inv}}^{-1})=\mathcal{O}((\log \sigma^{-1})^2\beta^2)\\
    &S=\mathcal{O}(\log(\sigma^{-1})N^d\beta^{d+2})=\mathcal{O}(\sigma^{-d}\beta^{d+2})
\end{aligned}
\end{equation}

Denoting $\beta$ as $\mathcal{O}(N^{\frac{1}{d+1}}\sigma^{\frac{d}{d+1}})$ (slight abuse of notation $N$) gives the desired result.

\end{proof}

\subsection{Proof of Corollary \ref{coro1} and Corollary \ref{coro2}}  \label{appd:proof4}

\begin{corollary}\label{appendix:coro1}
    To guarantee that Assumption \ref{assp:approx} holds with $\ell_{sm}(\hat{f_{\gamma}})\le \epsilon_k=\mathcal{O}(\epsilon)$, under Assumption \ref{assp:moment}, it suffices to require that the neural network complexity $N=\mathcal{O}(\sigma^{-\frac{7d}{3-o(1)}}\epsilon^{-\frac{d}{3-o(1)}})$, and the sample size $n=\mathcal{O}(\sigma^{-\frac{7d+20}{3-o(1)}}\epsilon^{-\frac{d+8}{3-o(1)}})$, where $\mathcal{O}$ hides the other terms that are not related to $\epsilon$ or $\sigma$.
\end{corollary}

\begin{proof}
    To guarantee that Assumption \ref{assp:approx} holds that $\ell_{sm}(\hat{f_{\gamma}})\le \epsilon_k=\mathcal{O}(\epsilon)$, we need $\mathcal{O}\left(\sqrt{\frac{M^2+\frac{d}{\sigma^2}}{n}\log\frac{\mathcal{N}}{\delta}}\right)=\mathcal{O}(\epsilon)$ and $\ell_{sm}(f_{\gamma})=\mathcal{O}(\epsilon)$. Theorem \ref{thm::nn approx} shows that the neural network complexity $N=\mathcal{O}(\sigma^{-\frac{7d}{3-o(1)}}\epsilon^{-\frac{d}{3-o(1)}})$, with $M=\mathcal{O}(\sigma^{-\frac{10}{3-o(1)}}\epsilon^{-\frac{1}{3-o(1)}})$. Plugging $M, S, L$ into Theorem \ref{thm:generalization} gives that the sample complexity $n=\mathcal{O}(\sigma^{-\frac{7d+20}{3-o(1)}}\epsilon^{-\frac{d+8}{3-o(1)}})$.
\end{proof}

\begin{corollary}\label{appendix:coro2}
    To guarantee that Assumption \ref{assp:approx} holds with $\ell_{sm}(\hat{f_{\gamma}})\le \epsilon_k=\mathcal{O}(\epsilon)$, under Assumption \ref{assp:subg}, it suffices to require that the neural network complexity $N=\mathcal{O}(\sigma^{-d}\text{polylog}(\epsilon^{-1}))$, and the sample complexity $n=\mathcal{O}(\sigma^{-d}\text{polylog}(\epsilon^{-1}) \epsilon^{-1})$, where $\mathcal{O}$ hides the other terms that are not related to $\epsilon$ or $\sigma$.
\end{corollary}

\begin{proof}
    To guarantee that Assumption \ref{assp:approx} holds that $\ell_{sm}(\hat{f_{\gamma}})\le \epsilon_k=\mathcal{O}(\epsilon)$, we need $\mathcal{O}\left(\frac{M^2}{n}\mathcal{N}\right)=\mathcal{O}(\epsilon)$ and $\ell_{sm}(f_{\gamma})=\mathcal{O}(\epsilon)$. By Theorem \ref{thm::nn approx2}, $\sigma^{-4}\exp(-C'\sigma^{\frac{d-2}{d+1}}N^{\frac{d-2}{d(d+1)}})\le \epsilon$, hence $N=\\ \mathcal{O}((C')^{-\frac{d(d+1)}{d-2}}\sigma^{-d}(\log\epsilon^{-1})^{\frac{d(d+1)}{d-2}})$. Then 
\begin{equation}
\begin{aligned}    
    &M=\mathcal{O}(\sigma^{-\frac{3(d+2)}{2(d+1)}}N^{\frac{d-2}{2d(d+1)}})=\mathcal{O}((C')^{-\frac{1}{2}}\sigma^{-2}(\log\epsilon^{-1})^{\frac{1}{2}})\\
    &\log B=\mathcal{O}(\sigma^{\frac{d}{d+1}} N^{\frac{1}{d+1}})=\mathcal{O}((C')^{-\frac{d}{d-2}}(\log\epsilon^{-1})^{\frac{d}{d-2}})\\
    &L=\mathcal{O}(\sigma^{\frac{2d}{d+1}}N^{\frac{2}{d+1}})=\mathcal{O}((C')^{-\frac{2d}{d-2}}(\log\epsilon^{-1})^{\frac{2d}{d-2}})\\
    &S=\mathcal{O}(\sigma^{\frac{d}{d+1}}N^{\frac{d+2}{d+1}})=\mathcal{O}((C')^{-\frac{d(d+2)}{d-2}}\sigma^{-d}(\log\epsilon^{-1})^{\frac{d(d+2)}{d-2}}).
\end{aligned}
\end{equation}
    Plugging $M, \log B, L,S$ into Theorem \ref{thm:generalization2} and ignoring logarithm terms gives the sample complexity $n=\mathcal{O}(\epsilon^{-1}[(C')^{-\frac{d^2+6d-2}{d-2}}\sigma^{-(d+4)}(\log\epsilon^{-1})^{\frac{d^2+6d-2}{d-2}}])$. Omitting the constants leads to $N=\mathcal{O}(\sigma^{-d}\text{polylog}(\epsilon^{-1}) )$, $n=\mathcal{O}(\sigma^{-d}\text{polylog}(\epsilon^{-1}) \epsilon^{-1})$.
\end{proof}

\section{Additional Details of Experiments}\label{appd:exp_setup}

\subsection{Gaussian Mixture}

For the 2D experiment, the target marginal probability
of each cluster is 1/5, the mean of each cluster is randomly sampled from the standard Gaussian distribution, the standard deviations of the clusters are 0.1, 0.2, 0.3, 0.4, 0.5. The initial distribution is the Gaussian distribution $\mathcal{N}((3,0),0.25\mathbf{I})$. The number of particles is 1000. 

For the neural network structure, we follow the setting of \citet{cheng2023gwg}. For $L_2$-GF, Ada-GWG and Ada-SIFG, we parameterize $f_{\gamma}$ as 3-layer neural networks with $\textit{tanh}$ activation function. Each hidden layer has $32$ neurons. The inner loop iteration is 5 and we use SGD optimizer with Nesterov momentum (momentum 0.9) to train $f_{\gamma}$ with learning rate $\eta$=1e-3. For $L_2$-GF and Ada-GWG, the particle step size is $0.1$. For Ada-SIFG, the particle step size is $0.01$.

For Ada-GWG, we set the initial exponent $p_0=2$ and the learning rate $\Tilde{\eta}= 5$e-7. For Ada-SIFG, we set the initial noise magnitude $\sigma_0=0.12$ and the learning rate $\hat{\eta}= 1$e-9.

For the 10D experiment, the target marginal probability
of each cluster is 1/5, the mean of each cluster is randomly sampled from the standard Gaussian distribution, the standard deviations of the clusters are 0.1, 0.2, 0.3, 0.4, 0.5. The initial distribution is the standard Gaussian distribution. The number of particles is 1000. The neural network structure is the same as which in the 2D Gaussian mixture experiment for $L_2$-GF, Ada-GWG, SIFG and Ada-SIFG. 
For $L_2$-GF and Ada-GWG, the particle step size is $0.1$. For SIFG and Ada-SIFG, the particle step size is $0.01$.

For Ada-GWG, we set the initial exponent $p_0=2$ and the learning rate $\Tilde{\eta}= 5$e-7. For SIFG, we set the initial noise magnitude $\sigma_0=0.1$. For Ada-SIFG, we align with the initial noise magnitude $\sigma_0=0.1$ of SIFG, and choose the learning rate $\hat{\eta}= 1$e-9.

We run the experiment on 5 random seeds. The average results and the variances are represented in the figure using lines and shades.

\subsection{Monomial Gamma}

We follow the same setting as \citet{cheng2023gwg}. The number of particles is 1000. For $L_2$-GF and Ada-GWG, the neural network structure and its training optimizer are the same as which in the Gaussian mixture experiment. For all methods, the inner loop iteration is 5. For $L_2$-GF and Ada-GWG, we use Adam optimizer with learning rate $\eta$=1e-3 to update the particles for better stability. For SIFG and Ada-SIFG, we use SGD optimizer with learning rate $\eta$=1e-2.

For Ada-GWG, we set the initial exponent $p_0=2$ and learning rate $\Tilde{\eta}= 1$.  For SIFG, we set the initial noise magnitude $\sigma_0=0.1$. For Ada-SIFG, we align with the initial noise magnitude $\sigma_0=0.1$ of SIFG, and choose the learning rate $\hat{\eta}= 1$e-10.

We run the experiment on 5 random seeds. The average results and the variances are represented in the figure using lines and shades.

\subsection{Independent Component Analysis}

We follow the setting of \citet{korba2021ksdd}. For all methods, the total number of iterations is 2000, which is sufficient for all methods to converge. For $L_2$-GF, Ada-GWG, SIFG and Ada-SIFG, we parameterize $f_{\gamma}$ as 3-layer neural networks with $\textit{tanh}$ nonlinearities. Each hidden layer has $120$ neurons. The inner loop $N'$ is 20. We use SGD optimizer with learning rate $\eta$=1e-3 to train $f_{\gamma}$. The particle step size is 
chosen from \{3e-4, 1e-3, 3e-3\} to obtain the best performance. 

For Ada-GWG, we set the initial exponent $p_0=2$ and the learning rate $\Tilde{\eta}$ chosen from $\{0.002, 0.004, 0.008\}$ for the best performance. We clip the exponent $p$ within $[1.1, 6]$. The gradient of $A(p)$ is also clipped within $[-0.1, 0.1]$. 

For SIFG and Ada-SIFG, we set the initial noise magnitude $\sigma_0=0.03$. For Ada-SIFG, the learning rate for $\sigma$ is chosen from \{1e-6, 2e-6, 4e-6\} to obtain the best performance.

For SVGD, we use RBF kernel $\exp(-\frac{\|x-y\|^2}{h})$, where $h$ is the heuristic bandwidth \citep{Liu2016SVGD}. 
The particle step size is
chosen from \{1e-2, 3e-2, 1e-1\} to obtain the best performance.

Additionally, we run standard HMC with step size 5e-4 and leap-frog number 40 for $100000$ iterations as the ground truth for the posterior distribution.

\subsection{Bayesian Neural Networks}\label{sec:detail-bnn}
We follow the settings of \citet{cheng2023gwg}. 
For the UCI datasets, the datasets are randomly partitioned into 90\% for training and 10\% for testing. 
Then, we further split the training dataset by 10\% to create a validation set for hyperparameter selection as done in \citet{cheng2023gwg}. We select the step size of particle updates from \{1e-4, 2e-4, 5e-4, 1e-3\} for the best performance.
For $L_2$-GF, Ada-GWG and Ada-SIFG, we parameterize $f_{\gamma}$ as 3-layer neural networks. Each hidden layer has $300$ neurons, and we use LeakyReLU as the activation function with a negative slope of 0.1. The inner loop $N'$ is 10. We use the Adam optimizer and choose the learning rate from \{1e-3, 1e-4\} to train $f_{\gamma}$. 

For Ada-GWG, we choose the initial exponent $p_0$ to be 2 and set the learning rate $\Tilde{\eta}$ to be 1e-4. The gradient of $A(p)$ is clipped within [-0.2, 0.2]. For Ada-SIFG, we choose the initial noise magnitude $\sigma$ from $\{0.1,0.01\}$. The learning rate for $\sigma$ is chosen from \{2e-5, 4e-5, 1e-4\} for the best performance. For SVGD, we use the RBF kernel as done in~\citep{Liu2016SVGD}.
For methods except SGLD, the iteration number is chosen to be 2000 to converge. 
For SGLD, the iteration number is set to 10000 to converge.

We additionally verify the fifth moment assumption on the BNN experiments. Based on 10 independent runs, Figure \ref{figure: bnn_moments} displays the fifth moments of the particle distributions along training iterations. We can see that the moments are uniformly bounded for all datasets during training. 

\begin{figure}[!t]
    \centering
    \includegraphics[width=0.90\linewidth]{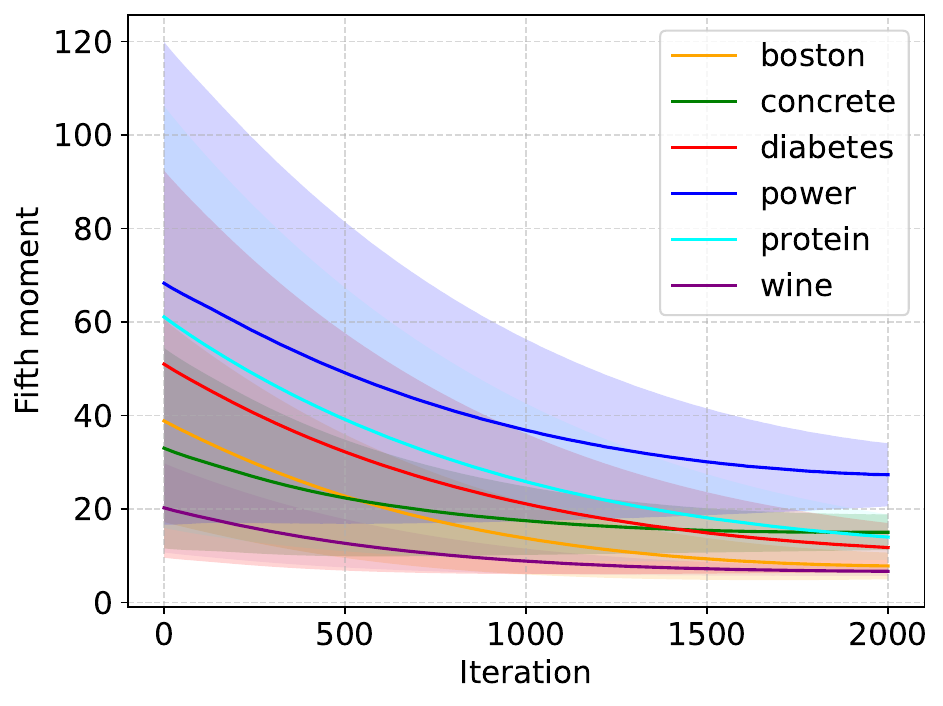}
    \caption{The fifth moments of SIFG for BNN experiments on the six datasets.
    }
    \label{figure: bnn_moments}
\end{figure}

\end{document}